\documentclass[twoside,11pt]{article}

\pdfoutput=1

\usepackage[hmargin=2.5cm, vmargin=2.3cm]{geometry}
\usepackage{amsthm}

\newtheorem{theorem}{Theorem}[section]

\newtheorem{lemma}{Lemma}[section]
\newtheorem{remark}{Remark}[section]
\newtheorem{definition}{Definition}[section]

\usepackage{microtype}
\usepackage{graphicx}
\usepackage{booktabs} 
\usepackage{amsmath}
\usepackage{subfig}
\usepackage{amsfonts}

\usepackage{hyperref}

\usepackage{definitions}

\title{\textbf{\Large{Approximate Leave-One-Out for High-Dimensional Non-Differentiable Learning Problems}}}
\author{\small Shuaiwen Wang\tsup{1,*}, Wenda Zhou\tsup{1,*}, Arian Maleki\tsup{1}, Haihao Lu\tsup{2}, Vahab Mirrokni\tsup{3}}

\date{}

\begin{document}
\maketitle

{\let\thefootnote\relax\footnotetext{
    \tsup{1}Department of Statistics, Columbia University, New York, USA;
    \tsup{2}Mathematics Department and Operation Research Center, Massachusetts Institute of Technology, Massachusetts, USA;
    \tsup{3}Google Research, New York, USA;
    \tsup{*}Equal contributions.
}}

\begin{abstract}
Consider the following class of learning schemes:
\begin{equation}\label{eq:main-problem1}
    \hat{\betav} := \argmin_{\betav \in \calC}\;\sum_{j=1}^n \ell(\xv_j^\top\betav; y_j) + \lambda R(\betav),
\end{equation}
where $\xv_i \in \mathbb{R}^p$ and $y_i \in \mathbb{R}$ denote the $i^{\rm
th}$ feature and response variable respectively. Let $\ell$ and $R$ be
the convex loss function and regularizer, $\betav$ denote the unknown weights, and
$\lambda$ be a regularization parameter. $\calC \subset \mathbb{R}^{p}$ is a closed convex
set. Finding the optimal choice of $\lambda$ is a challenging problem in
high-dimensional regimes where both $n$ and $p$ are large. We propose three
frameworks to obtain a computationally efficient approximation of the
leave-one-out cross validation (LOOCV) risk for nonsmooth losses and
regularizers. Our three frameworks are based on the primal, dual, and proximal
formulations of \eqref{eq:main-problem1}. Each framework shows its strength in
certain types of problems. We prove the equivalence of the three
approaches under smoothness conditions. This equivalence enables us to justify
the accuracy of the three methods under such conditions. We use our approaches
to obtain a risk estimate for several standard problems, including generalized
LASSO, nuclear norm regularization, and support vector machines. We
empirically demonstrate the effectiveness of our results for non-differentiable
cases.
\end{abstract}


\section{Introduction}
\label{intro}
\subsection{Motivation}
Consider a standard prediction problem in which a dataset $\{(y_j,
\xv_j)\}_{j=1}^n \subset \mathbb{R}\times\mathbb{R}^{p}$ is employed to learn
a model for inferring information about new datapoints that are yet to
be observed. One of the most popular classes of learning schemes, specially
in high-dimensional settings, studies the following optimization
problem:
\begin{equation}\label{eq:main-problem}
    \hat{\betav} := \argmin_{\betav \in \calC}\;\sum_{j=1}^n \ell(\xv_j^\top\betav; y_j) + \lambda R(\betav),
\end{equation}
where $\ell: \mathbb{R}^2 \rightarrow \mathbb{R}$ is a convex loss function, $R:
\mathbb{R}^p \rightarrow \mathbb{R}$ is a convex regularizer, $\calC \subset
\mathbb{R}^p$ is a closed convex set and $\lambda$ is
the tuning parameter that specifies the amount of regularization. By applying
an appropriate regularizer in \eqref{eq:main-problem},
we are able to achieve better bias-variance trade-off
and pursue special structures such as sparsity and low rank structure.
However, the performance of such techniques hinges upon the selection of tuning
parameters.

The most generally applicable tuning method is cross validation
\cite{stone1974cross}. One common choice
is $k$-fold cross validation. This method presents potential bias issues in
high-dimensional settings where $n$ is comparable to $p$, specially when the number of folds is not very large. For instance,
the phase transition phenomena that happen in such regimes
\cite{amelunxen2014living, donoho2009message, donoho2005neighborliness, weng2016overcoming}
indicate that any data splitting may cause dramatic effects on the solution of
\eqref{eq:main-problem} (see Figure \ref{fig:lasso-risk-loocv-kfold} for an
example). Hence, the
risk estimates obtained from $k$-fold cross validation may not be reliable.
The bias issues of $k$-fold cross validation may be alleviated by choosing
the number of folds $k$ to be large. This makes LOOCV particularly appealing, since it offers an approximately unbiased estimate of the risk. 
 However, the computation of LOOCV requires training the model $n$ times, which is unaffordable for large
datasets.

\begin{figure}[t!]
    \begin{center}
        \includegraphics[scale=0.4]{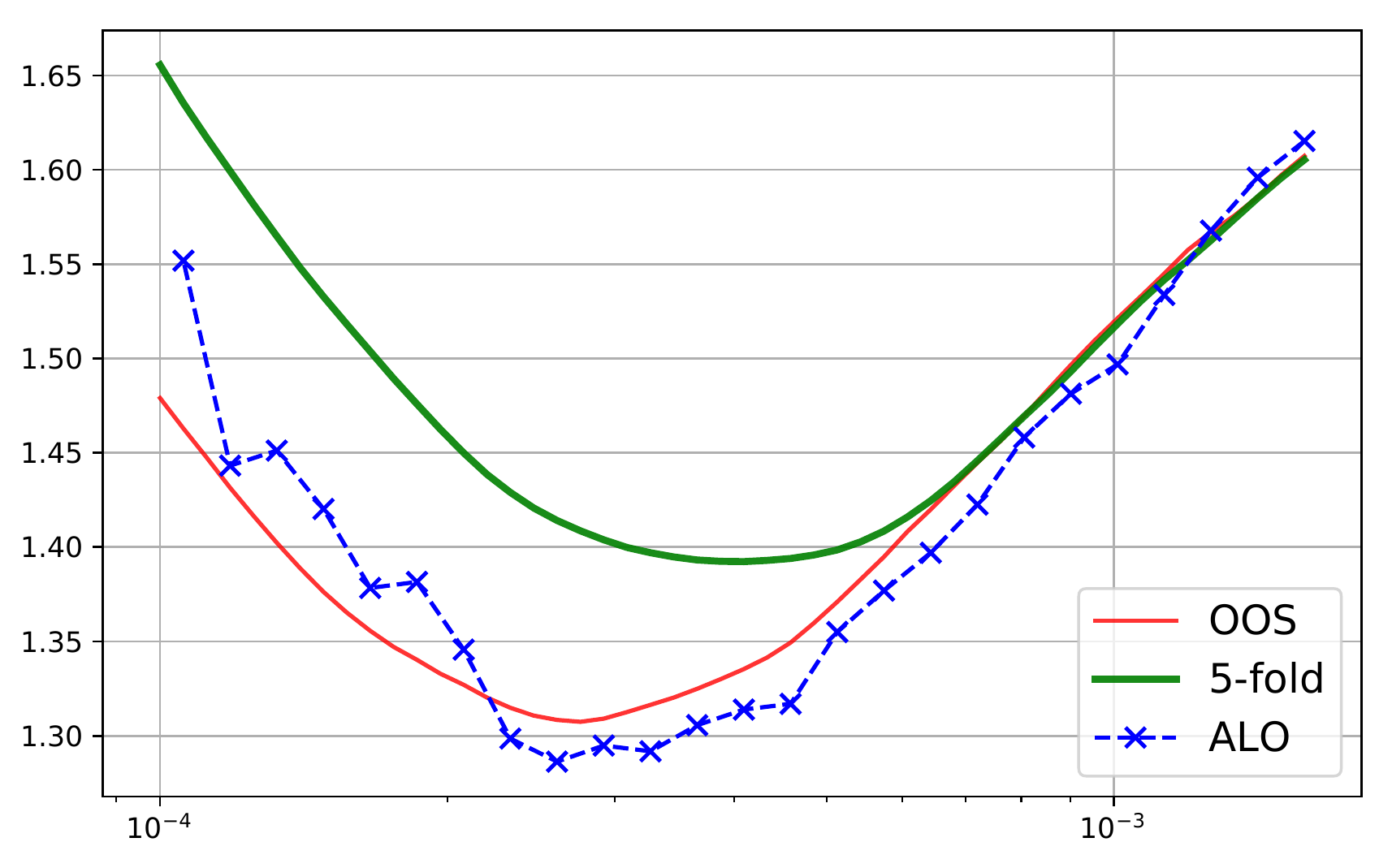}
        \caption{Risk estimates of LASSO based on 5-fold CV and ALO proposed in
        this paper, compared with the true out-of-sample prediction error
        (OOS). In this example, 5-fold CV provides biased estimates of OOS,
        while ALO works just fine. Here we use $n=5000$, $p=4000$ and $iid$ Gaussian design.}
        \label{fig:lasso-risk-loocv-kfold}
    \end{center}
\end{figure}

The high computational complexity of LOOCV has motivated researchers to propose
computationally less demanding approximations of the quantity. Early examples
offered approximations for the case $R(\betav) = \frac{1}{2}\|\betav\|_2^2$ and
the loss function being smooth \cite{allen1974relationship, o1986automatic,
le1992ridge, cawley2008efficient, meijer2013efficient, opper2000gaussian}.
In \cite{beirami2017optimal}, the authors considered such approximations for
smooth loss functions and smooth regularizers. In
this line of work, the accuracy of the approximations was either not studied or
was only studied in the $n$ large, $p$ fixed regime. In a recent paper,
\cite{kamiar2018scalable} employed a similar approximation strategy to obtain
approximate leave-one-out formulas for smooth loss functions and smooth
regularizers. They show that under some mild conditions, such approximations
are accurate in high-dimensional settings. Unfortunately, the approximations
offered in \cite{kamiar2018scalable} only cover twice differentiable loss
functions and regularizers. On the other hand, numerous modern regularizers,
such as generalized LASSO and nuclear norm, and also many loss functions, such as hinge loss, are
not smooth.

In this paper, we propose three powerful frameworks for
calculating an approximate leave-one-out estimator (ALO) of the LOOCV risk
that are capable of offering accurate parameter tuning even for
non-differentiable losses and regularizers. Our first approach is based on the
approximation of the dual of \eqref{eq:main-problem}. Our second approach is
based on the smoothing and quadratic approximation of the primal problem
\eqref{eq:main-problem}. The third approach is based  on the proximal
formulation of \eqref{eq:main-problem}.  While the three approaches consider
different approximations that happen in different domains, we will show that
when both $\ell$ and $r$ are twice differentiable, the three frameworks produce
the same ALO formulas, which are also the same as the formulas proposed in
\cite{kamiar2018scalable}.

We use our platforms to obtain concise formulas for several popular examples
including generalized LASSO, support vector machine (SVM) and nuclear norm
minimization. As will be clear from our examples, despite the equivalence of
the three frameworks for smooth loss functions and regularizers, the technical
aspects of deriving ALO formulas have major variations in different examples.
In Remark \ref{remark:three-approaches-strength} we have a short discussion
about the strength of different approaches on different problems.
Finally, we present extensive simulations to confirm the accuracy of our
formulas on various important machine learning models.

\subsection{Other Related Work}
The importance of parameter tuning in learning systems has encouraged many
researchers to study this problem from different perspectives. In addition to
cross validation, several other
approaches have been proposed including Stein's unbiased risk estimate
(SURE), Akaike information criterion (AIC), and Mallow's $C_p$. While AIC is
designed for smooth parametric models, SURE has been extended to emerging
optimization problems, such as generalized LASSO and nuclear norm
minimization \cite{candes2013unbiased, dossal2013degrees, tibshirani2012degrees,
vaiter2017degrees, zou2007degrees}.

Unlike cross validation which approximates the out-of-sample prediction
error, SURE, AIC, and $C_p$ offer estimates for in-sample prediction error
\cite{EOSL:chapter7}. This makes cross validation more appealing for
many learning systems. Furthermore, unlike ALO,
both SURE and $C_p$ only work on linear models (and not generalized linear
models) and their unbiasedness is only guaranteed under the Gaussian model
for the errors. There has been little success in extending SURE beyond this
model \cite{efron2004estimation}. 

Another class of parameter tuning schemes are based on approximate message
passing framework \cite{bayati2013estimating, mousavi2017consistent,
obuchi2016cross}. As pointed out in \cite{obuchi2016cross}, this approach is
intuitively related to LOOCV. It offers
consistent parameter tuning in high-dimensions \cite{mousavi2017consistent, wang2017bridge},
but the results strongly depend on the independence of the elements of $\Xv$. This limits to application of this approach to very specific problems.

\subsection{Organization of the Paper}
Our paper is organized as follows: Section \ref{sec:prelim} contributes to 
some preliminaries which will be uesd later. Section \ref{sec:approximatedual},
\ref{sec:primal-smoothing}, \ref{sec:proxformulation} introduce respectively
the dual approach, primal approach and proximal approach to obtain the ALO
formula. Then in Section \ref{sec:primal-dual-equiv} we prove the equivalence
of the three approaches under the smoothness conditions, followed by a corollary related
to accuracy. All the above sections discuss ALO without
including the intercept term in the model. Thus in Section \ref{sec:intercept}
we address the case when the intercept is contained. We then apply the ALO
approaches introduced in previous sections to several models and obtain
their specific ALO formula in Section \ref{sec:applications}. Experimental
results are presented in Section \ref{sec:experiments}. Finally, after a short discussion in Section \ref{sec:discussion}, we present all the proofs in
Section \ref{sec:proof}.

\subsection{Notation}

Lowercase and uppercase bold letters denote vectors and matrices,
respectively. For subsets $A \subset \{1,2, \ldots, n\}$ and $B \subset \{ 1,2,
\ldots, p\}$ of indices and a matrix $\Xv$, let $\Xv_{A, \cdot}$
and $\Xv_{\cdot, \Bv}$ denote the submatrices that include only
rows of $\Xv$ in $A$, and columns of $\Xv$ in $B$ respectively. Let
$\{a_i\}_{i \in S}$ denote the vector whose components are $a_i$ for $i \in S$. We
may omit $S$, in which case we consider all indices valid in the context.
For a function $f: \mathbb{R} \rightarrow \mathbb{R}$, let
$\dot{f}$, $\ddot{f}$ denote its 1\textsuperscript{st} and
2\textsuperscript{nd} derivatives. For a
vector $\av$, we use $\diag[\av]$ to denote a diagonal matrix
$\Av$ with $A_{ii} = a_i$. Finally,
let $\nabla R$ and $\nabla^2 R$ denote the gradient and Hessian of
a function $R: \mathbb{R}^p \rightarrow \mathbb{R}$.

\section{Preliminaries}\label{sec:prelim}
In this section we describe the problem to be studied in this paper and some
preliminary knowledge needed for subsequent analyses. We start with the
unconstrained learning problems. In Section \ref{ssec:constrainedopt}, we
will discuss the generalization to the constrained ones.

\subsection{Problem Description}
In this paper, we study the statistical learning models in form
\eqref{eq:main-problem}. For each value of $\lambda$, we evaluate the
following LOOCV risk estimate with respect to some error function $d$:
\begin{equation}\label{eq:err-func}
    \loo_{\lambda}: = \frac{1}{n}\sum_{i=1}^n d(y_i, \xv_i^\top \leavei{\estim{\betav}}),
\end{equation}
where $\estimi{\betav}$ is the solution of the leave-$i$-out problem
\begin{equation}\label{eq:leave-i-out}
    \estimi{\betav} := \argmin_{\betav}\; \sum_{j\neq i} \ell(\xv_j^\top\betav; y_j) + \lambda R(\betav).
\end{equation}
Calculating \eqref{eq:leave-i-out} requires training the model $n$ times,
which may be time-consuming in high-dimensions. As an alternative, we
propose an estimator $\surrogi{\betav}$ to approximate $\estimi{\betav}$
based on the full-data estimator $\estim{\betav}$ to reduce the computational
complexity. We consider three frameworks for obtaining $\surrogi{\betav}$, and
denote the corresponding risk estimate by:
\begin{equation*}
    \alo_\lambda: = \frac{1}{n}\sum_{i = 1}^n d(y_i, \xv_i^\top \surrogi{\betav}).
\end{equation*}

The estimates we obtain will be called approximated leave-one-out (ALO) throughout the paper. 

\subsection{Primal and Dual Correspondence}

The objective function of penalized regression problem with loss $\ell$ and
regularizer $R$ is given by:
\begin{equation} \label{eq:methods:primal}
    P(\betav) := \sum_{j = 1}^n \ell(\xv_j^\top \betav; y_j) + R(\betav).
\end{equation}

Here and subsequently, unless necessary, we absorb the value of $\lambda$ into $R$
to simplify the notation. We also consider the Lagrangian dual problem, which can be written in the form:
\begin{equation} \label{eq:methods:dual}
    \min_{\dualv \in \mathbb{R}^n} D(\thetav) := \sum_{j = 1}^n
    \ell^*(-\theta_j; y_j) + R^*(\Xv^\top \thetav),
\end{equation}
where $\ell^*$ and $R^*$ denote the \textit{Fenchel conjugates}\footnote{The Fenchel conjugate $f^*$ of a function $f$ is defined as
$f^*(x):=\sup_y\{\langle x, y\rangle - f(y)\}$.}
of $\ell$ and $R$ respectively. See the derivation in Appendix \ref{sec:dual-derivation}.
It is known that under mild conditions, \eqref{eq:methods:primal}
and \eqref{eq:methods:dual} are equivalent \cite{boyd2004convex}. In this case,
we have the primal-dual correspondence relating the primal optimal
$\estim{\betav}$ and the dual optimal $\estim{\thetav}$:
\begin{equation} \label{eq:primal-dual-correspondence}
\begin{gathered}
    \estim{\betav} \in \partial R^* (\Xv^\top  \estim{\dualv}), \quad
    \Xv^\top \estim{\dualv} \in \partial R(\estim{\betav}), \\
    \xv_j^\top\estim{\betav} \in \partial \ell^* (-\estim{\duals}_j; y_j), \quad
    - \estim{\duals}_j \in \partial \ell(\xv_j^\top \estim{\betav}; y_j),
\end{gathered}
\end{equation}
where $\partial f$ denotes the set of subgradients of a function $f$ with
respect to its first argument. These relations will help us approximate
$\loo_\lambda$ from primal and dual perspectives.

\subsection{Proximal Formulation}\label{ssec:proxformula_intro}
In this section, we review another characterization of $\estim{\betav}$ that will be
used for approximating $\loo_\lambda$. Consider the following definition:
\begin{definition}
    The proximal operator $\proxv_h: \mathbb{R}^p \rightarrow \mathbb{R}^p$ of a
    function $h: \mathbb{R}^p \rightarrow \mathbb{R}$ is defined as
    \begin{equation*}
        \proxv_h(\zv; \tau) := \argmin_{\uv} \frac{1}{2\tau} \| \zv - \uv \|_2^2
        + h(\uv)
    \end{equation*}
\end{definition}
 When $\tau = 1$, we will write $\proxv_h(\zv)$ instead of $\proxv_h(\zv; 1)$ for notational
simplicity. For many modern regularizers $R$, such as LASSO and nuclear norm,
$\proxv_R (\cdot)$ has an explicit expression. We summarize some of the
properties of the proximal operator in the following lemma:

\begin{lemma}\label{lem:proxproperties}
    The proximal operator satisfies the following properties:
    \begin{enumerate}
        \item
            The proximal operator $\proxv_h$ is nonexpansive, i.e.,
            \begin{equation*}
                \|\proxv_h(\zv; \tau) - \proxv_h(\wv; \tau)\|_2^2
                \leq
                \langle \proxv_h(\zv; \tau) - \proxv_h(\wv; \tau), \zv - \wv \rangle.
            \end{equation*}
        \item
            $\proxv_h = (I + \partial h)^{-1}$;
        \item Let $h:\mathbb{R} \rightarrow \mathbb{R}$ be a convex and
            piecewise smooth function with $k$ number of zeroth-order
            singularities\footnote{
                A singular point of a function is called
                $q$\textsuperscript{th} order, if at this point the function is $q$
                times differentiable, but its $(q+1)$\textsuperscript{th} order
                derivative does not exist.
            } $\{v_1, \ldots, v_k\}\subset\mathbb{R}$,
            then $\prox_h(z; \tau)$ takes constant value $v_j$ when
            $z \in [v_j + \tau \dot{h}_-(v_j), v_j + \tau \dot{h}_+(v_j)]$
            with $\dot{h}_{-}$ denoting the left-derivative and $\dot{h}_+$ for the right.
            Note that for different value of $v_j$, the convexity guarantees
            these intervals do not overlap with each other. Further,
            $\prox_h(z; \tau)$ is differentiable as long as $z$ does not lie on the
            boundaries of these intervals;
            \item 
            If $h: \mathbb{R}^p \rightarrow \mathbb{R}$ is a twice
            differentiable convex function, then the Jacobian of $\proxv_{h}$
            exists. In addition, the Jacobian matrix is symmetric and its
            eigenvalues are all between zero and one.
        \item
            A function $\etav:\mathbb{R}^p \rightarrow \mathbb{R}^p$ is a
            proximal operator of a convex function if and only
            if $\etav$ is nonexpansive and a gradient of a convex function;
    \end{enumerate}
\end{lemma}

The proof of the first two claims can be found in \cite{parikh2014proximal}.
Short proofs of the third and fourth parts can be found in Appendix
\ref{proof:proxproperties}. The proof of the last part can be found in
\cite{moreau1965proximite}.

Our interest in the proximal operator stems from the fact that it provides another formulation for evaluating $\estim{\betav}$. More specifically, under some mild conditions, the solution of the primal problem $\estim{\betav}$ is the unique fixed point of the following equation:
\begin{equation}\label{eq:prox-smooth-optimal}
    \estim{\betav} = \proxv_R \Big(\estim{\betav} - \sum_{j=1}^n
    \dot{\ell}(\xv_j^\top\estim{\betav}; y_j)\xv_j\Big).
\end{equation}
In the next three sections we show how the primal, dual and proximal
formulations introduced in \eqref{eq:methods:primal}, \eqref{eq:methods:dual},
and \eqref{eq:prox-smooth-optimal} can be used to approximate LOOCV.

\section{Approximation in the Dual Domain}\label{sec:approximatedual}
In this section, we introduce the dual approach to obtain the ALO formula. We
first explain the idea using LASSO as an example. Then the approach is
extended to general regularzers and general smooth losses.

\subsection{The First Example: LASSO}
\label{sec:example:lasso}

Let us first start with a simple example that illustrates our dual method in deriving
an approximate leave-one-out (ALO) formula for the standard LASSO. The LASSO
estimator, first proposed in \cite{tibshirani1996regression}, can be formulated as
the penalized regression framework in \eqref{eq:methods:primal} by setting
$\ell(\mu; y) = (\mu - y)^2 / 2$, and $R(\betav) = \lambda\norm{\betav}_1$. We recall the general formulation of the dual for penalized regression problems
\eqref{eq:methods:dual}, and note that in the case of the LASSO we have:
\begin{equation*}
    \ell^*(\duals_i; y_i) = \frac{1}{2}(\duals_i - y_i)^2, \\
    \quad
    R^*(\betav) = \begin{cases}
        0 & \text{ if } \norm{\betav}_\infty \leq \lambda, \\
        +\infty & \text{ otherwise.}
    \end{cases}
\end{equation*}
In particular, we note that the solution of the dual problem
\eqref{eq:methods:dual} can be obtained from:
\begin{equation}
    \label{eq:lasso:projection}
    \estim{\dualv} = \projv_{\Delta_X}(\yv).
\end{equation}
Here $\projv_{\Delta_X}$ denotes the projection onto $\Delta_X$, where $\Delta_X$
is the polytope given by:
\begin{equation*}
    \Delta_X
    =
    \{ \dualv \in \RR^n : \norm{\Xv^\top \dualv}_{\infty} \leq \lambda \}.
\end{equation*}
Let us now consider the leave-$i$-out problem. Unfortunately, the dimension of
the dual problem is reduced by 1 for the leave-$i$-out problem, making it
difficult to leverage the information from the full-data solution to help
approximate the leave-$i$-out solution. We propose to augment the leave-$i$-out
problem with a virtual $i$\tsup{th} observation which does not affect the result of
the optimization, but restores the dimensionality of the problem.

More precisely, let $\yv_a$ be the same as $\yv$, except that its
$i$\tsup{th} coordinate is replaced by $\leavei{\estim{y}}_i = \xv_i^\top
\leavei{\estim{\betav}}$, the leave-$i$-out predicted value. We note that the
leave-$i$-out solution $\leavei{\estim{\betav}}$  is also the solution for the following
augmented problem:
\begin{eqnarray}\label{eq:augmentedmethod}
   \min_{\betav \in \mathbb{R}^p}  \sum_{j = 1}^n \ell(\xv_j^\top \betav; y_{a,j}) + R(\betav).
\end{eqnarray}
Let $\leavei{\estim{\dualv}}$ be the corresponding dual solution of
\eqref{eq:augmentedmethod}. Then, by \eqref{eq:lasso:projection}, we know
that
\begin{equation*}
    \leavei{\estim{\dualv}} = \projv_{\Delta_X} (\yv_a).
\end{equation*}
Additionally, the primal-dual correspondence
\eqref{eq:primal-dual-correspondence} gives that $\leavei{\estim{\dualv}}=\yv_a
- \Xv \leavei{\estim{\betav}}$, which is
the residual in the augmented problem, and hence that $\leavei{\estim{\duals}}_i = 0$.
These two features allow us to characterize the leave-$i$-out
predicted value $\leavei{\estim{y}}_i$, as satisfying:
\begin{equation}
    \label{eq:augmented-dual-equation}
    \ev_i^\top \projv_{\Delta_X}\big(\yv - (y_i - \leavei{\estim{y}}_i)
    \ev_i\big) = 0,
\end{equation}
where $\ev_i$ denotes the $i$\tsup{th} standard vector. Solving exactly for the above
equation is in general a procedure that is computationally comparable to
fitting the model, which may be expensive. However, we may attempt to obtain an
approximate solution of \eqref{eq:augmented-dual-equation}
by linearizing the projection operator at the full data solution
$\estim{\dualv}$. The approximate leave-$i$-out fitted value
$\leavei{\surrog{y}}_i$ is thus given by:
\begin{equation}\label{eq:dualsimpleJaco}
    \leavei{\surrog{y}}_i = y_i - \frac{\estim{\duals}_i}{J_{ii}},
\end{equation}
where $\Jv$ denotes the Jacobian of the projection operator $\projv_{\Delta_X}$ at
the full data problem $\yv$. The nonexpansiveness of $\projv_{\Delta_X}$
guarantees the almost everywhere existence of $\Jv$.
Note that $\Delta_X$ is a polytope, and thus the projection onto $\Delta_X$ is
almost everywhere locally affine \cite{tibshirani2012degrees}.
Furthermore, it is straightforward to calculate the Jacobian of
$\projv_{\Delta_X}$. Let $E = \{ j: \abs{\Xv_j^\top \estim{\dualv}} = \lambda \}$ be
the equicorrelation set (where $\Xv_j$ denotes the $j^{\text{th}}$ column of $\Xv$),
then we have that the projection at the full data problem $\yv$ is locally
given by a projection onto the orthogonal complement of the span of $\Xv_{\cdot, E}$,
thus giving $\Jv = \Iv - \Xv_{\cdot, E} (\Xv_{\cdot, E}^\top \Xv_{\cdot,
E})^{-1} \Xv_{\cdot, E}^\top$. We can then obtain $\leavei{\surrog{y}}$ by
plugging $\Jv$ in \eqref{eq:dualsimpleJaco}. The risk of LASSO can be estimated
through $\alo_\lambda = \frac{1}{n}\sum_{i = 1}^n d(y_i, \tilde{y}_i)$

\subsection{General Case}
\label{ssec:dual:general-case}

In this section we extend the dual approach outlined in Section
\ref{sec:example:lasso} to more general loss functions and regularizers.

\paragraph{General regularizers}
Let us first extend the dual approach
to other regularizers, while the loss function remains $\ell(\mu, y) =
\frac{1}{2}(\mu-y)^2$.
In this case the dual problem \eqref{eq:methods:dual} has the following form:
\begin{equation}\label{eq:dual:general-regularizer}
    \min_{\dualv} \frac{1}{2} \sum_{j = 1}^n (\duals_j - y_j)^2 + R^*(\Xv^\top \dualv).
\end{equation}
Note that the optimal value of $\dualv$ is by definition the value of the proximal operator of
$R^*(\Xv^\top\cdot)$ at $\yv$:
\begin{equation*}
    \estim{\dualv} = \proxv_{R^*(\Xv^\top \cdot)}(\yv).
\end{equation*}
Following the argument of Section \ref{sec:example:lasso}, we obtain
\begin{equation}\label{eq:dual:alo-least-square}
    \leavei{\surrog{y}}_i = y_i - \frac{\estim{\duals}_i}{J_{ii}},
\end{equation}
with $\Jv$ now denoting the Jacobian of $\proxv_{R^*(\Xv^\top \cdot)}$. We
note that the Jacobian matrix $\Jv$ exists almost everywhere, because the
non-expansiveness of the proximal operator guarantees its almost-everywhere
differentiability \cite{Combettes2011}. In particular, if the distribution of $\yv$
is absolutely continuous with respect to the Lebesgue measure, $\Jv$
exists with probability 1. This approach is particularly useful when $R$ is a norm, as
its Fenchel conjugate is then the convex indicator of the unit ball of the dual
norm, and the proximal operator reduces to a projection operator. 

In summary, since $\hat{\theta}_i = y_i - \xv_i^\top \estim{\betav}$, the risk of $\estim{\betav}$ can be estimated through the following formula:
\begin{equation} \label{eq:aloformuladualgenralreg}
    \alo_\lambda
    = \frac{1}{n}\sum_{i = 1}^n d(y_i, \tilde{y}_i)
    = \frac{1}{n}\sum_{i = 1}^n d\bigg(y_i, y_i - \frac{y_i - \xv_i^\top
    \estim{\betav}}{J_{ii}}\bigg),
\end{equation}
where $\Jv$ is the Jacobian of $\proxv_{R^*(\Xv^\top \cdot)}$. We calculate $\Jv$ for several popular regularizers in Section \ref{sec:applications}. 

\paragraph{General smooth loss}
Let us now assume we have a convex smooth loss in \eqref{eq:methods:primal}, such as those that
appear in generalized linear models. As we are arguing from a second-order
perspective by considering Newton's method, we will attempt to expand the
loss as a quadratic form around the full data solution. We will thus consider the
approximate problem obtained by expanding $\ell^*$ around the dual optimal
$\estim{\dualv}$ of \eqref{eq:methods:dual}:
\begin{equation}\label{eq:dual-general-optimal}
    \min_{\dualv} \frac{1}{2}\sum_{j = 1}^n \ddot{\ell}^*(-\estim{\duals}_j; y_j)
    \Bigg(
        \duals_j - \estim{\duals}_j
        - \frac{\dot{\ell}^*(-\estim{\duals}_j; y_j)}{\ddot{\ell}^*(-\estim{\duals}_j; y_j)}
    \Bigg)^2 + R^*(\Xv^\top \dualv).
\end{equation}

The constant term has been removed from \eqref{eq:dual-general-optimal} for
simplicity. We note that we have reduced the problem to a problem with a
weighted $\ell_2$ loss which may be further reduced to a simple $\ell_2$
problem by a change of variable and a rescaling of $\Xv$. Indeed, let $\Kv$ be
the diagonal matrix such that $K_{jj} = \sqrt{\ddot{\ell}^*(-\estim{\duals}_j;
y_j)}$, and note that we have: $\dot{\ell}^*(-\estim{\duals}_j; y_j) =
\xv_j^\top \estim{\betav} := \estim{y}_j$ by the primal-dual correspondence
\eqref{eq:primal-dual-correspondence}. Consider the change of variable $\uv =
\Kv \dualv$ to obtain:
\begin{equation*}\label{eq:dual:alo-general}
    \min_{\uv} \frac{1}{2} \sum_{j = 1}^n \left( u_j
    - \frac{\estim{\duals}_j \ddot{\ell}^*(-\estim{\duals}_j; y_j) + \estim{y}_j}
                       {\sqrt{\ddot{\ell}^*(-\estim{\duals}_j; y_j)}}
    \right)^2 + R^*(\Xv^\top \Kv^{-1} \uv).
\end{equation*}
We may thus reduce to the $\ell_2$ loss case in \eqref{eq:dual:general-regularizer}
with a modified $\Xv$ and $\yv$:
\begin{equation}\label{eq:dual:general-loss}
    \Xv_u = \Kv^{-1} \Xv,
    \quad
    \yv_u = \left\{\frac{\estim{\duals}_j \ddot{\ell}^*(-\estim{\duals}_j; y_j) + \estim{y}_j}
    {\sqrt{\ddot{\ell}^*(-\estim{\duals}_j; y_j)}}\right\}_j.
\end{equation}
Similar to \eqref{eq:dual:alo-least-square}, the ALO formula in the case of
general smooth loss can be obtained as $\leavei{\surrog{y}_{i}} =
K_{ii}\leavei{\surrog{y}_{u,i}}$, with
\begin{equation}\label{eq:dual:alo-general-loss}
    \leavei{\surrog{y}_{u,i}}
    =
    y_{u, i} - \frac{K_{ii}\estim{\theta}_i}{J_{ii}},
\end{equation}
where $\Jv$ is the Jacobian of $\proxv_{R^*(\Xv_u^\top\cdot)}$. 

In summary, we can calculate $\alo_\lambda$ in the following way. Given
$\estim{\betav}$, calculate the dual variable $\estim{\dualv}$ from
\eqref{eq:primal-dual-correspondence}, and the diagonal matrix $\Kv$, such that
$K_{jj} = \sqrt{\ddot{\ell}^*(-\estim{\duals}_j;y_j)}$. Then, compute
$\yv_{u,j}$ using \eqref{eq:dual:general-loss}. Finally,
$\leavei{\surrog{y}_{i}} = K_{ii}(y_{u, i} -
\frac{K_{ii}\estim{\theta}_i}{J_{ii}})$, where $\Jv$ is the Jacobian of
$\proxv_{R^*(\Xv_u^\top\cdot)}$. The $\alo_\lambda$ formula is then obtained through
\begin{equation*}
  \alo_\lambda= \frac{1}{n}\sum_{i = 1}^n d(y_i, \leavei{\surrog{y}_{i}}).
\end{equation*}

\section{Approximation in the Primal Domain}
\label{sec:primal-smoothing}

The dual approach is typically powerful for models with smooth losses and
norm-type regularizers, such as the LASSO. However, it might be difficult to
carry out the calculations for other problems. Hence, in this section we
introduce our second method for finding $\alo_\lambda$.
    
\subsection{Smooth Loss and Smooth Regularizer}\label{ssec:primal-smooth}
Recall that to obtain $\loo_\lambda$ we need to solve
\begin{equation}\label{eq:leaveoneoutopt_1}
    \estimi{\betav} := \argmin_{\betav}\; \sum_{j\neq i}
    \ell(\xv_j^\top\betav; y_j) + R(\betav).
\end{equation}
Assuming $\estimi{\betav}$ is close to $\estim{\betav}$, we can take one
\textit{Newton step} from $\estim{\betav}$ towards $\estimi{\betav}$ to obtain
its approximation $\surrogi{\betav}$ as:
\begin{equation}\label{eq:primal-alo-smooth0}
    \surrogi{\betav}
    =
    \estim{\betav} + \bigg[\sum_{j\neq i} \xv_j \xv_j^\top \ddot{\ell}
    (\xv_j^\top\estim{\betav}; y_j) + \nabla^2 R(\estim{\betav})\bigg]^{-1}
    \xv_i\dot{\ell}(\xv_i^\top\estim{\betav}; y_i).
\end{equation}
By employing the matrix inversion lemma \cite{hager1989updating} we obtain:
\begin{equation}\label{eq:primal-alo-smooth}
    \xv_i^\top\surrogi{\betav}
    =
    \xv_i^\top\estim{\betav}
    +
    \frac{H_{ii}}{1 - H_{ii}\ddot{\ell}(\xv_i^\top\estim{\betav};
    y_i)}\dot{\ell}(\xv_i^\top\estim{\betav}; y_i),
\end{equation}
where
\begin{equation}\label{eq:smooth-H}
    \Hv = \Xv\big[\Xv^\top \diag[\{\ddot{\ell}(\xv_i^\top\estim{\betav};
    y_i)\}_i]\Xv + \nabla^2 R(\estim{\betav})\big]^{-1}\Xv^\top.
\end{equation}

This is the formula reported in \cite{kamiar2018scalable}. By calculating
$\estim{\betav}$ and $\Hv$ in advance, we can cheaply approximate the
leave-$i$-out prediction for all $i$ and efficiently evaluate the LOOCV risk.
On the other hand, in order to use the above strategy, twice differentiability
of both the loss and the regularizer is necessary in a neighborhood of
$\estim{\betav}$. However, this assumption is violated for many machine
learning models including LASSO, nuclear norm, and SVM.
In the next two sections, we introduce a smoothing technique which lifts the scope
of the above primal approach to nondifferentiable losses and regularizers.

\subsection{Nonsmooth Loss and Smooth Regularizer}\label{ssec:nonsmooth-loss}
In this section we study the piecewise smooth loss functions and twice
differentiable regularizers. Such problems arise for instance in SVM \cite{cortes1995support}
and robust regression \cite{huber1973robust}. Below we assume the loss $\ell$
is piecewise twice differentiable with $k$ zeroth-order singularities $v_1,
\ldots, v_k \in \mathbb{R}$.
The existence of singularities prohibits us from directly applying strategies in
\eqref{eq:primal-alo-smooth0} and \eqref{eq:primal-alo-smooth}, where twice
differentiability of $\ell$ and $R$ is necessary. A natural solution is to
first smooth out the loss function $\ell$, then apply the framework in the
previous section to the smoothed version and finally reduce the smoothness to
recover the ALO formula for the original nonsmooth problem. As the first step, consider the following smoothing idea:
\begin{equation*}
    \ell_h(\mu; y) =: \frac{1}{h} \int \ell(u; y) \phi ((\mu- u)/h) du,
\end{equation*}
where $h>0$ is a parameter controlling the smoothness of $\ell_h$ and $\phi$ is
a symmetric, infinitely many times differentiable function with the following
properties:
\begin{description}
    \item[] \qquad \textit{Normalization}:
        $\int \phi(w)dw = 1$, $\phi(w) \geq 0$, $\phi(0)>0$;
    \item[] \qquad \textit{Compact support}:
        $\mathrm{supp}(\phi) = [-C, C]$ for some $C>0$.
\end{description}

Now plug in this smooth version $\ell_h$ into \eqref{eq:leaveoneoutopt_1} to
obtain the following formula from \eqref{eq:primal-alo-smooth0}:
\begin{equation}\label{eq:smoothedbeta}
    \surrogi{\betav}_h
    :=
    \estim{\betav}_h + \bigg[\sum_{j\neq i} \xv_j \xv_j^\top\ddot{\ell}_h
    (\xv_j^\top\estim{\betav}_h; y_j) + \nabla^2 R(\estim{\betav}_h)\bigg]^{-1}
    \xv_i\dot{\ell}_h(\xv_i^\top\estim{\betav}_h; y_i).
\end{equation}
where $\estim{\betav}_h$ is the minimizer on the full data from loss
$\ell_h$ and $R$.  $\surrogi{\betav}_h$ is a good approximation to the leave-$i$-out
estimator $\estimi{\betav}_h$ based on smoothed loss $\ell_h$.

Setting $h \rightarrow 0$, we have that $\ell_h(\mu, y)$ converges to
$\ell(\mu, y)$ uniformly in the region of interest (see Appendix
\ref{append:ssec:property-kernel-smoothing} for the proof), implying that
$\lim_{h\rightarrow 0} \surrogi{\betav}_h$ serves as a good estimator of
$\lim_{h\rightarrow 0}\estimi{\betav}_h$, which is heuristically close to the
true leave-$i$-out $\estimi{\betav}$. Equation \eqref{eq:smoothedbeta} can be
simplified in the limit $h \rightarrow 0$. We define the sets of indices $V$
and $S$ for the samples at singularities and smooth parts respectively:
\begin{align}\label{eq:indexset}
    V &:= \big\{ j: \xv_j^\top \estim{\betav} = v_t \text{ for some } t \in
    \{1, \ldots, k\} \big\}, \nonumber \\
    S &:= \{ 1, \dotsc, n \} \setminus V.
\end{align}

The following assumptions are necessary to derive the limit as $h \rightarrow 0$.
\begin{assumption}\label{assum:nonsmooth-loss-assump1}
    We need the following assumptions on $\ell$, $R$ and $\estim{\betav}$:
    \begin{enumerate}
        \item
            $\ell$ is locally Lipschitz, that is, for any $A > 0$, for any $x, y
            \in [-A, A]$, we have $|\ell(x) - \ell(y)| \leq L_A|x - y|$, where $L_A$ is a
            constant depends only on $A$.
        \item
            $\lambda_{\min}(\Xv_{V}\Xv_{V}^\top) > 0$.
        \item
            $\estim{\betav}$ is the unique minimizer.
        \item
            Whenever $\xv_j^\top\estim{\betav} = v \in K$, the subgradient of
            $\ell$ at $\xv_j^\top\estim{\beta}$, $g_\ell(\xv^\top \estim{\betav})$ satisfies
            $g_\ell(\xv^\top \estim{\betav}) \in (\ell_-(v), \ell_+(v))$.
        \item
            $R$ is coercive in the sense that $|R(\betav)| \rightarrow \infty$ as
            $\|\betav\| \rightarrow \infty$.
    \end{enumerate}
\end{assumption}

We characterize the limit of $\xv_i^\top\surrogi{\betav}_h$ below.
\begin{theorem}\label{thm:nonsmooth-loss-approx}
    Under Assumptions \ref{assum:nonsmooth-loss-assump1}, as $h \rightarrow 0$,
    \begin{equation*}
        \xv_i^\top\surrogi{\betav}_h \rightarrow  \xv_i^\top\estim{\betav}+
        a_i g_{\ell,i},
    \end{equation*}
    where
    \begin{align*}\label{eq:nonsmooth-H}
        a_i &= \begin{cases}
            \frac{W_{ii}}{1 - W_{ii}\ddot{\ell}(\xv_i^\top\estim{\betav};y_i)}
            & \text{ if } i \in S, \\
            \frac{1}{[(\Xv_{V\cdot}\Yv^{-1}\Xv_{V\cdot}^\top)^{-1}]_{ii}} & \text{ if } i \in V, \\
        \end{cases} \\
        \Yv &= \nabla^2 R(\estim{\betav})
        + \Xv_{S,\cdot}^\top \diag[\{\ddot{\ell}(\xv_j^\top\estim{\betav})\}_{j
        \in S}] \Xv_{S,\cdot}, \\
        W_{ii} &= \xv_i^\top \Yv^{-1} \xv_i - \xv_i^\top
        \Yv^{-1}\Xv_{V,\cdot}^\top(\Xv_{V,\cdot}
        \Yv^{-1}\Xv_{V,\cdot}^\top)^{-1}\Xv_{V,\cdot} \Yv^{-1}\xv_i.
    \end{align*}
    
    For $i \in S$, $g_{\ell,i} = \dot{\ell}(\xv_i^\top\estim{\betav}; y_i)$, and for $i \in V$, we have:
    \begin{equation*}
        \gv_{\ell,V} = (\Xv_{V,\cdot}\Xv_{V,\cdot}^\top)^{-1}\Xv_{V,\cdot}
        \Bigg[\nabla R(\estim{\betav}) - \sum_{j \in
        S}\xv_j\dot{\ell}(\xv_j^\top \estim{\betav}; y_j)\Bigg]. 
    \end{equation*}
\end{theorem}
The conditions and proof of Theorem \ref{thm:nonsmooth-loss-approx} can be
found in the Section \ref{append:ssec:primal-nonsmooth-loss}. Based on this
theorem we can obtain the following $\alo_\lambda$ formula:
\begin{equation*} 
    \alo_\lambda = \frac{1}{n}\sum_{i = 1}^n d(y_i, \xv_i^\top\hat{\betav} + a_i g_{\ell, i}),
\end{equation*}
We will apply this formula to the example of hinge loss used for SVM in Section \ref{ssec:linearsvm}.

\subsection{Nonsmooth Separable Regularizer and Smooth Loss}\label{ssec:nonsmooth-regularizer}

The smoothing technique proposed in the last section can also handle many
nonsmooth regularizers. In this section we focus on separable regularizers $R$,
defined as $R(\betav) = \sum_{l=1}^p r(\beta_l)$, where $r: \mathbb{R}
\rightarrow \mathbb{R}$ is piecewise twice differentiable with finite number of
zeroth-order singularities in $v_1, \ldots, v_k \in \mathbb{R}$ (examples on
non-separable regularizers are studied in Section \ref{sec:applications}.)
We further assume the loss function $\ell$ to be twice differentiable and denote by
$A = \big\{l: \estim{\beta}_l \neq v_t, \text{ for any } t \in \{1, \ldots, k\}
\big\}$ the active set.

For the coordinates of $\estim{\betav}$ that lie in $A$, our objective
function, constrained to these coordinates, is locally twice differentiable.
Hence we expect $\estimi{\betav}_A$ to be well approximated by the ALO
formula using only $\estim{\betav}_A$. On the other hand, components not in
$A$ are trapped at singularities. Thus as long as they are not on the
boundary of being in or out of the singularities, we expect these locations
of $\estimi{\betav}$ to stay at the same values.
Technically, consider a similar smoothing scheme for $r$:
\begin{equation*}
    r_h(w) = \frac{1}{h}\int r(u)\phi((w - u) / h) du,
\end{equation*}
and let $R_h(\betav) = \sum_{l=1}^p r_h(\beta_l)$. We then consider the ALO
formula of Model \eqref{eq:leaveoneoutopt_1} with regularizer $R_h$:
\begin{equation}\label{eq:nonsmooth-reg:smoothedbeta}
    \surrogi{\betav}_h := \estim{\betav}_h
    + \bigg[\sum_{j\neq i} \xv_j\xv_j^\top\ddot{\ell}(\xv_j^\top\estim{\betav}_h; y_j)
    + \nabla^2 R_h(\estim{\betav}_h)\bigg]^{-1}
    \xv_i\dot{\ell}(\xv_i^\top\estim{\betav}_h; y_i).
\end{equation}
We need the following assumptions to obtain the limiting case as $h \rightarrow
0$.
\begin{assumption}\label{assum:nonsmooth-reg-assump1}
    We will need the following assumptions on the problem.
    \begin{enumerate}
        \item
            $r$ is locally Lipschiz in the sense that, for any $C > 0$, and for any
            $x, y \in [-C, C]$, we have $|r(x) - r(y)| \leq L_C|x - y|$, where
            $L_C$ is a constant that only depends on $C$;
        \item
            $\estim{\betav}$ is the unique minimizer of
            \eqref{eq:nonsmooth-reg-main};
        \item
            When $\estim{\beta}_l = v \in K$, the subgradient $g_r(\estim{\beta}_l)$
            of $r$ at $\estim{\beta}_l$ satisfies
            $g_r(\estim{\beta}_l) \in (\dot{r}_-(v), \dot{r}_+(v))$.
        \item
            $r$ is coercive in the sense that $|r(z)| \rightarrow \infty$ as
            $|z| \rightarrow \infty$. \\
    \end{enumerate}
\end{assumption}

Setting $h \rightarrow 0$, under Assumption \ref{assum:nonsmooth-reg-assump1},
\eqref{eq:nonsmooth-reg:smoothedbeta} reduces to a
simplified formula which heuristically serves as a good approximation to the true
leave-$i$-out estimator $\estimi{\betav}$, stated as the following theorem:
\begin{theorem}\label{thm:nonsmooth-reg-approx}
    Under Assumption \ref{assum:nonsmooth-reg-assump1}, as $h \rightarrow 0$,
    \begin{equation*}\label{eq:nonsmooth-reg:1}
        \xv_i^\top\surrogi{\betav}_h
        \rightarrow
        \xv_i^\top\estim{\betav}
        +
        \frac{H_{ii}\dot{\ell}(\xv_i^\top\estim{\betav}; y_i)}{1 -
        H_{ii}\ddot{\ell}(\xv_i^\top\estim{\betav}; y_i)},
    \end{equation*}
    with
    \begin{equation}\label{eq:nonsmooth-reg:2}
        \Hv = \Xv_{\cdot,A}\big[\Xv_{\cdot,A}^\top
        \diag[\{\ddot{\ell}(\xv_i^\top\estim{\betav}; y_i)\}_i]\Xv_{\cdot,A} +
        \nabla^2 R(\estim{\betav}_{A})\big]^{-1}\Xv_{\cdot,A}^\top.
    \end{equation}
    \normalsize
\end{theorem}
The conditions and proof of Theorem \ref{thm:nonsmooth-reg-approx} can be found in the
Section \ref{append:ssec:primal-nonsmooth-reg}. Based on this Theorem we can
obtain the following formula for $\alo_\lambda$ (in case of non-differentiable
regularizers):
\begin{equation} \label{eq:aloformulaprimalgenreg}
    \alo_\lambda = \frac{1}{n}\sum_{i = 1}^n d\bigg(y_i, \xv_i^\top\estim{\betav}
    +
    \frac{H_{ii}\dot{\ell}(\xv_i^\top\estim{\betav}; y_i)}{1 - H_{ii}\ddot{\ell}
    (\xv_i^\top\estim{\betav}; y_i)}\bigg),
\end{equation}
where $\Hv$ is given by \eqref{eq:nonsmooth-reg:2}. We will see how this method can be used for non-separable regularizers, such as nuclear norm, in Section \ref{sec:applications}. 

\begin{remark}
Note that if we use \eqref{eq:aloformulaprimalgenreg} for LASSO we obtain the
same formula as the one we derived from the dual approach in Section
\ref{sec:example:lasso}.
\end{remark}

\begin{remark}
    For nonsmooth problems, higher order singularities do not cause issues: the
    set of tuning values which cause $\estim{\beta}_l$ (for regularizer) or
    $\xv_j^\top\estim{\betav}$ (for loss) to fall at those higher order
    singularities has measure zero.
\end{remark}

\begin{remark}
    For both nonsmooth losses and regularizers, we need to invert some matrices
    in the ALO formula. Although the invertibility does not seem guaranteed in
    the general formula, as we apply ALO to specific models, the structures of
    the loss and/or the regularizer ensures this invertibility. For example,
    for LASSO, we have the size of the equi-correlation set $|E| \leq
    \min(n, p)$ under weak conditions on $\yv$ and $\Xv$.
    \cite{tibshirani2013lasso}.
\end{remark}

\section{Approximation with Proximal Formulation}\label{sec:proxformulation}
The primal and dual formulas for approximating $\loo_\lambda$ cover a large
number of optimization problems. However, carrying out the calculations
involved in these two methods is still challenging for certain classes of
optimization problems, such as constrained optimization problems we discussed in the introduction. Hence, in this section, we introduce our
third approach which is based on the proximal formulation. We will later prove
that for smooth losses and regularizers this method is equivalent to the primal
formulation and the dual formulation.

\subsection{Smooth Loss and Regularizer}\label{ssec:smoothprox}
In this section, we start with twice differentiable loss functions and
regularizers. As discussed in Section \ref{sec:prelim}, $\estimi{\betav}$ is
the unique solution of the following fixed point equation:
\begin{equation*}
    \estimi{\betav} = \proxv_R\bigg(\estimi{\betav} - \sum_{j\neq i}
    \dot{\ell}(\xv_j^\top\estimi{\betav}; y_j)\xv_j\bigg).
\end{equation*}
Since $\estimi{\betav}$ is close to $\estim{\betav}$, we can obtain a good
approximation of $\estimi{\betav}$ by linearizing
$\proxv_R\Big(\estimi{\betav} - \sum_{j \neq i} \dot{\ell}
(\xv_j^\top\estimi{\betav}; y_j)\xv_j \Big)$ at $\estim{\betav}$. Since the regularizer is twice differentiable, according to Lemma \ref{lem:proxproperties}, $\proxv_R$ is a differentiable function. Let
 $\Jv$ denote the Jacobian of $\proxv_R$ at $\estim{\betav} -
\sum_{j=1}^n\dot{\ell}(\xv_j^\top\estim{\betav}; y_j)\xv_j$. The following Newton
step for finding root of equation systems enables us to obtain an approximation of $\estimi{\betav}$. 
\begin{align*}
    \estimi{\betav}
    &=
    \proxv_R\bigg(\estimi{\betav} - \sum_{j\neq i}
    \dot{\ell}(\xv_j^\top\estimi{\betav}; y_j)\xv_j\bigg) \nonumber \\
    &\approx
    \proxv_R\bigg(\estim{\betav} - \sum_{j=1}^n
    \dot{\ell}(\xv_j^\top\estim{\betav}; y_j)\xv_j\bigg) + \Jv
    \bigg(\estimi{\betav} -
    \sum_{j\neq i}\dot{\ell}(\xv_j^\top\estimi{\betav}; y_j)\xv_j -
    \estim{\betav} + \sum_{j=1}^n \dot{\ell}(\xv_j^\top\estim{\betav};
    y_j)\xv_j \bigg) \nonumber \\
    &\approx
    \estim{\betav} + \Jv \bigg(\Iv - \sum_{j\neq
    i}\ddot{\ell}(\xv_j^\top\estim{\betav};
    y_j)\xv_j\xv_j^\top\bigg)(\estimi{\betav} - \estim{\betav}) + \Jv \xv_i
    \dot{\ell}(\xv_i^\top\estim{\betav}; y_i).
\end{align*}
Using this heuristic argument we obtain the following approximation
$\surrogi{\betav}$ for $\estimi{\betav}$:
\begin{equation} \label{eq:prox-approx-formula0}
    \surrogi{\betav}
    =
    \estim{\betav} + \bigg[\Iv - \Jv \bigg(\Iv - \sum_{j\neq i} \ddot{\ell}
    (\xv_j^\top\estim{\betav}; y_j)\xv_j\xv_j^\top\bigg)\bigg]^{-1}
    \Jv \xv_i \dot{\ell}(\xv_i^\top\estim{\betav}; y_i).
\end{equation}
Define
\begin{equation*}
    \Gv := \Iv - \Jv + \Jv\Xv^\top
    \diag\big[ \{\ddot{\ell}(\xv_j^\top\estim{\betav}; y_j)\}_j\big]\Xv.
\end{equation*}
Assuming $\Gv$ is invertible, one can use the matrix inversion lemma to obtain  
\begin{align*}
    & \xv_i^\top \bigg[\Iv - \Jv \bigg(\Iv - \sum_{j\neq i}
    \ddot{\ell}(\xv_j^\top\estim{\betav}; y_j)\xv_j\xv_j^\top\bigg)\bigg]^{-1}\Jv
    \xv_i\dot{\ell}(\xv_i^\top\estim{\betav}; y_i) \nonumber \\
    =&
    \xv_i^\top \bigg[\Gv^{-1} + \frac{ \Gv^{-1}\Jv\xv_i
    \xv_i^\top\Gv^{-1}}{\ddot{\ell}^{-1}(\xv_i^\top \estim{\betav}; y_i) -
    \xv_i^\top\Gv^{-1}\Jv\xv_i} \bigg]\Jv\xv_i\dot{\ell}(\xv_i^\top\estim{\betav};
    y_i) \nonumber \\
    =&
    \frac{ \xv_i^\top\Gv^{-1}\Jv\xv_i }{1 - \xv_i^\top\Gv^{-1}\Jv\xv_i
    \ddot{\ell}(\xv_i^\top \estim{\betav}; y_i)}
    \dot{\ell}(\xv_i^\top\estim{\betav}; y_i).
\end{align*}
Hence, our final approximation of $\xv_i^\top\estimi{\betav}$ is given by
\begin{equation}\label{eq:prox-smooth-alo}
    \xv_i^\top \surrogi{\betav}
    =
    \xv_i^\top\estim{\betav} + \frac{H_{ii}}{1 -
    H_{ii}\ddot{\ell}(\xv_i^\top \estim{\betav}; y_i)}
    \dot{\ell}(\xv_i^\top \estim{\betav}; y_i),
\end{equation}
where 
\begin{equation}\label{eq:prox-smooth-H}
    \Hv := \Xv\big(\Jv \Xv^\top
    \diag\big[\{\ddot{\ell}(\xv_j^\top\estim{\betav}; y_j)\}_j\big] \Xv
    + \Iv - \Jv\big)^{-1} \Jv \Xv^\top.
\end{equation}
In summary, the $\alo_\lambda$ formula is given by
\begin{equation*}
    \alo_\lambda = \frac{1}{n}\sum_{i = 1}^n d\bigg(y_i, \xv_i^\top\estim{\betav}
    +
    \frac{H_{ii}\dot{\ell}(\xv_i^\top\estim{\betav}; y_i)}{1 -
    H_{ii}\ddot{\ell}(\xv_i^\top\estim{\betav}; y_i)}\bigg). 
\end{equation*}

Even though we used several heuristic steps to obtain this formula, in Section
\ref{sec:primal-dual-equiv}, we will connect this formula with those derived
from the primal and dual perspectives and prove the accuracy of this formula.

\subsection{Generalization to Nonsmooth Regularizer}
\label{ssec:prox-nonsmooth}

In this section, we handle non-differentiable regularizers using the approach
developed in Section \ref{ssec:smoothprox}. Here we consider separable
nonsmooth regularizers where $ R(\betav) = \sum_{j=1}^p r(\beta_j)$, while
similar technique can be used in more general scenarios. Suppose that
$r$ has $k$ zeroth-order singularities $\{v_1, \ldots, v_k\}$. To use
\eqref{eq:prox-smooth-alo} and \eqref{eq:prox-smooth-H}, we apply the same
smoothing scheme introduced in Section \ref{ssec:nonsmooth-regularizer} to
$\prox_r$ and obtain its smoothed version $\prox^h_r$:
\begin{equation*}
    \prox_{r}^h (t) = \frac{1}{h}\int \prox_r(u) \phi((t - u) / h) du.
\end{equation*}

\begin{lemma}\label{lem:smoothedprox}
    $\prox^h_r$ satisfies the following conditions:
    \begin{enumerate}
        \item
            $\prox_r^h (t)$ is also a proximal operator of a convex function;
        \item
            $\sup_{t \in \mathbb{R}} |\prox_r^h (t) - \prox_r(t)|
            \leq h\int |u|\phi(u)du$.
    \end{enumerate}
\end{lemma}

Refer to Section \ref{sec:appendixprox} for the proof of this lemma.
Let $\proxv_R^h(\zv)$ denote the vector of $\big(\prox_r^h(z_1), \ldots,
\prox_r^h(z_p)\big)$ and $\estim{\betav}_h$ denote the fixed point solution of
the following equation:
\begin{equation*}
    \estim{\betav}_h = \proxv_R^h \bigg(\estim{\betav}_h - \sum_{j=1}^n\xv_j
    \dot{\ell}(\xv_j^\top \estim{\betav}_h; y_j)\bigg).
\end{equation*}
Note that since $\prox_r^h(t)$ is also a proximal operator of a convex
function, $\estim{\betav}_h$ is a solution of a convex optimization
problem, hence well-defined. We can now approximate the LOOCV for this new
optimization problem using the methods in Section \ref{ssec:smoothprox}.
Let $\Jv_h$ denote the Jacobian of $\proxv_R^h$ at
$\estim{\betav}_h - \sum_{j=1}^n \xv_j \dot{\ell}(\xv_j^\top
\estim{\betav}_h; y_j)$. We then obtain the ALO formula for the smoothed
formulation as
\begin{equation}\label{eq:prox-smoothing-alo}
    \xv_i^\top \surrogi{\betav}_h
    =
    \xv_i^\top\estim{\betav}_h + \frac{H^h_{ii}}{1 -
    H^h_{ii}\ddot{\ell}(\xv_i^\top \estim{\betav}_h; y_i)}
    \dot{\ell}(\xv_i^\top \estim{\betav}_h; y_i),
\end{equation}
where 
\begin{equation}\label{eq:prox-smoothing-H}
    \Hv^h = \Xv\big(\Jv_h \Xv^\top \diag[\{\ddot{\ell}(\xv_j^\top
    \estim{\betav}_h; y_j)\}_j]\Xv + \Iv - \Jv_h \big)^{-1} \Jv_h \Xv^\top.
\end{equation}

We expect this to be a good estimate of the risk when $h$ is small. Below we
summarize how formula \eqref{eq:prox-smoothing-alo} and \eqref{eq:prox-smoothing-H}
is simplified for $h \rightarrow 0$. Notice the separability of $R$ implies that $\Jv_h =
\diag[\dot{\prox}^h_r(\estim{\beta}_{h,k} - \sum_{j} x_{jk} \dot{\ell}
(\xv_j^\top\estim{\betav}_h; y_j))]$. Similar to the primal approach we need to
let $h \rightarrow 0$ and obtain the limiting formula. Toward this goal we
need to make the following assumptions.

\begin{assumption}\label{assump:prox-smoothing}
    \begin{enumerate}
        \item
            The true minimizer $\estim{\betav}$ is the unique solution of
            \eqref{eq:prox-smooth-optimal}.
        \item
            Let $E = \big\{i: \estim{\betav}_i \in \{v_1, \ldots, v_k\} \big\}$.
            If $k \in E$ and $\estim{\beta}_k = v_m$, we assume
            $ \estim{\beta}_k - \sum_{j=1}^n x_{jk}\dot{\ell}
            (\xv_j^\top\estim{\betav}; y_j) \in (v_m + \dot{r}_-(v_m), v_m +
            \dot{r}_+(v_m))$; For any $k \notin E$, $\estim{\beta}_k
            - \sum_{j=1}^n x_{jk}\dot{\ell} (\xv_j^\top\estim{\betav}; y_j)$
            does not lie on the boundary of any of the above intervals.
    \end{enumerate}
\end{assumption}

Note that the boundaries of $(v_m + \dot{r}_-(v_m), v_m + \dot{r}_+(v_m))$ are
the set of non-differentiable points of the proximal operator. Hence, the
second assumption implies that for each $k=1, \ldots, p$, in a small
neighborhood of $\estim{\betav}_k - \sum_{j=1}^n\xv_{jk} \dot{\ell}(\xv_j^\top
\estim{\betav}; y_j)$, $\prox_r$ is differentiable.

\begin{theorem}\label{thm:prox-smoothing}
    Under Assumptions \ref{assump:prox-smoothing}, we have
\begin{equation*}
    \lim_{h \rightarrow 0} \xv_i^\top \surrogi{\betav}_h
    =
    \xv_i^\top\estim{\betav}
    +
    \frac{H_{ii}}{1 - H_{ii}\ddot{\ell}(\xv_i^\top
    \estim{\betav}; y_i)} \dot{\ell}(\xv_i^\top \estim{\betav}; y_i),
\end{equation*}
where
\begin{equation}\label{eq:hdef:proximal}
    \Hv
    =
    \Xv_{\cdot, E} \big( \Jv_{E,E} \Xv_{\cdot, E}^\top
    \diag[\{\ddot{\ell}(\xv_j^\top\betav; y_j)\}_j] \Xv_{\cdot, E}
    + \Iv_{E,E} - \Jv_{E,E} \big)^{-1} \Jv_{E,E} \Xv_{\cdot,E}^\top.
\end{equation}

\end{theorem}

The proof of this theorem can be found in Section
\ref{append:proof:thm:prox-smoothing}. Note that this theorem leads to the
following $\alo_\lambda$ formula:
\begin{equation*} 
    \alo_\lambda = \frac{1}{n}\sum_{i = 1}^n d\bigg(y_i, \xv_i^\top\estim{\betav}
        +
        \frac{H_{ii}\dot{\ell}(\xv_i^\top\estim{\betav}; y_i)}{1 -
        H_{ii}\ddot{\ell}(\xv_i^\top\estim{\betav}; y_i)}\bigg), 
\end{equation*}
where $\Hv$ is defined in \eqref{eq:hdef:proximal}.

\subsection{Generalization to Constrained Optimization Problems}
\label{ssec:constrainedopt}

The proximal approach developed in the last two sections enables us to study
more general problems of the form:
\begin{equation}\label{eq:constraintedoptimization}
    \min_{\betav} \; \sum_{j=1}^n \ell(\xv_j^\top \betav; y_j)
    + R(\betav),
    \quad
    \text{subject to } \betav \in \calC.
\end{equation}
where $\calC$ is a closed convex set. Simple examples of $\calC$ include positive orthant (when the elements of $\betav$ are known to be positive), or the cone of positive semi-definite matrices for covariance matrices. In this section, we consider the case where both the loss and the regularizer are twice differentiable. We can formulate this optimization problem as
\begin{equation*}
    \min_{\betav} \; \sum_{j=1}^n \ell(\xv_j^\top \betav; y_j)
    + R(\betav) + i_\calC(\betav), 
\end{equation*}
where $ i_\calC(\betav)$ denotes the convex indicator function of $\calC$.
According to the proximal formulation, the optimizer $\estim{\betav}$ of this problem satisfies
\begin{equation*}
    \estim{\betav} = \projv_{\calC}\bigg(\estim{\betav} - \sum_{j=1}^n \xv_j
    \dot{\ell}(\xv_j^\top \estim{\betav}; y_j) - \nabla R(\estim{\betav})\bigg)
\end{equation*}
where $\projv_{\calC}$ is the proximal operator of $ i_\calC(\betav)$ or equivalently the projection operator onto the set $\calC$. The leave-$i$-out problem optimizer also satisfies
\begin{equation*}
    \estimi{\betav} = \projv_{\calC}\bigg(\estimi{\betav}
    - \sum_{j\neq i} \xv_j \dot{\ell}(\xv_j^\top \estimi{\betav}; y_j)
    - \nabla R(\estimi{\betav})\bigg)
\end{equation*}

Note that $\projv_{\calC}$ is not necessarily a smooth function, unless
$\calC=\mathbb{R}^p$ or affine. However, since the projection is a Lipschitz function, it is differentiable almost everywhere \cite{heinonen2005lectures}. The following lemma helps us understand the singularity points of the projection operator for a general class of convex sets. 

\begin{lemma}[\cite{fitzpatrick1982differentiability}]
    Let $\partial \calC$ denote the boundary of the set $\calC$. If $\partial
    \calC$ is $C^k$,\footnote{
        $\partial\calC$ is $C^k$ means there is a locally 1-to-1 mapping $h$ from
        $\partial\calC$ to $\mathbb{R}^{m}$ for some $m$ such that $h$ is
        $k$-times differentiable.
    }then $\projv_{\calC}$ is at least $(k-1)$-times
    differentiable for any $\betav \in \mathbb{R}^p \backslash \partial \calC$.
\end{lemma}

This lemma implies if $\estim{\betav} - \sum_{j=1}^n \xv_j
\dot{\ell}(\xv_j^\top \estim{\betav}; y_j) - \nabla R(\estim{\betav}) \notin
\partial C$, then
\begin{equation}\label{eq:constrained-smooth-alo}
    \surrogi{\betav}
    =
    \projv_{\calC}\Bigg(\estim{\betav} + \frac{\Gv\xv_i}{1 -
    \xv_i^\top \Gv \xv_i\ddot{\ell}(\xv_i^\top \estim{\betav}; y_i)}
    \dot{\ell}(\xv_i^\top \estim{\betav}; y_i)\Bigg),
\end{equation}
where $\Gv = \big(\Jv \Xv^\top
\diag[\{\ddot{\ell}(\xv_j^\top\estim{\betav};y_j)\}_j] \Xv
+ \Iv - \Jv + \Jv \nabla^2 R(\estim{\betav})\big)^{-1} \Jv$ with $\Jv$ representing
the Jacobian of the projection. In Section \ref{sec:applications} we study specific problems and show how the Jacobian can be calculated.

\begin{remark}
Note that while the Jacobian of the projection maps every vector in
$\mathbb{R}^p$ to a vector in the tangent space of $\partial C$, the action of
the Jacobian on a vector is not equivalent to the projection onto the tangent
space of $\partial C$.
\end{remark}

\begin{remark}
    Let $\calC^\circ$ be the interior of $\calC$. If $\calC^\circ \neq
    \emptyset$ and $\estim{\betav} - \sum_{j=1}^n\xv_j \dot{\ell}
    (\xv_j^\top\estim{\betav}; y_j) - \nabla R(\estim{\betav}) \in \calC^\circ$,
    we have $\Jv = \Iv$.
\end{remark}

\begin{remark} \label{remark:three-approaches-strength}
    We note that the dual approach is typically powerful for models with smooth
    losses and norm-type regularizers, such as the SLOPE norm and the
    generalized LASSO. On the other hand, the primal approach is valuable for
    models with nonsmooth loss, such as SVM, or when the Hessian of the regularizer is
    feasible to calculate. Such regularizers often exhibit some type of
    separability or symmetry, such as LASSO and nuclear norm. Finally the
    proximal approach can handle the problems with constraints nicely. It can
    also deal with models involving nonsmooth regularizers, as long as the
    Jacobian of the corresponding proximal operator can be easily obtained.
\end{remark}

\section{Equivalence Between Primal, Dual and Proximal Methods}\label{sec:primal-dual-equiv}

So far we have introduced three frameworks to approximate $\loo_\lambda$. Although the
primal, dual and prixmal methods may be harder or easier to carry out depending
on the specific problem at hand, one may wonder if they always obtain the same
result. In this section, we show that if the loss function and regularizer are
twice differentiable, these frameworks lead to equivalent formulas. We first
show the equivalence of primal and dual in Section \ref{ssec:primdualequ}, and
then discuss the equivalence of primal and proximal in Section
\ref{ssec:primproxequ}. Finally, Section \ref{ssec:theory} uses these
equivalence results to show the accuracy of our formulas for the case of smooth
losses and regularizers.

\subsection{Primal and Dual Equivalence}\label{ssec:primdualequ}

As both the primal and dual methods are based on a first-order approximation strategy,
we will study them not as approximate solutions to the leave-$i$-out problem, but will
instead show that they are exact solutions to a surrogate leave-$i$-out problem. Indeed,
recall that the leave-$i$-out problem is given by \eqref{eq:leave-i-out}, which cannot
be solved in closed form. However, we note that the solution does exist in closed form
in the case where both $\ell$ and $R$ are quadratic functions.

We may thus consider the approximate leave-$i$-out problem, where both $\ell$ and $R$
in the leave-$i$-out problem \eqref{eq:leave-i-out} have been replaced by their
quadratic expansion at the full data solution:
\begin{equation} \label{eq:alo:surrog}
    \min_{\leavei{\betav}} \sum_{j \neq i} \surrog{\ell}(\xv_j^\top \leavei{\betav}; y_j)
    + \surrog{R}(\leavei{\betav}).
\end{equation}

When both $\ell$ and $R$ are twice differentiable at the full data solution, $\surrog{\ell}$
and $\surrog{R}$ can be taken to simply be their respective second order Taylor expansions at $\estim{\betav}$.
The way we obtain $\leavei{\surrog{\betav}}$ in \eqref{eq:primal-alo-smooth0}
indicates that the primal formula in \eqref{eq:primal-alo-smooth} and
\eqref{eq:smooth-H} are the exact leave-$i$-out solution of the 
surrogate primal problem \eqref{eq:alo:surrog}.
On the other hand, we may also wish to consider the surrogate dual problem, by replacing
$\ell^*$ and $R^*$ by their quadratic expansion at full data dual solution $\estim{\dualv}$
in the dual problem \eqref{eq:methods:dual}. One may possibly worry that the surrogate dual
problem is then different from the dual of the surrogate primal problem \eqref{eq:alo:surrog}. This
does not happen, and we have the following theorem.
\begin{theorem}\label{thm:primal-dual-equivalence}
   Let $\ell$ and $R$ be twice differentiable convex functions. Let
   $\surrog{\ell}$ and $\surrog{R}$ denote the quadratic surrogates of the
   loss and regularizer
    at the full data solution $\estim{\betav}$, and let $\surrog{\ell}_D^*$ and $\surrog{R}_D^*$
    denote the quadratic surrogates of the conjugate loss and regularizer at the dual full data
    solution $\estim{\dualv}$. We have that the following problems are equivalent (have the same minimizer):
    \begin{gather}
        \label{eq:equivalence:primal-surrogate}
        \min_{\dualv} \sum_{j = 1}^n \surrog{\ell}^*(-\duals_j; y_j) +
        \surrog{R}^*(\Xv^\top \dualv), \\
        \label{eq:equivalence:dual-surrogate}
        \min_{\dualv} \sum_{j = 1}^n \surrog{\ell}^*_D(-\duals_j; y_j) +
        \surrog{R}^*_D(\Xv^\top \dualv).
    \end{gather}
\end{theorem}

Additionally, we note that the dual method described in Section
\ref{sec:approximatedual} solves the surrogate dual problem
\eqref{eq:equivalence:dual-surrogate}.
\begin{theorem}\label{thm:primal-dual-equivalence-2}
    Let $\Xv_u$, $\yv_u$ be as in \eqref{eq:dual:general-loss},
    and let $\leavei{\surrog{y}}_{u,i}$ be the transformed ALO obtained in
    \eqref{eq:dual:alo-general-loss}. Let $\tilde{\yv}_a$ be the same as $\yv_u$
    except $\tilde{y}_{a,i} = \leavei{\surrog{y}}_{u,i}$. Then
    $\surrog{\yv}_a$ satisfies
    \begin{equation*}
        [\proxv_{\tilde{g}}(\tilde{\yv}_a)]_i = 0,
    \end{equation*}
    where $\tilde{g}(\uv) = \tilde{R}^*(\Xv_u^\top \uv)$ and $\tilde{R}$
    denotes the quadratic surrogate of the regularizer.

    In particular, $\leavei{\surrog{y}}_i=K_{ii}\leavei{\surrog{y}}_{u,i}$ is
    the exact leave-$i$-out predicted value for the surrogate problem described
    in Theorem \ref{thm:primal-dual-equivalence}.
\end{theorem}

We refer the reader to Section \ref{append:sec:primal-dual-equiv} for the
proofs. These two theorems imply that for twice differentiable losses and
regularizers, the frameworks we laid out in Sections \ref{sec:approximatedual}
and \ref{sec:primal-smoothing} lead to exactly the same ALO formulas.
This equivalence theorem reflects the deep connections between
the primal and dual optimization problem. The central
property used by the proof is captured in the following lemma:
\begin{lemma}
    Let $f$ be a proper closed convex function, such that both $f$ and $f^*$ are twice differentiable.
    Then, we have for any $\xv$ in the domain of $f$:
    \begin{equation*}
        \nabla^2 f^*(\nabla f(\xv)) = [\nabla^2 f(\xv)]^{-1}.
    \end{equation*}
\end{lemma}

By combining this lemma with the primal dual correspondence \eqref{eq:primal-dual-correspondence},
we obtain a relation between the curvature of the primal and dual problems at the optimal value,
ensuring that the approximation is consistent with the dual structure.

\subsection{Primal and Proximal Equivalence}\label{ssec:primproxequ}
As discussed in the last section the primal approximation
\begin{equation}\label{eq:loocva_primal}
    \surrogi{\betav}
    =
    \estim{\betav}
    +
    \bigg[\sum_{j\neq i} \xv_j \xv_j^\top\ddot{\ell}
    (\xv_j^\top\estim{\betav}; y_j) + \nabla^2 R(\estim{\betav})\bigg]^{-1}
    \xv_i\dot{\ell}(\xv_i^\top\estim{\betav}; y_i),
\end{equation}
is the exact leave-one-out estimate for the surrogate problem $\min_{\betav}
\sum_{j \neq i} \surrog{\ell}(\xv_j^\top \betav; y_j) + \surrog{R}(\betav)$. We
start by applying the proximal method discussed in Section
\ref{ssec:smoothprox} to this surrogate problem. Since $\surrog{R}(\betav)$ is
a quadratic function, its proximal operator is a linear function in $\betav$
and is given by
\begin{equation}\label{eq:surrogate}
    \proxv_{\surrog{R}} (\betav)
    = \big[\Iv+ \nabla^2  R(\estim{\betav})\big]^{-1}
    (\nabla^2 R(\estim{\betav})\estim{\betav} - \nabla R(\estim{\betav}))
    + \big[\Iv + \nabla^2 R(\estim{\betav})\big]^{-1} \betav.
\end{equation}

Hence, we can calculate the Jacobian $\tilde{\Jv}$ of ${\proxv}_{\surrog{R}}$
and plug it in \eqref{eq:prox-approx-formula0} to obtain the following approximation
of $\estimi{\betav}$:
\begin{equation} \label{eq:prox-approx-surrog}
    \surrogi{\betav}_P
    =
    \estim{\betav} + \bigg[\Iv - \tilde{\Jv} \bigg(\Iv - \sum_{j \neq i} \ddot{\ell}
    (\xv_j^\top\estim{\betav}; y_j)\xv_j\xv_j^\top\bigg)\bigg]^{-1} \tilde{\Jv}
    \xv_i \dot{\ell}(\xv_i^\top\estim{\betav}; y_i),
\end{equation}
where $\tilde{\Jv} = \big[\Iv+ \nabla^2 R(\estim{\betav})\big]^{-1}$. Even though this
formula looks different from  \eqref{eq:loocva_primal}, we can see that
since $\Iv - \tilde{\Jv} = \big[\Iv+ \nabla^2 R(\estim{\betav})\big]^{-1}\nabla^2
R(\estim{\betav}) = \tilde{\Jv} \nabla^2 R(\estim{\betav})$. Note that
$\tilde{\Jv}$ is
invertible, we have 
\begin{equation}\label{eq:proxtosurrogate}
    \bigg[\Iv - \tilde{\Jv} \bigg(\Iv - \sum_{j \neq i} \ddot{\ell}
    (\xv_j^\top\estim{\betav}; y_j) \xv_j\xv_j^\top \bigg)\bigg]^{-1}\tilde{\Jv}
    = \bigg[ \nabla^2 R(\estim{\betav}) + \sum_{j\neq i} \ddot{\ell}
    (\xv_j^\top \estim{\betav}; y_j)\xv_j\xv_j^\top\bigg]^{-1}.
\end{equation}

Hence, the proximal approach when applied to the surrogate problem, returns the
same formula as the primal approach. In our next step, we would like to show
that the formulas we obtain by applying the proximal approach to the original
and surrogate problems return the same formulas. Note that when the proximal
approach is applied to these two problem, the formulas look exactly the same,
and they only differ in the Jacobians of the proximal operator. Note that the
proximal operator of $R$ and $\tilde{R}$ are different and hence the Jacobians
can be different. However, a nice property of proximal operators leads to the
following lemma:
\begin{lemma}\label{prox:equivalence}
    Suppose that $R$ is twice differentiable. Let $\Jv$ and $\surrog{\Jv}$
    denote the Jacobian of the proximal operators of $R$ and $\tilde{R}$ in
    \eqref{eq:prox-approx-formula0} and \eqref{eq:prox-approx-surrog} respectively.
    Then,
    \begin{equation*}
        \Jv = \surrog{\Jv}.
    \end{equation*}
    i.e., $\Jv$ at $\estim{\betav} - \sum_{j=1}^n \dot{\ell}
    (\xv_j^\top\estim{\betav}; y_j)\xv_j$ coincides with $\tilde{\Jv}$.
\end{lemma}

The proof of this lemma is presented in Section
\ref{append:sec:primal-dual-equiv}. Combining Lemma \ref{prox:equivalence} with
\eqref{eq:proxtosurrogate} proves the following equivalence theorem:
\begin{theorem}
    Let both $\ell$ and $r$ be twice differentiable. Furthermore, let
    $\surrogi{\betav}$ and $\surrogi{\betav}_P$ denote the approximations
    obtained from the primal and proximal approach. Then we have
    \begin{equation*}
        \surrogi{\betav} = \surrogi{\betav}_P. 
    \end{equation*}
\end{theorem}

\subsection{Discussion on the Accuracy of the ALO formulas}\label{ssec:theory}
The results we derived in Sections \ref{ssec:primdualequ} and
\ref{ssec:primproxequ}, combined with Theorem 3 of \cite{kamiar2018scalable},
offer an upper bound on the error of the primal, dual, and proximal
$\alo_\lambda$ formulas. Specifically, under some regularity conditions on the
second order derivatives of the loss and the regularizer,
\cite{kamiar2018scalable} proved the following holds with high probability:
\begin{equation*}
    \max_{i}\big|\xv_i^\top \estimi{\betav} - \xv_i^\top \surrogi{\betav} \big|
    \leq \frac{C_0(p)}{\sqrt{p}},
\end{equation*}
where $\surrogi{\betav}$ denotes the primal approximation in Section
\ref{ssec:primal-smooth} and $C_0(p)$ is expected to be of a logarithmic order in $p$.
We want to remind the reader that in \cite{kamiar2018scalable}, $n$ and $p$ are assumed to
be at the same order. That is why $n$ does not appear in the upper bound.
Now if we combine this upper bound with the equivalence theorems in the last
sections, we can prove the following result. When the loss and regularizer are twice
differentiable with a few regularity conditions on their second order derivatives
(please check Section 3 of \cite{kamiar2018scalable}), the formulas we
obtained from the dual and proximal approaches in Sections
\ref{ssec:dual:general-case} and \ref{ssec:smoothprox} are also accurate.

\section{Inclusion of Intercept}\label{sec:intercept}
In all the previous discussions, we assumed that the regression coefficient corresponding to the intercept term is penalized similar to the other regression coefficients.
However, often researchers prefer not to regularize the intercept term. For some of the model formulations, such as the penalized linear models with square loss, one may get rid of the intercept by centering each variable. However in many
other cases, there is no simple way to absorb the intercept term without
altering the meaning of the model. In this section, we discuss the ALO formula for models
involving intercepts. The goal of this section is to describe how the formulas should be modified when the intercept term is not regularized.

\subsection{Smooth Models}
\label{ssec:intercept-smooth-nonsmooth-loss}
Denote the intercept by $\beta_0$. Also, let $\betav$ denote the vector of all the regression coefficients except for $\beta_0$.  For the smooth models, we can naturally
treat $\bm{1}$ as a variable with coefficient $\beta_0$ and obtain the $\Hv$
matrix with the following form:
\begin{align}\label{eq:intercept-H}
    \Hv
    =&
    [\bm{1}, \Xv]\bigg(
    \begin{bmatrix}\bm{1}^\top \\ \Xv^\top\end{bmatrix}
    \diag[\{\ddot{\ell}(\hat{\beta}_0 + \xv_j^\top\hat{\betav}; y_j)\}_j]
    [\bm{1}, \Xv] +
    \begin{bmatrix} 0 & \\ & \nabla^2 R(\estim{\betav})\end{bmatrix}\bigg)^{-1}
    \begin{bmatrix} \bm{1}^\top \\ \Xv^\top \end{bmatrix} \nonumber \\
    =&
    [\bm{1}, \Xv]
    \begin{bmatrix}
        \sum_j \ddot{\ell}(\hat{\beta}_0 + \xv_j^\top\hat{\betav}; y_j)
        & 
        \sum_j \ddot{\ell}(\hat{\beta}_0 + \xv_j^\top\hat{\betav}; y_j) \xv_j^\top \\
        \sum_j \ddot{\ell}(\hat{\beta}_0 + \xv_j^\top\hat{\betav}; y_j) \xv_j
        &
        \Xv^\top \diag[\{\ddot{\ell}(\hat{\beta}_0 + \xv_j^\top\hat{\betav};
        y_j)\}_j] \Xv + \nabla^2 R(\estim{\betav})
    \end{bmatrix}^{-1}
    \begin{bmatrix} \bm{1}^\top \\ \Xv^\top \end{bmatrix}
\end{align}
We can then plug \eqref{eq:intercept-H} into \eqref{eq:primal-alo-smooth} to
obtain the ALO formula for prediction on the leave-$i$-out sample.

\subsection{Models with Nonsmooth Losses}

In this section, we study the models we discussed in Section \ref{ssec:nonsmooth-loss}, i.e., the regularizer is smooth, while the loss function has a finite number of zero-order singularities. For such models, we need to adapt the results in Theorem
\ref{thm:nonsmooth-loss-approx} to get the ALO formula, when the intercept term is not penalized.  

\begin{theorem} \label{thm:nonsmooth-loss-intercept}
    Following the notations and results of Theorem
    \ref{thm:nonsmooth-loss-approx}, we need the following modifications to
    obtain the ALO formula when the intercept term is not penalized:
    \begin{align*}\label{eq:nonsmooth-H-intercept}
        a_i &= \begin{cases}
            \frac{W_{ii}}{1 - W_{ii}\ddot{\ell}(\xv_i^\top\estim{\betav};y_i)}
            & \text{ if } i \in S, \\
            \frac{1}{\Uv_{ii}} & \text{ if } i \in V, \\
        \end{cases}
    \end{align*}
    where
    \begin{align*}
        \Yv =& \nabla^2 R(\estim{\betav})
        + \Xv_{S,\cdot}^\top \diag[\{\ddot{\ell}(\estim{\beta}_0 + \xv_j^\top\estim{\betav})\}_{j
        \in S}] \Xv_{S,\cdot}, \\
        \Uv =& \big[\Xv_{V,\cdot}\Yv^{-1}\Xv_{V,\cdot}^\top\big]^{-1}
        -
        \frac{\big[\Xv_{V,\cdot}\Yv^{-1}\Xv_{V,\cdot}^\top\big]^{-1}
        \big(\bm{1} - \Xv_{V,\cdot} \Yv^{-1} \bv\big)
        \big(\bm{1} - \Xv_{V,\cdot} \Yv^{-1} \bv\big)^\top
        \big[\Xv_{V,\cdot}\Yv^{-1}\Xv_{V,\cdot}^\top\big]^{-1}}
        {a - \bv^\top\Yv^{-1}\bv
        + \big(\bm{1} - \Xv_{V,\cdot} \Yv^{-1} \bv\big)^\top
        \big[\Xv_{V,\cdot}\Yv^{-1}\Xv_{V,\cdot}^\top\big]^{-1}
        \big(\bm{1} - \Xv_{V,\cdot} \Yv^{-1} \bv\big)}, \\
        \Wv =& \Xv_{S,\cdot}\Yv^{-1}\Xv_{S,\cdot}^\top -
        \Xv_{S,\cdot}\Yv^{-1}\Xv_{V,\cdot}^\top
        \big[\Xv_{V,\cdot}\Yv^{-1}\Xv_{V,\cdot}^\top \big]^{-1}
        \Xv_{V,\cdot}\Yv^{-1}\Xv_{S,\cdot}^\top \nonumber \\
        &+
        \frac{\dv \dv^\top}
        {a - \bv^\top\Yv^{-1}\bv + \big(\bm{1} - \Xv_{V,\cdot} \Yv^{-1} \bv\big)^\top
        \big[\Xv_{V,\cdot}\Yv^{-1}\Xv_{V,\cdot}^\top\big]^{-1}
        \big(\bm{1} - \Xv_{V,\cdot} \Yv^{-1} \bv\big)}.
    \end{align*}
    where $a = \sum_{j\in S} \ddot{\ell}(\hat{\beta}_0 +
    \xv_j^\top\hat{\betav}; y_j)$, $\bv = \sum_{j\in S} \ddot{\ell}
    (\hat{\beta}_0 + \xv_j^\top\hat{\betav}; y_j) \xv_j$,
    $\dv = \Xv_{S,\cdot}\Yv^{-1}\Xv_{V,\cdot}^\top \big[\Xv_{V,\cdot}
    \Yv^{-1} \Xv_{V,\cdot}^\top \big]^{-1} (\bm{1} - \Xv_{V,\cdot}\Yv^{-1}\bv) -
    (\bm{1} - \Xv_{S,\cdot}\Yv^{-1}\bv)$.
\end{theorem}
The derivation is slightly complicated. Hence, we refer the reader to Section
\ref{ssec:proof-of-nonsmooth-loss-intercept} for the proof.

\subsection{Models with Nonsmooth Regularizers} \label{ssec:intercept-nonsmooth-reg}
In this section, we consider the cases where the loss function is twice differentiable everywhere, while the regularizer is not smooth. 
To simplify the discussion, we present a slightly simplified variation of
\eqref{eq:intercept-H} based on the Woodbury matrix inversion formula.
Define $a = \sum_j \ddot{\ell}(\hat{\beta}_0 + \xv_j^\top\hat{\betav}; y_j)$,
$\bv = \sum_j \ddot{\ell}(\hat{\beta}_0 + \xv_j^\top\hat{\betav}; y_j) \xv_j$
and
$\Av = \Xv^\top \diag[\{\ddot{\ell}(\hat{\beta}_0 + \xv_j^\top\hat{\betav};
y_j)\}_j] \Xv
+ \nabla^2 R(\estim{\betav})$. The matrix $\Hv$ in \eqref{eq:intercept-H} can be simplified to
\begin{align*}
    \Hv
    =&
    [\bm{1}, \Xv]
    \begin{bmatrix}
        a & \bv^\top \\
        \bv
        &
        \Av
    \end{bmatrix}^{-1}
    \begin{bmatrix} \bm{1}^\top \\ \Xv^\top \end{bmatrix}
    =
    [\bm{1}, \Xv]
    \begin{bmatrix}
        \frac{1}{a - \bv^\top \Av^{-1} \bv} & - \frac{\bv^\top \Av^{-1}}{a - \bv^\top \Av^{-1} \bv} \\
        - \frac{\Av^{-1}\bv}{a - \bv^\top \Av^{-1} \bv}
        &
        \Av^{-1} + \frac{\Av^{-1}\bv\bv^\top\Av^{-1}}{a - \bv^\top \Av^{-1} \bv}
    \end{bmatrix}
    \begin{bmatrix} \bm{1}^\top \\ \Xv^\top \end{bmatrix} \nonumber \\
    =&
    \Xv\Av^{-1}\Xv^\top + \frac{1}{a - \bv^\top\Av^{-1}\bv}
    \big(\bm{1} - \Xv \Av^{-1} \bv\big)
    \big(\bm{1} - \Xv \Av^{-1} \bv\big)^\top.
\end{align*}

When we have a smooth loss and nonsmooth regularizer (separable or non-separable),
if we adopt some smoothing strategy and let the smoothing parameter go to 0,
it is straightforward to see that $\Xv \Av^{-1} \Xv^\top$ still converges to the
``hat'' matrix presented in the intercept-free models. Assume
$\Xv \Av^{-1} \Xv^\top \rightarrow \Hv_0$, we note that $\bv = \Xv \ddot{\ellv}$
with $\ddot{\ellv} = [\ddot{\ell}(\hat{\beta}_0 + \xv_1^\top\hat{\betav}; y_1),
\ldots, \ddot{\ell}(\hat{\beta}_0 + \xv_n^\top\hat{\betav}; y_n)]^\top$ and then have
\begin{equation}\label{eq:intercept-H-nonsmooth-reg}
    \Hv
    =
    \Hv_0 + \frac{1}{a - \ddot{\ellv}^\top\Hv_0\ddot{\ellv}}
    (\bm{1} - \Hv_0 \ddot{\ellv})(\bm{1} - \Hv_0 \ddot{\ellv})^\top.
\end{equation}

Again we can plug \eqref{eq:intercept-H-nonsmooth-reg} into
\eqref{eq:primal-alo-smooth} to obtain the ALO prediction.

\subsection{Models with Constraints} \label{ssec:intercept-constrained}
In this section, we address the intercept issue for models with
constraints. These are the models we described in details in Section \ref{ssec:constrainedopt}. Here we
assume no constraint on $\beta_0$. Hence, the constraint set on all the regression coefficients becomes $\calC_1 = \mathbb{R} \times \calC$, where $\calC$ is the set of constraints that we apply to all the regression coefficients except for the intercept. It is straightforward to see that the Jacobian $\Jv_1$ of
$\projv_{\calC_1}((\beta_0, \betav))$ takes the form
\begin{equation*}
    \Jv_1 = \begin{bmatrix}
        1 & \\
          & \Jv
    \end{bmatrix},
\end{equation*}
where $\Jv$ is the Jacobian of $\projv_{\calC}(\betav)$. Now we can simplify
the matrix $\Gv$ in \eqref{eq:constrained-smooth-alo}. Treating the
intercept as the coefficient for constant variable $1$, we have
\begin{align*}
    \Gv_1
    =&
    \bigg(
    \begin{bmatrix}
        1 & \\  & \Jv
    \end{bmatrix} 
    \begin{bmatrix}\bm{1}^\top \\ \Xv^\top\end{bmatrix}
    \diag[\{\ddot{\ell}(\hat{\beta}_0 + \xv_j^\top\hat{\betav}; y_j)\}_j]
    [\bm{1}, \Xv] +
    \begin{bmatrix}
        1 & \\  & \Iv
    \end{bmatrix} -
    \begin{bmatrix}
        1 & \\  & \Jv
    \end{bmatrix} +
    \begin{bmatrix}
        1 & \\  & \Jv
    \end{bmatrix}
    \begin{bmatrix}
        0 & \\  & \nabla^2 R(\estim{\betav})
    \end{bmatrix} \bigg)^{-1}
    \begin{bmatrix}
        1 & \\  & \Jv
    \end{bmatrix} \nonumber \\
    =&
    \begin{bmatrix}
        \sum_j \ddot{\ell}(\hat{\beta}_0 + \xv_j^\top\hat{\betav}; y_j)
        & 
        \sum_j \ddot{\ell}(\hat{\beta}_0 + \xv_j^\top\hat{\betav}; y_j) \xv_j^\top \\
        \Jv \sum_j \ddot{\ell}(\hat{\beta}_0 + \xv_j^\top\hat{\betav}; y_j) \xv_j
        &
        \Jv \Xv^\top \diag[\{\ddot{\ell}(\hat{\beta}_0 + \xv_j^\top\hat{\betav};
        y_j)\}_j] \Xv + \Iv - \Jv + \Jv \nabla^2 R(\estim{\betav})
    \end{bmatrix}^{-1}
    \begin{bmatrix}
        1 & \\  & \Jv
    \end{bmatrix}.
\end{align*}

Similar to the previous arguments, we simplify the above formula using Woodbury
matrix inversion formula.
Again let $a = \sum_j \ddot{\ell}(\hat{\beta}_0 + \xv_j^\top\hat{\betav}; y_j)$,
$\bv = \sum_j \ddot{\ell}(\hat{\beta}_0 + \xv_j^\top\hat{\betav}; y_j) \xv_j$
and
$\Av = \Xv^\top \diag[\{\ddot{\ell}(\hat{\beta}_0 + \xv_j^\top\hat{\betav};
y_j)\}_j] \Xv + \nabla^2 R(\estim{\betav})$. In addition, set $\Gv = (\Jv\Av +
\Iv - \Jv)^{-1}\Jv$, we can rewrite $\Gv_1$ as
\begin{align} \label{eq:intercept-G-constraints}
    \Gv_1
    =&
    \begin{bmatrix}
        a & \bv^\top \\
        \Jv \bv & \Jv \Av + \Iv - \Jv
    \end{bmatrix}^{-1}
    \begin{bmatrix} 1 & \\ & \Jv \end{bmatrix} \nonumber \\
    =&
    \begin{bmatrix}
        \frac{1}{a - \bv^\top (\Jv\Av + \Iv - \Jv)^{-1}\Jv\bv} & - \frac{\bv^\top (\Jv\Av + \Iv - \Jv)^{-1}}{a - \bv^\top (\Jv\Av + \Iv - \Jv)^{-1}\Jv\bv} \\
        - \frac{(\Jv\Av + \Iv - \Jv)^{-1}\Jv\bv}{a - \bv^\top (\Jv\Av + \Iv - \Jv)^{-1}\Jv\bv}
        &
        (\Jv\Av + \Iv - \Jv)^{-1} + \frac{(\Jv\Av + \Iv -
        \Jv)^{-1}\Jv\bv\bv^\top(\Jv\Av + \Iv - \Jv)^{-1}}{a - \bv^\top (\Jv\Av + \Iv - \Jv)^{-1}\Jv\bv}
    \end{bmatrix}
    \begin{bmatrix}
        1 & \\  & \Jv
    \end{bmatrix} \nonumber \\
    =&
    \begin{bmatrix}
        0 & \\
          & \Gv
    \end{bmatrix}
    +
    \frac{1}{a - \bv^\top \Gv \bv}
    \begin{bmatrix}
        1 & - \bv^\top \Gv \\
        - \Gv\bv
        &
        \Gv\bv\bv^\top\Gv
    \end{bmatrix}.
\end{align}

We can plug \eqref{eq:intercept-G-constraints} into
\eqref{eq:constrained-smooth-alo} and change $\Xv$ to $[1, \Xv]$, $\betav$ to
$\begin{bmatrix} \beta_0 \\ \betav \end{bmatrix}$ to get the ALO formula.
Specifically the following two quantities will be used.
\begin{align*}
    \Gv_1 \begin{bmatrix} \bm{1}^\top \\ \Xv^\top \end{bmatrix}
    =&
    \begin{bmatrix}
        0 \\ \Gv \Xv^\top
    \end{bmatrix}
    +
    \frac{1}{a - \bv^\top \Gv \bv}
    \begin{bmatrix}
        \bm{1}^\top - \bv^\top \Gv\Xv^\top \\
        - \Gv\bv \bm{1}^\top + \Gv\bv\bv^\top\Gv\Xv^\top
    \end{bmatrix}, \nonumber \\
    [\bm{1}, \Xv] \Gv_1 \begin{bmatrix} \bm{1}^\top \\ \Xv^\top \end{bmatrix}
    =&
    \Xv \Gv \Xv^\top
    +
    \frac{1}{a - \bv^\top \Gv \bv} \big(\bm{1} - \Xv\Gv\bv\big) \big(\bm{1} -
    \Xv\Gv\bv\big)^\top.
\end{align*}

\section{Applications}\label{sec:applications}
In this section, we apply the three approaches introduced in Section
\ref{sec:approximatedual}, \ref{sec:primal-smoothing},
\ref{sec:proxformulation} to eight specific models and obtain their ALO
formula.

\subsection{Generalized LASSO}

The generalized LASSO \cite{tibshirani2011genlasso} is a generalization of the LASSO
problem which captures many applications, such as the fused LASSO \cite{tibshirani2005fused},
$\ell_1$ trend filtering \cite{kim2009trend} and wavelet smoothing in a
unified framework. The generalized LASSO problem corresponds to the following
penalized regression problem:
\begin{equation}\label{eq:genlasso:statement}
    \min_{\betav} \frac{1}{2}\sum_{j = 1}^n (y_j - \xv_j^\top \betav)^2 + \lambda \norm{\Dv \betav}_1,
\end{equation}
where the regularizer is parameterized by a fixed matrix $\Dv \in \RR^{m \times p}$
which captures the desired structure in the data. We note that the
regularizer is a semi-norm, and hence we can formulate the dual problem as a
projection. In fact, a dual formulation of \eqref{eq:genlasso:statement} can
be obtained as (see Appendix \ref{append:sec:generalized-lasso-dual}):
\begin{equation} \label{eq:dualformulagenlasso}
    \min_{\dualv, \uv} \frac{1}{2} \norm{\dualv - \yv}_2^2,
    \quad
    \text{subject to: } \norm{\uv}_\infty \leq \lambda \text{ and }  \Xv^\top \dualv = \Dv^\top \uv. 
\end{equation}

The dual optimal solution satisfies $\estim{\dualv} = \projv_{\Delta_X}(\yv)$,
where $\Delta_X$ is the polytope given by
\begin{equation*}
    \Delta_X = \{ \dualv \in \RR^n : \exists \uv, \norm{u}_\infty \leq \lambda
    \text{ and } \Xv^\top \dualv = \Dv^\top u \}.
\end{equation*}

The projection onto the polytope $C = \{ \Dv^\top \uv: \norm{\uv}_\infty \leq
\lambda \}$ is given in \cite{tibshirani2011genlasso} as locally being the projection
onto the affine space orthogonal to the nullspace of $\Dv_{\cdot, -E}$, where
$E = \{i : \abs{\estim{u}_i} = \lambda \}$ and $-E = \{ 1, \dotsc, p \} \setminus E$.
Since $\Delta_X = [\Xv^\top]^{-1} C$ is the inverse image of $C$ under the
linear map given by $\Xv^\top$, the projection onto $\Delta_X$ is given locally
by the projection onto the affine space normal to the space spanned by the
columns of $[\Xv^\top]^+ \mathrm{null} \, \Dv_{\cdot, -E}$, provided $\Xv$ has
full column rank. Here, $[\Xv^\top]^+$ denotes the Moore-Penrose pseudoinverse
of $\Xv^\top$. Finally, to obtain a spanning set of this space, we may
consider $\Av = \Xv \Bv$, where $\Bv$ is a set of vectors spanning the nullspace
of $\Dv_{\cdot,-E}$. This allows us to compute $\Hv = \Av \Av^+$, the
projection onto the normal space required to compute the ALO.

In summary, the $\alo$ formula can be obtained in the following way. We solve
the primal ( eq. \eqref{eq:genlasso:statement}) and dual (eq.
\eqref{eq:dualformulagenlasso}) problems to obtain $\estim{\betav}$ and
$\estim{\uv}$ respectively. Then we calculate $E = \{i : \abs{\estim{u}_i} =
\lambda \}$ and construct the matrix $\Bv$ whose columns span the null space of
$ \Dv_{\cdot, -E}$. Finally, we can compute $\Hv = \Av \Av^+$ with $\Av = \Xv
\Bv$ and obtain that $\alo_\lambda = \frac{1}{n}\sum_{i = 1}^n d(y_i,
\surrog{y}_i)$, where $\surrog{y}_i = \xv_i^\top \estim{\betav} +
\frac{\Hv_{ii}}{1 - \Hv_{ii}}(\xv_i^\top \estim{\betav} - y_i)$.

\subsection{Nuclear Norm}
Consider the following problem
\begin{equation}\label{eq:nuclear-norm-main}
    \estim{\Bv}:
    =
    \argmin_{\Bv} \frac{1}{2}\sum_{j=1}^n \big(y_j - \langle \Xv_j,
    \Bv\rangle \big)^2 + \lambda \|\Bv\|_*,
\end{equation}
with $\Bv, \Xv_j \in \mathbb{R}^{p_1 \times p_2}$. $\langle \Xv, \Bv\rangle =
\mathrm{trace}(\Xv^\top\Bv)$ denotes the inner product. We use $\|\cdot\|_*$
for nuclear norm, which is defined as the sum of the singular values of a
matrix. This problem is used in many applications, such as the matrix sensing and matrix completion. 

The nuclear norm is a unitarily invariant function of the matrix
\cite{lewis1995convex}. Such functions are only indirectly related to the
components of the matrix, making the calculation of $\alo$
difficult even when they are smooth, and exacerbating the difficulties when they
are non-smooth, such as in the case of the nuclear norm. We are nonetheless able to leverage the specific structure of such functions
to obtain the following theorem. Let $R$ be a smooth unitarily invariant matrix function, with:
\begin{equation*}
    R(\Bv) = \sum_{j = 1}^{\min(p_1, p_2)} r(\sigma_j),
\end{equation*}
where $\sigma_j$ denotes the $j$\tsup{th} singular value of $\Bv$. Consider the
following matrix penalized regression problem:
\begin{equation*}
    \estim{\Bv} = \argmin_{\Bv} \sum_{j=1}^n \ell\big(\langle \Xv_j,
    \Bv\rangle; y_j \big) + \lambda R(\Bv).
\end{equation*}

Without loss of generality, below we assume $p_1 \geq p_2$. Let
$\estim{\Bv}=\estim{\Uv}\diag[\estim{\sigmav}]\estim{\Vv}^\top$ be the
singular value decomposition (SVD) of the full data estimator $\estim{\Bv}$,
where $\estim{\Uv} \in \mathbb{R}^{p_1 \times p_1}$, $\estim{\Vv} \in
\mathbb{R}^{p_2 \times p_2}$. Let $\estim{\uv}_k$, $\estim{\vv}_l$ be the
$k$\tsup{th} and $l$\tsup{th} column of $\estim{\Uv}$ and $\estim{\Vv}$
respectively. $\diag[\estim{\sigmav}]$ in this section is a $p_1 \times
p_2$ matrix with $\estim{\sigma}_j$ on the diagonal of its upper square
sub-matrix and 0 elsewhere. If we assume all the $\estim{\sigma}_j$'s are
nonzero, then we have the following ALO formula:
\begin{equation*}\label{eq:matrix-smooth-alo}
    \langle \Xv_i, \leavei{\surrog{\Bv}} \rangle
    =
    \langle \Xv_i, \estim{\Bv} \rangle
    +
    \frac{H_{ii}\dot{\ell}(\langle \Xv_i, \estim{\Bv} \rangle;y_i)}{1 -
    H_{ii}\ddot{\ell}(\langle \Xv_i, \estim{\Bv} \rangle;y_i)},
\end{equation*}
where
\begin{equation*}
    \Hv = \cb{X}[\cb{X}^\top\diag[\{\ddot{\ell}(\langle \Xv_j, \Bv\rangle;
    y_j)\}_j] \cb{X} + \lambda \cb{G}]^{-1}\cb{X}^\top.
\end{equation*}

Here $\cb{X}$ is a $n \times p_1 p_2$ matrix and $\cb{G}$ is a symmetric square $p_1 p_2
\times p_1 p_2$ matrix given by:
\begin{equation}\label{eq:matrix-smooth-alo-definition}
    \begin{aligned}
        \cb{X}_{j,kl} &= \estim{\uv}_k^\top \Xv_j \estim{\vv}_l, \\
        \cb{G}_{kl, st} &=
        \begin{cases}
            \ddot{r}(\estim{\sigma}_t) & s = t = k = l, \\
            \frac{\estim{\sigma}_s \dot{r}(\estim{\sigma}_s) - \estim{\sigma}_t\dot{r}(\estim{\sigma}_t)}
            {\estim{\sigma}_s^2 - \estim{\sigma}_t^2} & s \neq t, s \leq p_2, (k, l) = (s, t), \\
            - \frac{\estim{\sigma}_s \dot{r}(\estim{\sigma}_t) - \estim{\sigma}_t\dot{r}(\estim{\sigma}_s)}
            {\estim{\sigma}_s^2 - \estim{\sigma}_t^2} & s \neq t, s \leq p_2, (k, l) = (t, s), \\
            \frac{\dot{r}(\estim{\sigma}_t)}{\estim{\sigma}_t} & s \neq t, s > p_2, (k,l)=(s,t), \\
            0 & \text{otherwise.}
        \end{cases}
    \end{aligned}
\end{equation}

Note that the rows of $\cb{X}$ and the indices of $\cb{G}$ are vectorized
in a consistent way. The proof can be found in Section \ref{append:ssec:smooth-unitary}.
A nice property of this result is that the effect on singular values
decouples from the original matrix, enabling us to apply the smoothing
strategy in Section \ref{ssec:nonsmooth-regularizer} to function $r(\sigma)$
when it is nonsmooth. This leads to the following theorem for nuclear norm. For
more details on the derivation, please refer to Section
\ref{append:ssec:nuclear-proof}.
\begin{theorem}\label{thm:matrix-nonsmooth-alo}
    Consider the nuclear-norm penalized matrix regression problem
    \eqref{eq:nuclear-norm-main}, and let
    $\estim{\Bv}=\estim{\Uv}\diag[\estim{\sigmav}]\estim{\Vv}^\top$ be the
    SVD of the full data estimator $\estim{\Bv}$,
    with $\estim{\Uv} \in \mathbb{R}^{p_1 \times p_1}$, $\estim{\Vv} \in
    \mathbb{R}^{p_2 \times p_2}$. Let $m=\rank(\estim{\Bv})$ be the number of
    nonzero $\estim{\sigma}_j$'s for $\estim{\Bv}$. Let
    $\leavei{\surrog{\Bv}}_h$ denote the approximate of $\estimi{\Bv}$
    obtained from the smoothed problem. Then, as $h \rightarrow 0$
    \begin{equation*}\label{eq:matrix-nonsmooth-alo}
        \langle \Xv_i, \leavei{\surrog{\Bv}}_h \rangle
        \rightarrow
        \langle \Xv_i, \estim{\Bv} \rangle
        +
        \frac{H_{ii}}{1 - H_{ii}}(\langle \Xv_i, \estim{\Bv} \rangle - y_i),
    \end{equation*}
    where
    \begin{equation*}
        \Hv = \cb{X}_{\cdot,E}[\cb{X}_{\cdot,E}^\top\cb{X}_{\cdot,E} +
        \lambda \cb{G}]^{-1}\cb{X}_{\cdot,E}^\top,
    \end{equation*}
    with $\cb{X}$ as defined in \eqref{eq:matrix-smooth-alo-definition}
    and $\cb{G} \in \mathbb{R}^{(mp_1 + mp_2 - m^2) \times (mp_1 + mp_2 - m^2)}$ given by:
    \begin{equation}\label{eq:nuclear-hessian-G}
        \cb{G}_{kl, st}
        =
        \begin{cases}
            0 & s = t = k = l \leq m, \\
            \frac{1}{\estim{\sigma}_s + \estim{\sigma}_t} & 1 \leq s \neq t \leq m,(k,l)=(s,t), \\
            \frac{1}{\estim{\sigma}_s} & 1 \leq s \leq m < t \leq p_2, (k, l) = (s, t), \\
            \frac{1}{\estim{\sigma}_t} & 1 \leq t \leq m < s \leq p_1, (k, l) = (s, t), \\
            -\frac{1}{\estim{\sigma}_s + \estim{\sigma}_t} & 1 \leq s \neq t \leq m, (k,l)=(t,s), \\
            -\frac{g_r[\estim{\sigma}_t]}{\estim{\sigma}_s} & 1 \leq s \leq m < t \leq p_2,
            (k, l) = (t, s), \\
            -\frac{g_r[\estim{\sigma}_s]}{\estim{\sigma}_t} & 1 \leq t \leq m < s \leq p_2,
            (k, l) = (t, s), \\
            0 & \text{otherwise.}
        \end{cases}
    \end{equation}
    where for $t > m$, $\estim{\sigma}_t=0$ and
    $g_r[\estim{\sigma}_t]$ is the corresponding subgradient at this singular
    value, which can be obtained through the SVD of
    $\frac{1}{\lambda}\sum_{j=1}^n (y_j - \langle \Xv_j,
    \estim{\Bv}\rangle)\Xv_j$.
    The set $E$ is then defined as:
    \begin{equation*}
        E = \{(k, l): k \leq m \text{ or } l \leq m\}.
    \end{equation*}
    Note that the indices of $\cb{G}$ and the index set $E$ are consistent.
\end{theorem}

\subsection{Linear SVM}\label{ssec:linearsvm}
The linear SVM optimization can be written as
\begin{equation*}\label{eq:svm-main}
    \argmin_{\betav}\sum_{j = 1}^n \big(1 - y_j\xv_j^\top\betav\big)_+ +
    \frac{\lambda}{2}\|\betav\|_2^2,
\end{equation*}
with $y_j\in\{-1, 1\}$ and $(\cdot)_+=\max\{\cdot, 0\}$. Note that this is a special instance of the problem
we studied in Section \ref{ssec:nonsmooth-loss}. Here,
$\ell(u; y_j) = (1- y_j u)_+$ has only one zeroth-order singularity at $y_j$.
Let $V=\{j: \xv_j^\top\estim{\betav} = y_j\}$ and $S=[1,\ldots, n]\backslash
V$. Using Theorem \ref{thm:nonsmooth-loss-approx} and simplifying the expressions,
we obtain the following ALO formula for SVM:
\begin{equation*}
    \xv_i^\top\leavei{\surrog{\betav}} = \xv_i^\top\estim{\betav} +  a_i g_{\ell,i},
\end{equation*}
where 
\begin{equation*}\label{eq:svm-H}
    a_i =
    \left\{
        \begin{array}{ll}
            \frac{1}{\lambda}\xv_i^\top (\Iv_p -
            \Xv_{V,\cdot}^\top(\Xv_{V,\cdot}\Xv_{V,\cdot}^\top)^{-1}\Xv_{V,\cdot})\xv_i
            & i \in S, \\
            \big(\lambda[(\Xv_{V,\cdot}\Xv_{V,\cdot}^\top)^{-1}]_{ii}\big)^{-1} & i \in V, \\
        \end{array}
    \right.
\end{equation*}
and for $i \in S$, $g_{\ell,i} = -y_i$ if $y_i\xv_i^\top\estim{\betav} < 1$,
$g_{\ell,i} = 0$ if $y_i\xv_i^\top\estim{\betav} > 1$, and for $i \in V$
\begin{equation*}\label{eq:svm-subg}
    \gv_{\ell,V} =
    (\Xv_{V,\cdot}\Xv_{V,\cdot}^\top)^{-1}\Xv_{V,\cdot}\Bigg[\lambda\estim{\betav} +
    \sum_{j: y_j\xv_j^\top\estim{\betav} < 1}y_j \xv_j\Bigg].
\end{equation*}

\subsection{Polyhedron Constraints} \label{ssec:application-polyhedron}
Consider the constrained optimization problem
\eqref{eq:constraintedoptimization} in which the constraint set $\calC$ is a polyhedron. For a point $\betav \notin \calC$, let $\Gammav$ be the matrix whose columns form an orthonormal basis for the face of $\calC$ that includes $\projv_{\calC}(\betav)$. Let $\Gammav_1$ denote the orthogonal complement of
$\Gammav$. Assume the columns of $\Gammav_1$ are also orthonormal. Then for any point $\vv \in \mathbb{R}^p$, there is a unique decomposition $\vv =
\Gammav \alphav + \Gammav_1 \alphav_1$. It is not hard to see that for small
$t$,
\begin{equation*}
    \projv_{\calC}(\betav + t \vv) = \projv_{\calC}(\betav) +
    t\Gammav\alphav + o(t).
\end{equation*}
Noting that $\alphav = \Gammav^\top\vv$, we obtain the following expression for
the Jacobian of $\projv_{\calC}$:
\begin{equation*}
    \Jv=\Gammav\Gammav^\top
\end{equation*}
Define $\Vv = \Xv^\top \diag[\ddot{\ell}(\xv_j^\top\estim{\betav}; y_j)] \Xv
+ \nabla^2 R(\estim{\betav})$. Now we can simplify the forms of $\Gv$ in
\eqref{eq:constrained-smooth-alo} as
\begin{align*}
    \Gv
    =&
    (\Iv - \Gammav\Gammav^\top + \Gammav\Gammav^\top \Vv)^{-1} \Gammav\Gammav^\top
    =
    \big( \Iv - \Gammav\Gammav^\top (\Iv - \Vv) \big)^{-1} \Gammav\Gammav^\top \nonumber \\
    =&
    \big[\Iv + \Gammav \big(\Iv - \Gammav^\top(\Iv - \Vv)\Gammav\big)^{-1}
    \Gammav^\top (\Iv - \Vv) \big] \Gammav \Gammav^\top \nonumber \\
    =&
    \big[\Iv + \Gammav \big( \Gammav^\top \Vv \Gammav \big)^{-1}
    \Gammav^\top (\Iv - \Vv) \big] \Gammav \Gammav^\top \nonumber \\
    =&
    \Gammav \Gammav^\top + \Gammav \big( \Gammav^\top \Vv \Gammav \big)^{-1}
    \Gammav^\top - \Gammav \big( \Gammav^\top \Vv \Gammav \big)^{-1}
    \Gammav^\top \Vv \Gammav \Gammav^\top \nonumber \\
    =&
    \Gammav \big( \Gammav^\top \Vv \Gammav \big)^{-1} \Gammav^\top
\end{align*}

That is to say, for this class of constraints, we have the $\alo$ formula
\eqref{eq:constrained-smooth-alo} holds with $\Gv = \Gammav \big[ \Gammav^\top
\big( \Xv^\top \diag[\ddot{\ell}(\xv_j^\top\estim{\betav}; y_j)] \Xv +
\nabla^2 R(\estim{\betav}) \big) \Gammav \big]^{-1} \Gammav^\top$.
Notice that the choice of $\Gammav$ does not affect $\Gv$ since different
orthonormal bases differ from each other by an orthogonal matrix.

\subsection{Positive Semidefinite Cone Constraints}
\label{sssec:constraint-psd-cone}

In this section, we discuss the matrix optimization problem under the
constraints of positive semidefinite cone. Such problems exist in for instance
covariance matrix estimation. We denote the set of symmetric matrices and the set of positive
semidefinite matrices in $\mathbb{R}^{p \times p}$ by $\calS^p$ and
$\calS_+^p$ respectively. We then consider the following formulation:
\begin{equation*}
    \min_{\Bv}\; \sum_{j=1}^n \ell\big(\langle \Xv_j, \Bv \rangle; y_j\big) +
    R(\Bv), \quad \text{subject to }\; \Bv \in \calS_+^p,
\end{equation*}
where $\Xv_j \in \mathbb{R}^{p \times p}$.
For $\Bv \in \mathbb{R}^{p \times p}$, consider the eigen-decomposition of
$\frac{1}{2}(\Bv + \Bv^\top) = \Qv \diag[\{d_j\}_j] \Qv^\top$,
then the projection of $\Bv$ onto $\calS_+^p$ under Frobenious norm is
\begin{equation*}
    \projv_{\calS_+^p}(\Bv) = \Qv \diag[\{(d_j)_+\}_j] \Qv^\top.
\end{equation*}
See for instance \cite{hiriart2012fresh} for the derivation.

Following the framework described in Section \ref{ssec:constrainedopt}, we need to characterize the
Jacobian of $\projv_{\calS_+^p}(\Bv)$. The nonexpansiveness of the projection
operator implies that it is differentiable almost everywhere. Let
$\vecop(\cdot)$ be a vectorization operator that transforms a  matrix in
$\mathbb{R}^{p \times p}$ into a vector in $\mathbb{R}^{p^2}$. Let $\lambda_1,
\ldots, \lambda_p$ be the eigenvalues and $\qv_1, \ldots, \qv_p$ be the
eigenvectos of matrix $\projv_{\calS^p} (\Bv) = \frac{1}{2}(\Bv + \Bv^\top)$.
Construct a matrix $\bcalQ \in
\mathbb{R}^{p^2 \times \frac{1}{2}p(p+1)}$ in the following way: the first $p$ columns of
$\bcalQ$ are given by $\vecop\big(\qv_i \qv_i^\top\big)$ for $i=1, \ldots, p$.
The next $p(p-1)/2$ columns take the form $\vecop \big(\frac{1}{\sqrt{2}} \qv_i
\qv_j^\top+ \frac{1}{\sqrt 2} \qv_j \qv_i^\top \big)$ for $1 \leq i < j \leq
p$. The Jacobian of the projection is given by
\begin{equation} \label{eq:proj-psd-jacob}
    \Jv = \Jv_1 \Jv_2,
\end{equation}
where 
\begin{equation*}
    \Jv_1 = \bcalQ
    \begin{bmatrix}
        \Av_1 & 0 \\ 0 & \Av_2
    \end{bmatrix}
    \bcalQ^\top,
    \quad
    \Jv_2 =
    \begin{bmatrix}
        \Iv_p & 0 \\
        0 & \Av_4
    \end{bmatrix}.
\end{equation*}

Here $\Av_1 \in \calS^p$ is a diagonal matrix with $A_{1, ii} = 1$ if
$\lambda_i > 0$ and $0$ if $\lambda_i < 0$. $\Av_2 \in \calS^{\frac{1}{2}
p(p-1)}$ is also diagonal specified by the following rules: if $A_{2,ii}$ is
multiplied by the column $\vecop(\qv_t \qv_s^\top)$ in $\bcalQ$, then $A_{2,ii}
= \frac{(\lambda_t)_+ - (\lambda_s)_+}{\lambda_t - \lambda_s}$. $\Jv_2$ is the
Jacobian of $\projv_{\calS^p}(\Bv)$. It is not hard to see that
$\Av_4 \in \mathbb{R}^{p(p-1) \times p(p-1)}$ with $A_{4, st, st} = A_{4,
st, ts} = \frac{1}{2}$ for $1 \leq s \neq t \leq p$. This result is proved in Section
\ref{ssec:JacobianPSDcalc} of Appendix. By plugging this Jacobian in \eqref{eq:constrained-smooth-alo} we obtain the $\alo$ formula.

\subsection{$\ell_\infty$ minimization} \label{ssec:application-l_inf}
In this section, we consider the $\ell_\infty$ penalized regression problem, given by:
\begin{equation*}
    \hat{\betav} = \argmin_{\betav} \frac{1}{2} \sum_{j=1}^n (y_i -
    \xv_i^\top\betav)^2 + \lambda \norm{\betav}_\infty,
\end{equation*}
for some $\lambda > 0$. This penalty is of interest for recovering integer (or binary) solutions
of linear equations \cite{mangasarian2011}. We will use the dual method to
obtain an approximation. The dual norm of $\norm{\cdot}_\infty$ is given by
$\norm{\cdot}_1$, thus we have that the dual optimizer $\estim{\dualv} = \projv_{\Delta_X}(\yv)$,
where the polytope $\Delta_X$ is given by:
\begin{equation*}
    \Delta_X = \{ \dualv: \norm{\Xv^\top \dualv}_1 \leq \lambda \}.
\end{equation*}

To determine the face of $\Delta_X$ containing $\estim{\dualv}$, let $E =  \{
i: \Xv_i^\top \estim{\dualv} = 0 \}$. Additionally, for $i \notin E$, let $s_i
\in \{1, -1\}$ be the sign of $\Xv_i^\top\estim{\dualv}$. The face containing
$\estim{\dualv}$ is then specified by the set of affine equations:
\begin{equation*}
    \Xv_{\cdot, E}^\top\dualv = 0,
    \quad
    \sum_{i \notin E} s_i\Xv_i^\top\dualv = \lambda.
\end{equation*}
This indicates the following matrix $\Wv$ whose columns span the normal
space of the face:
\begin{equation*}
    \Wv =
    \begin{bmatrix}
        \Xv_{\cdot, E}, \sum_{j \notin E} s_j \Xv_j
    \end{bmatrix} \in \mathbb{R}^{n \times (|E|+1)}.
\end{equation*}

Hence, the Jacobian of the projection operator is $\Iv - \Wv(\Wv^\top
\Wv)^{-1}\Wv$. Let $\Hv = \Wv(\Wv^\top \Wv)^{-1}\Wv$. According to
\eqref{eq:aloformuladualgenralreg} we obtain:
\begin{equation*}
    \alo_\lambda = \frac{1}{n}\sum_{i = 1}^n d(y_i, \surrog{y}_i),  
\end{equation*}
where $\surrog{y}_i = y_i + \frac{1}{1 - H_{ii}}(\xv_i^\top\estim{\betav} -
y_i)$.

\subsection{Group Lasso} \label{ssec:application-group-lasso}
The group Lasso \cite{yuan2006group} is a method that performs model selection
and estimation in the presence of grouped variables. More formally, let $I_1,
\ldots, I_k$ be a partition of $\{1, \ldots, p\}$, representing the groups
of variables. The group lasso penalty is then given by:
\begin{equation}\label{eq:groupLASSOreg}
    R(\betav) = \sum_{j = 1}^k \lambda_j \|\betav_{I_j}\|_2,
\end{equation}
It is straightforward to confirm that $\proxv_{\|\cdot\|_2}(\uv; \tau)
= \Big(1 - \frac{\tau}{\|\uv\|_2}\Big)_+ \uv$. Now consider the following
problem:
\begin{equation*}
    \hat{\betav} = \argmin_{\betav}\;\sum_{j=1}^n \ell(\xv_j^\top\betav; y_j) + R(\betav),
\end{equation*}
where $\ell$ is twice differentiable and $R$ is given by
\eqref{eq:groupLASSOreg}. We can then use the proximal formulation in Section
\ref{sec:proxformulation} to obtain an $\alo$ formula. It is straightforward to
see that if
\begin{equation*}
    \bigg\| \estim{\betav}_{I_l} - \sum_{j=1}^n \dot{\ell}
    (\xv_j^\top\estim{\betav}; y_j)\xv_{j, I_l} \bigg\|_2 \neq \lambda_l,
    \quad
    \forall l=1, \ldots, k,
\end{equation*}
then $\proxv_R$ is differentiable at $\estim{\betav} - \sum_{j=1}^n \dot{\ell}
(\xv_j^\top\estim{\betav}; y_j)\xv_j$. Hence, the $\alo$ estimate is given by
 \begin{equation*}
    \xv_i^\top \surrogi{\betav}
    =
    \xv_i^\top\estim{\betav} + \frac{H_{ii}}{1 -
    H_{ii}\ddot{\ell}(\xv_i^\top \estim{\betav}; y_i)}
    \dot{\ell}(\xv_i^\top \estim{\betav}; y_i),
\end{equation*}
with 
\begin{equation*}
    \Hv = \Xv\big(\Jv \Xv^\top \diag[\{\ddot{\ell}(\xv_j^\top\estim{\betav};
    y_j)\}_j] \Xv + \Iv - \Jv \big)^{-1} \Jv \Xv^\top.
\end{equation*}
The Jacobian matrix $\Jv$ is a block diagonal matrix of the form
\begin{equation*}
    \Jv =
    \begin{bmatrix}  
        \Jv_1& \bm{0}& \cdots & \bm{0} \\
        \bm{0}& \Jv_2 & \ldots & \bm{0}\\
        \vdots & \vdots & \ddots & \vdots \\
        \bm{0}& \bm{0} & \cdots & \Jv_k
    \end{bmatrix}. 
\end{equation*}

If $\big\|\estim{\betav}_{I_l} - \sum_{j=1}^n \dot{\ell}
(\xv_j^\top\estim{\betav}; y_j)\xv_{j, I_l}\big\|_2 < \lambda_l$, then $\Jv_l =
\bm{0}$. Otherwise it is given by 
\begin{equation*}
    \Jv_l =
    \bigg( 1 - \frac{\lambda_l}{\|\uv\|_2}\bigg) \Iv
    + \frac{\lambda_l}{\|\uv\|_2^3} \uv\uv^\top.
\end{equation*}
where $\uv = \estim{\betav}_{I_l} - \sum_{j=1}^n \dot{\ell}
(\xv_j^\top\estim{\betav}; y_j) \xv_{j, I_l}$.

This formula can be simplified further. Let $E=\bigcup_{\Jv_l \neq \bm{0}}
I_l$. Then we can simplify the expression of $\Hv$ using the matrix inverse
formula as follows:
\begin{align*}
    \Hv =&
    \Xv_{\cdot, E} \big[\Jv_{E, E} \Xv_{\cdot, E}^\top \diag[\{\ddot{\ell}
    (\xv_j^\top\estim{\betav}; y_j)\}_j] \Xv_{\cdot, E}
    + \Iv_{E, E} - \Jv_{E, E}\big]^{-1} \Jv_{E, E} \Xv_{\cdot, E}^\top
    \nonumber \\
    =&
    \Xv_{\cdot, E} \big[\Xv_{\cdot, E}^\top \diag[\{\ddot{\ell}
    (\xv_j^\top\estim{\betav}; y_j)\}_j] \Xv_{\cdot, E}
    + \Jv_{E, E}^{-1} - \Iv_{E, E} \big]^{-1} \Xv_{\cdot, E}^\top
\end{align*}
We note that $\Jv_{E, E}^{-1} - \Iv_{E, E}$ is also a block diagonal matrix
with each block being of the form $\Jv_l^{-1} - \Iv$. Since $\estim{\betav}_{I_l} -
\sum_{j=1}^n \dot{\ell} (\xv_j^\top\estim{\betav}; y_j) \xv_{j, I_l} =
\Big(1 + \frac{\lambda_l}{\|\estim{\betav}_{I_l}\|_2}\Big)
\estim{\betav}_{I_l}$, we have $\Jv_l =
\frac{\|\estim{\betav}_{I_l}\|_2}{\|\estim{\betav}_{I_l}\|_2 + \lambda_l}
\Big( \Iv + \frac{\lambda_l \estim{\betav}_{I_l} \estim{\betav}_{I_l}^\top}
{\|\estim{\betav}_{I_l}\|_2^3} \Big)$. This finally leads to
\begin{equation*}
    \Jv_l^{-1} - \Iv
    =
    \frac{\lambda_l}{\|\estim{\betav}_{I_l}\|_2}
    \bigg( \Iv - \frac{\estim{\betav}_{I_l}\estim{\betav}_{I_l}^\top}
    {\|\estim{\betav}_{I_l}\|_2^2}\bigg).
\end{equation*}

\subsection{SLOPE} \label{ssec:application-slope}
The SLOPE (sorted $\ell_1$ penalized estimation) technique is proposed in
\cite{bogdan2015slope}. It combines the intuition from high-dimensional estimation
and multiple testing to consider the sorted $\ell_1$ penalty, which is denoted by
$\|\cdot\|_S$ and defined as:
\begin{equation*}
    \|\betav\|_S = \sum_{i = 1}^p \lambda_i |\beta|_{(i)},
\end{equation*}
where $\lambda_1 \geq \lambda_2 \geq \ldots \geq \lambda_p \geq 0$ is a chosen
sequence, $|\beta|_{(i)}$ denotes the $i^{\rm th}$ largest element in absolute
value of $\betav$. Note that the sorted $\ell_1$ penalty is indeed a norm
\cite{bogdan2015slope}.

We will use the dual approach in Section \ref{sec:approximatedual} to obtain an
$\alo$ estimate. Let us consider the $\ell_2$ loss function. As the first step, we
need to characterize the dual norm $\|\cdot\|_{S*}$ of $\|\cdot\|_S$. According
to \cite{bogdan2015slope}, we have that
\begin{equation*}
    \| \betav \|_{S*}
    = \max_{1 \leq j \leq p} \frac{\sum_{l=1}^j \abs{\beta}_{(l)}}
    {\sum_{l=1}^j \lambda_l}.
\end{equation*}

The dual optimizer then satisfies $\estim{\dualv} = \projv_{\Delta_X}(\yv)$,
where $\Delta_X$ is the polytope $\Delta_X = \big\{ \dualv: \| \Xv^\top
\dualv\|_{S*} \leq 1 \big\}$. In order to obtain the Jacobian of the
projection, we should identify the face of $\Delta_X$ containing
$\estim{\dualv}$. Define
\begin{equation*}
    E =
    \bigg\{ j: \frac{\sum_{l=1}^j \abs{\Xv_{k_l}^\top \estim{\dualv}}}
    {\sum_{l=1}^j \lambda_l} = 1 \bigg\},
\end{equation*}
where $\{ k_1, \ldots, k_p\}$ is a permutation of $\{1, \ldots, p\}$ such that
$|\Xv_{k_1}^\top \estim{\dualv}| \geq \ldots \geq |\Xv_{k_p}^\top \estim{\dualv}|$.
Let $s_i \in \{1, -1\}$ be the sign of $\Xv_{k_i}^\top \estim{\dualv}$, then
the face of $\Delta_X$ containing $\hat{\thetav}$ is determined by a set of linear equations:
\begin{equation*}
    \sum_{i=1}^j s_i \Xv_{k_i}^\top \estim{\dualv} = \sum_{i=1}^j \lambda_i,
    \quad
    \text{for } j \in E.
\end{equation*}

This suggests the following construction of the matrix $\Wv \in \mathbb{R}^{n
\times |E|}$ whose columns expand the normal space of the face containing
$\estim{\dualv}$. Let $\Zv = \big[ \Xv_{k_1}, \ldots, \Xv_{k_p} \big]$, i.e.,
a matrix composed of the permuted columns of $\Xv$. Set $\Wv = \Zv \Av$ where
each column of $\Av$ corresponds to exactly one $j \in E$. For $j_0 \in E$, its
corresponding column of $\Av$ can be specified as (by abusing the notation $\Av_{j_0}$)
\begin{equation*}
    A_{t, j_0} =
    \begin{cases}
        s_{t} & \text{if } t \leq j_0, \\
        0 & \text{otherwise.}
    \end{cases}
\end{equation*}
Finally, we put $\Hv = \Wv (\Wv^\top \Wv)^{-1} \Wv^\top$ and obtain the leave-$i$-out
predicted value as $\tilde{y}_i = y_i + \frac{y_i - \hat{y}_i}{1 - H_{ii}}$.

\section{Numerical Experiments} \label{sec:experiments}

We illustrate the performance of ALO through three experiments. The first one
(Section \ref{ssec:experiment-alo-accuracy}) compares the ALO risk estimate
with that of LOOCV. The second one (Section \ref{ssec:experiment-alo-timing})
discusses the computational complexity of ALO, LOOCV and 5-fold CV. Our last
experiment (Section \ref{ssec:real-world-data}) evaluates the performance of
ALO on real-world datasets.

\subsection{Evaluating the Accuracy of ALO on Simulated Data}
\label{ssec:experiment-alo-accuracy}

In this section, we run ALO and LOOCV for different models under different
settings to compare the accuracy of ALO as an approximation of LOOCV.
Since all the models we considered contain a tuning parameter $\lambda$, the
accuracy is examined against different values of $\lambda$.

In the first part (Figure \ref{fig:risk-alo-loo1}), we run ALO and LOOCV for
seven models studied in Section \ref{sec:applications} under iid Gaussian
design and without including the intercept. Their risk estimates are compared
under the settings $ n> p$ and $n < p$ respectively.
The details of the simulations are explained in Section
\ref{sssec:experiment-alo-accuracy-no-intercept}. In general, we observe that
the estimates given by ALO are close to LOOCV, although the performance may
deteriorate for very small values of $\lambda$, as is clear in the fused-LASSO
($n<p$) and $\ell_\infty$ norm ($n < p$) examples. These values of $\lambda$
correspond to ``dense'' solutions, and are not close to the optimal choice. Hence,
such inaccuracies do not harm the parameter tuning algorithm.

For the second part (Figure \ref{fig:risk-alo-loo2}), we consider the risk
estimates for LASSO from ALO and LOOCV under settings with model
mis-specification, heavy-tail noise and correlated design. As is clear from
Figure \ref{fig:risk-alo-loo2}, for all three cases, ALO approximates LOOCV
well. Note that we choose $n < p$ for these three settings, and again for very small
value of $\lambda$, the ALO risk estimates skew upward slightly compared to
LOOCV risk estimates. The details of the simulations are given in Section
\ref{ssec:simdetailscorr}.

The third part (Figure \ref{fig:risk-alo-loo-intercept}) justifies the ALO
formula on models involving intercepts, as
presented in Section \ref{sec:intercept}. We include three examples: LASSO, SVM and Ridge
regression with positive quadrant constraint, which correspond to the nonsmooth
regularizer, nonsmooth loss and constrained problem respectively. Our adaption
proposed in Section \ref{sec:intercept} works well on these three models. The
details of the simulation are provided in Section \ref{sssec:experiments-iid-intercept}.

\subsubsection{IID Gaussian design without Intercept}
\label{sssec:experiment-alo-accuracy-no-intercept}

In this section, we summarize the details of the simulations whose results are presented in Figure \ref{fig:risk-alo-loo1}.
\paragraph{Support Vector Machine}
For all SVM simulations the data is generated according to a Gaussian logistic
model: the design matrix $\Xv$ is generated as a matrix of i.i.d.
$\mathcal{N}(0, 1)$; the true parameter $\betav$ is i.i.d. $\mathcal{N}(0, 9)$,
and each response $y_i$ is generated as an independent Bernoulli with
probability $p_i$ given by the following logistic model:
\begin{equation*}
    \log \frac{p_i}{1 - p_i} = \xv_i^\top \betav.
\end{equation*}

The $n > p$ scenario is generated with $n = 300$ and $p = 80$, and the $n < p$
scenario is generated with $n = 300$ and $p = 600$. We consider a sequence of 
$40$ different values of $\lambda$ ranging between $e^{4}\sim e^{12}$, with their logarithm equally
spaced between $[4, 12]$. The model is fitted using the \texttt{sklearn.svm.linearSVC} function in Python
package \texttt{scikit-learn} \cite{scikit-learn}, which is implemented by the
\texttt{LibSVM} package \cite{chang2011libsvm}. For using the \texttt{sklearn.svm.linearSVC}, we set
\texttt{tolerance=$10^{-6}$} and
\texttt{max\_iter=10000}. We identify an observation as a support vector
if $|1 - y_i \xv_i^\top \estim{\betav}| < 10^{-5}$.

\paragraph{Fused LASSO}
We use the fused LASSO \cite{tibshirani2005fused} as a special case of genralized LASSO.
For the fused LASSO experiment, each component of the design matrix $\Xv$ is generated
from i.i.d. $\mathcal{N}(0, 0.05)$. We generated the true parameter $\betav$ through the following process: given a number $k<p$, we generate a
sparse vector $\betav_0$ with a random sample of $k$ of its components i.i.d.
from $\mathcal{N}(0, 1)$. Then we construct a new vector $\betav_1$ as the
cumulative sum of $\betav_0$: $\beta_{1, i} = \sum_{j=1}^i \beta_{0, j}$;
Finally we normalize $\betav_1$ such that it has standard deviation 1. Note
that $\betav_1$ is a piecewise constant vector.
The response $\yv$ is generated as $\yv = \Xv\betav + \bm{\epsilon}$, where
$\bm{\epsilon}$ denotes i.i.d. random gaussian noise from $\mathcal{N}(0, 0.25)$.
For our simulation, we use $k=20$ (so piecewise constant with 20 pieces). The
$n > p$ scenario is generated with $n = 200$ and $p = 100$, whereas the $n < p$
scenario is generated with $n = 200$ and $p = 400$.

The model is fitted through a direct translation of the generalized LASSO model
into the package \texttt{CVX} \cite{cvx}. We use the default tolerance and
maximal iteration. We identify the location $i$ such that $\estim{\beta}_{i+1}
= \estim{\beta}_{i}$ by checking if $|\estim{\beta}_{i + 1} - \estim{\beta}_i|
< 10^{-8}$. For $n > p$, we consider a sequence of 40 tuning parameters from
$10^{-2} \sim 10^2$; For $n < p$, we consider a sequence of 30 tuning
parameters from $10^{-1} \sim 10$. Both are equally spaced on the log-scale.

\paragraph{Nuclear Norm Minimization}

For the nuclear norm simulations the data is generated according to the Gaussian
low-rank model; each observation matrix $\Xv_j$ is generated as an i.i.d. $\mathcal{N}(0,1)$ matrix. The true
parameter matrix $\Bv$ is generated as a low rank matrix, by setting $k=1$ in
the following formula
\begin{equation*}
    \Bv = \sum_{l = 1}^k \zv_l \wv_l^\top,
\end{equation*}
where $\zv, \wv$ are independent of each other. $\zv \sim \mathcal{N}(0,
\Iv_{p_1})$, $\wv \sim \mathcal{N}(0, \Iv_{p_2})$. Hence, the rank of
$\Bv$ in our experiments is equal to $1$.
The response $\yv$ is generated as $y_j = \langle \Xv_j, \Bv \rangle + \epsilon_j$,
where $\epsilon_j$ is i.i.d. $\mathcal{N}(0, 0.25)$.

The $n > p$ scenario is generated with $n = 600$, and $\Bv \in \RR^{20 \times 20}$ (i.e.
$p = 400$). The $n < p$ scenario is generated with $n = 200$, and $\Bv \in
\RR^{20 \times 20}$ again. For both settings, we consider a
sequence of 30 tuning parameters from $5\times 10^{-1}\sim 5\times 10$, equally
spaced on the log-scale.

The model is fitted using an implementation of a proximal gradient algorithm as
described in \cite{Lan2011}, implemented using the Matlab package
\texttt{TFOCS} \cite{tfocs}. The threshold we use to identify singular values
with value 0 is $10^{-3} \times \lambda_{\max}(\estim{\Bv})$, where
$\lambda_{\max}$ is the maximal singular value of $\estim{\Bv}$.

\paragraph{Group LASSO}
For the group LASSO experiment, each component of the design matrix $\Xv \in
\mathbb{R}^{n \times p}$ is generated from i.i.d. $\mathcal{N}(0,
\frac{1}{n})$. We generated the true parameter $\betav$ through the following
process: given a number $k<p$, we randomly select $k$ components and generate
their values from \texttt{Uniform[-3, 3]}. The rest of them are set to be 0.

The response $\yv$ is generated as $\yv = \Xv\betav + \bm{\epsilon}$, where
$\bm{\epsilon}$ denotes i.i.d. random gaussian noise from $\mathcal{N}(0, 0.64)$.
For our simulation, we use $k=50$. The $n > p$ scenario is generated with $n =
300$ and $p = 150$, whereas the $n < p$ scenario is generated with $n = 300$
and $p = 600$. We use 15 equally spaced groups for both settings.

We implemented a proximal gradient descent algorithm to fit the model.
We identify those groups with their norms small than $10^{-6}$.
For both $n > p$ and $n < p$, we consider a sequence of 20 tuning parameters from
$10^{-2} \sim 10^2$, equally spaced on log-scale.

\paragraph{$\ell_\infty$ norm}
For the $\ell_\infty$-norm experiment, we generated the data using $\yv = \Xv \betav + \epsilon$. For $\Xv$ we have
$X_{ij} \overset{iid}{\sim} \frac{1}{\sqrt{n}} \mathcal{N}(0, 1)$; For $\betav$
we randomly pick $p - k$ out of $p$ components from \texttt{Uniform[-3, 3]},
then the remaining $k$ components are with equal probability chosen from $\{-3, 3\}$. Finally, the noise $\epsilon_j
\overset{iid}{\sim} 0.8 \mathcal{N}(0, 1)$. We use $n=900$, $k=225$ and $p=450,
1800$. We describe the method we used for solving this optimization problem in Section \ref{sec:discussion}. 

\paragraph{Ridge regression with positive quadrant constraint}
To examine the accuracy of the ALO formula on models with polyhedron constraint, we consider the
following optimization problem:
\begin{equation}\label{eq:ridgewithpos}
    \estim{\betav} = \argmin_{\betav} \frac{1}{2} \| \yv - \Xv \betav \|_2^2 + \lambda
    \|\betav\|_2^2,
    \quad
    \text{subject to } \betav_j \geq 0, \text{ for } 1 \leq j \leq n.
\end{equation}
The data generating process is based on $\yv = \Xv\betav_0 + \epsilonv$, where
$\Xv$ has iid elements from $\frac{1}{\sqrt{n}} \mathcal{N}(0, 1)$, $\betav$
has iid components from \texttt{Uniform[-1, 3]}. $\epsilonv$ also has iid
elements from $\mathcal{N}(0, 4)$. $n$ is set to $300$. Two values of $p$ are also considered: $p=600$ and $p=150$. 

To solve the optimization problem \eqref{eq:ridgewithpos}, we use the projected gradient descent. Then we follow the discussion of Section \ref{ssec:application-polyhedron}; We find $E=\{k: \estim{\beta}_k > 0\}$. A natural choice for the orthonormal basis
of the tangent space on the first quadrant at $\estim{\betav}$ is specified by
$\{\ev_j: j \in E\}$. Here $\ev_j$ is the canonnical basis for Euclidean space.
Then we can use the result in Section \ref{ssec:application-polyhedron} to
obtain the ALO formula.

\paragraph{Positive semidefinite cone constraint}
For the positive semidefinite cone constraint, we consider the following
optimization problem:
\begin{equation*}
    \min_{\Bv} \frac{1}{2} \sum_{j=1}^n (y_j - \langle \Xv_j, \Bv \rangle)^2 +
    \lambda \|\Bv\|_F^2,
    \quad
    \text{subject to } \Bv \in \calS_+^p.
\end{equation*}
The data generation process is based on $\yv_j = \langle \Xv_j, \Bv_0 \rangle +
\epsilon_j$ for $1 \leq j \leq n$ where $\Xv_j \in \mathbb{R}^{p \times p}$ has iid
elements from $\frac{1}{\sqrt{n}}\mathcal{N}(0, 1)$. $\Bv_0 = \Cv^\top\Cv +
\diag[\dv]$ with $\Cv \in \mathbb{R}^{p \times p}$ having elements from iid $\mathcal{N}(0, 1)$ and $\dv$
having elements from $p\mathcal{N}(0, 1)$. $\epsilon_j \sim \mathcal{N}(0,
49)$. Finally we use $n=300$ and $p=10, 20$.

To solve the optimization problem, a projected gradient descent algorithm is implemented to solve the problem. For the ALO formula we directly use \eqref{eq:constrained-smooth-alo} with
$\Jv$ specified as in \eqref{eq:proj-psd-jacob}.
\begin{figure}[!htbp]
    \begin{center}
        \setlength\tabcolsep{2pt}
        \renewcommand{\arraystretch}{0.3}
        \begin{tabular}{r|rrrr}
            \hline
            \hline
            & \multicolumn{1}{c}{\small svm}
            & \multicolumn{1}{c}{\small fused lasso}
            & \multicolumn{1}{c}{\small nuclear norm}
            & \multicolumn{1}{c}{\small group lasso} \\
            \hline
            \rotatebox{90}{\small \hspace{1.6cm} $n > p$}
            & \includegraphics[scale=0.34]{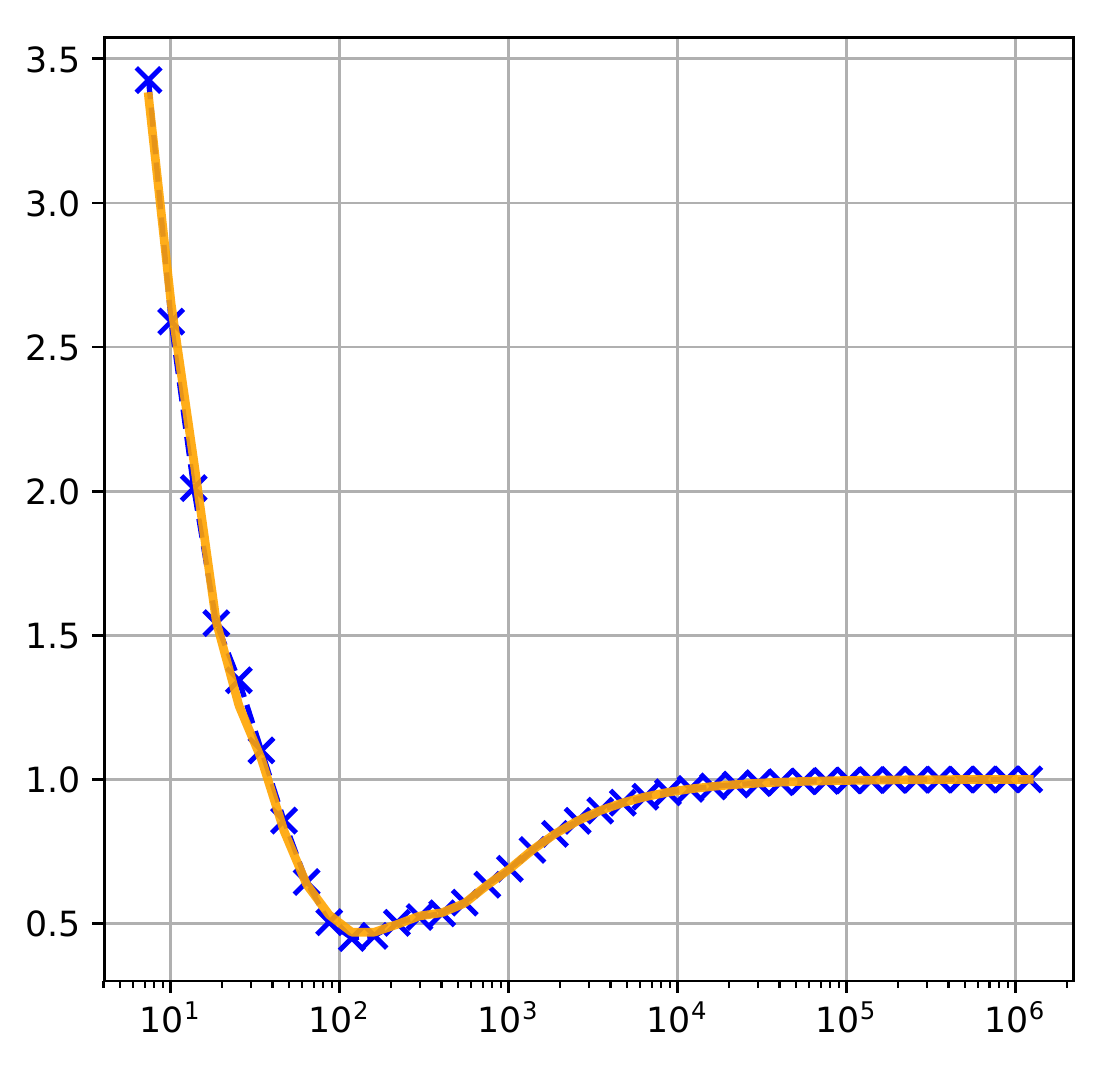}
            & \includegraphics[scale=0.34]{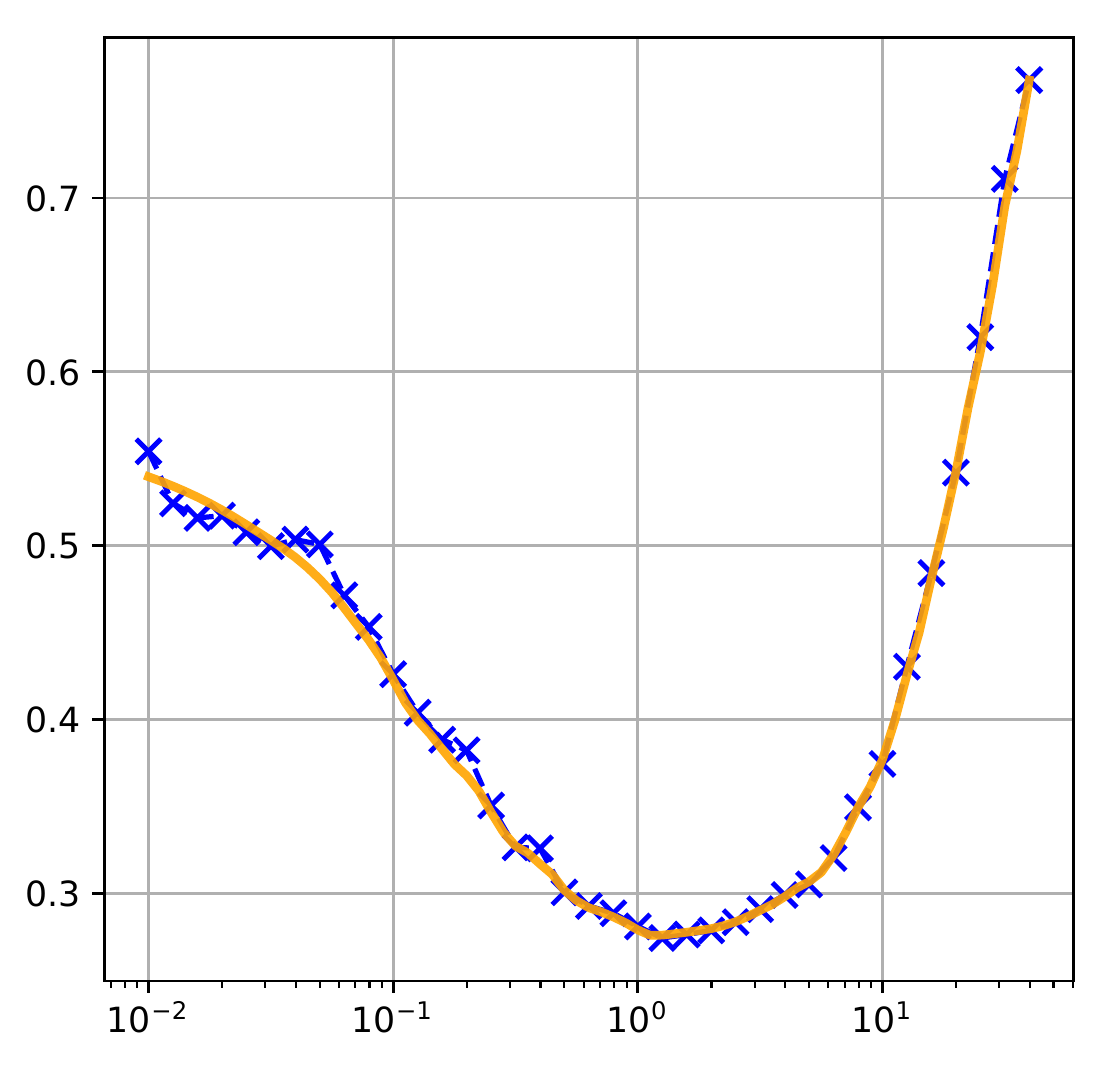}
            & \includegraphics[scale=0.34]{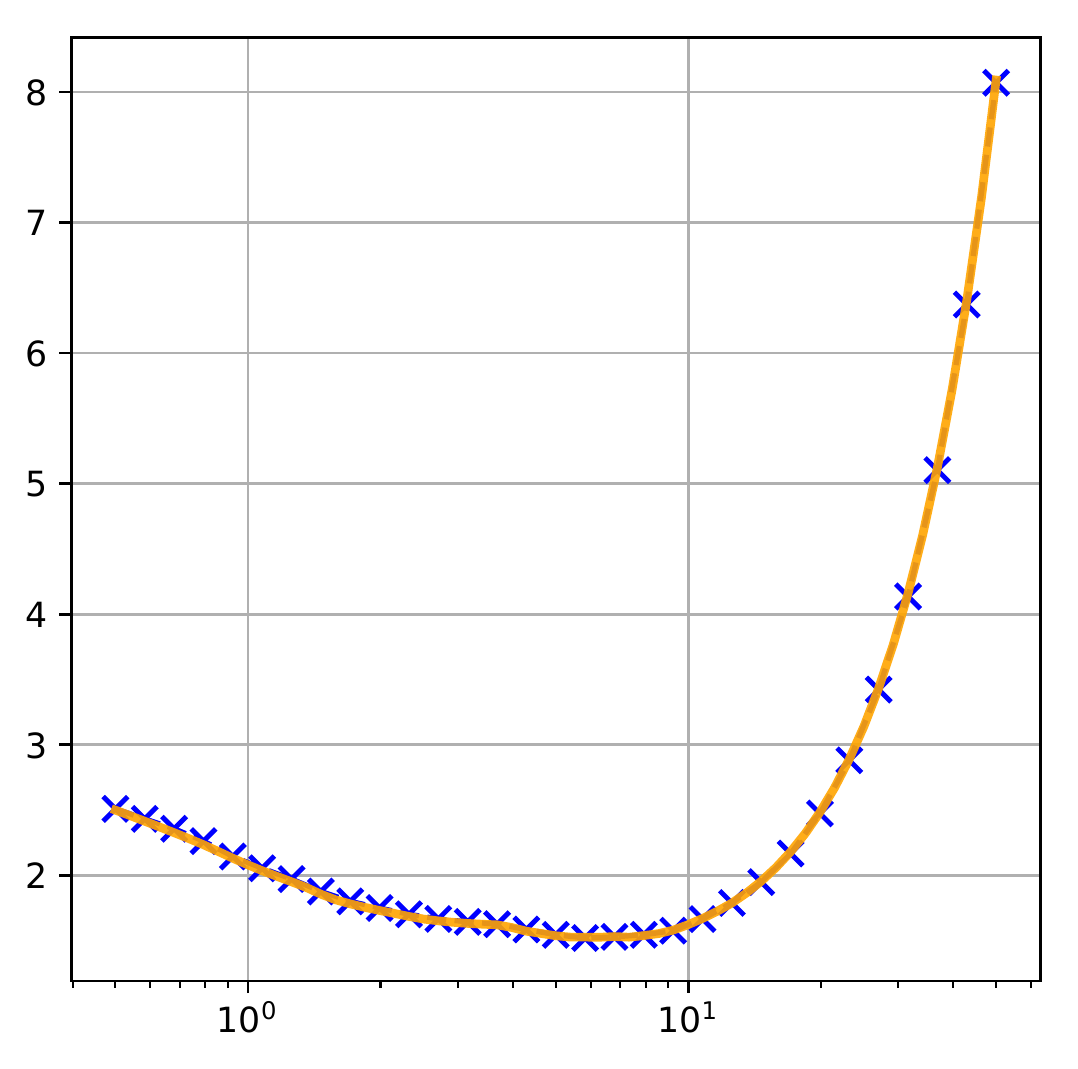}
            & \includegraphics[scale=0.34]{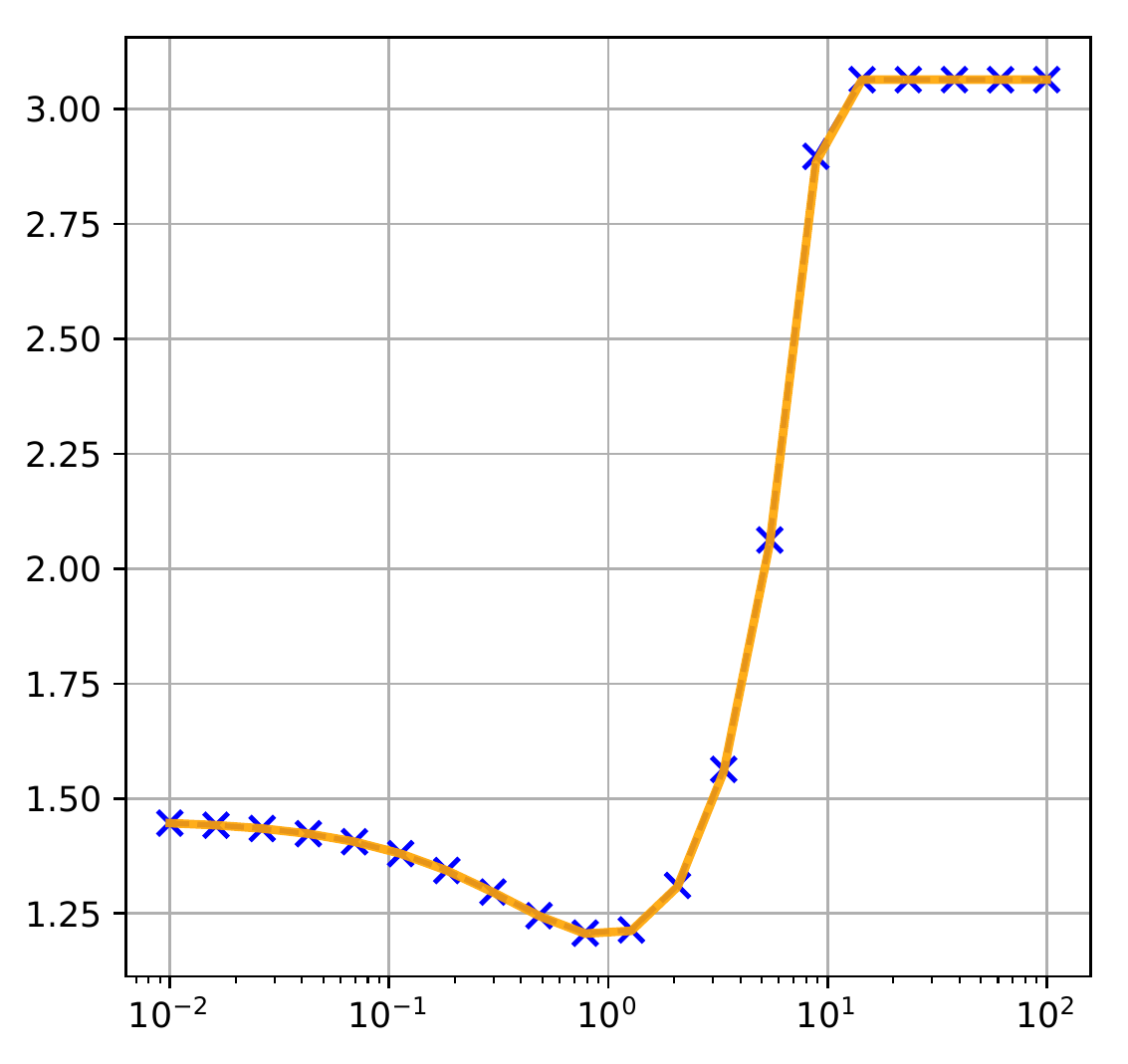} \\
            & \multicolumn{1}{c}{{\fontsize{7}{6} $n=300, p=80$}}
            & \multicolumn{1}{c}{{\fontsize{7}{6} $n=200, p=100$}}
            & \multicolumn{1}{c}{{\fontsize{7}{6} $n=600, p_1=p_2=20$}}
            & \multicolumn{1}{c}{{\fontsize{7}{6} $n=300, p=150$}} \\
            \rotatebox{90}{\small \hspace{1.6cm} $n < p$}
            & \includegraphics[scale=0.34]{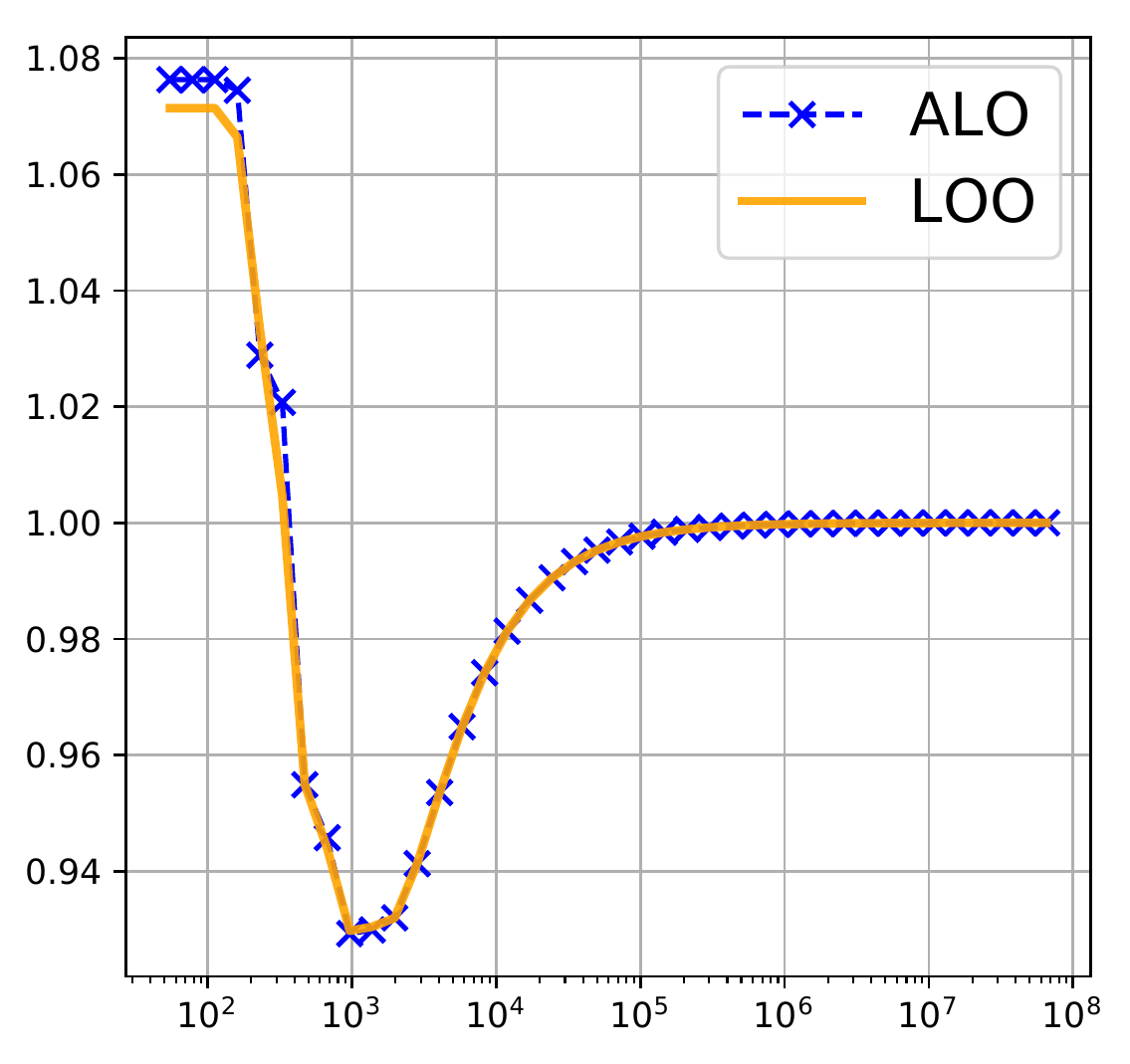}
            & \includegraphics[scale=0.34]{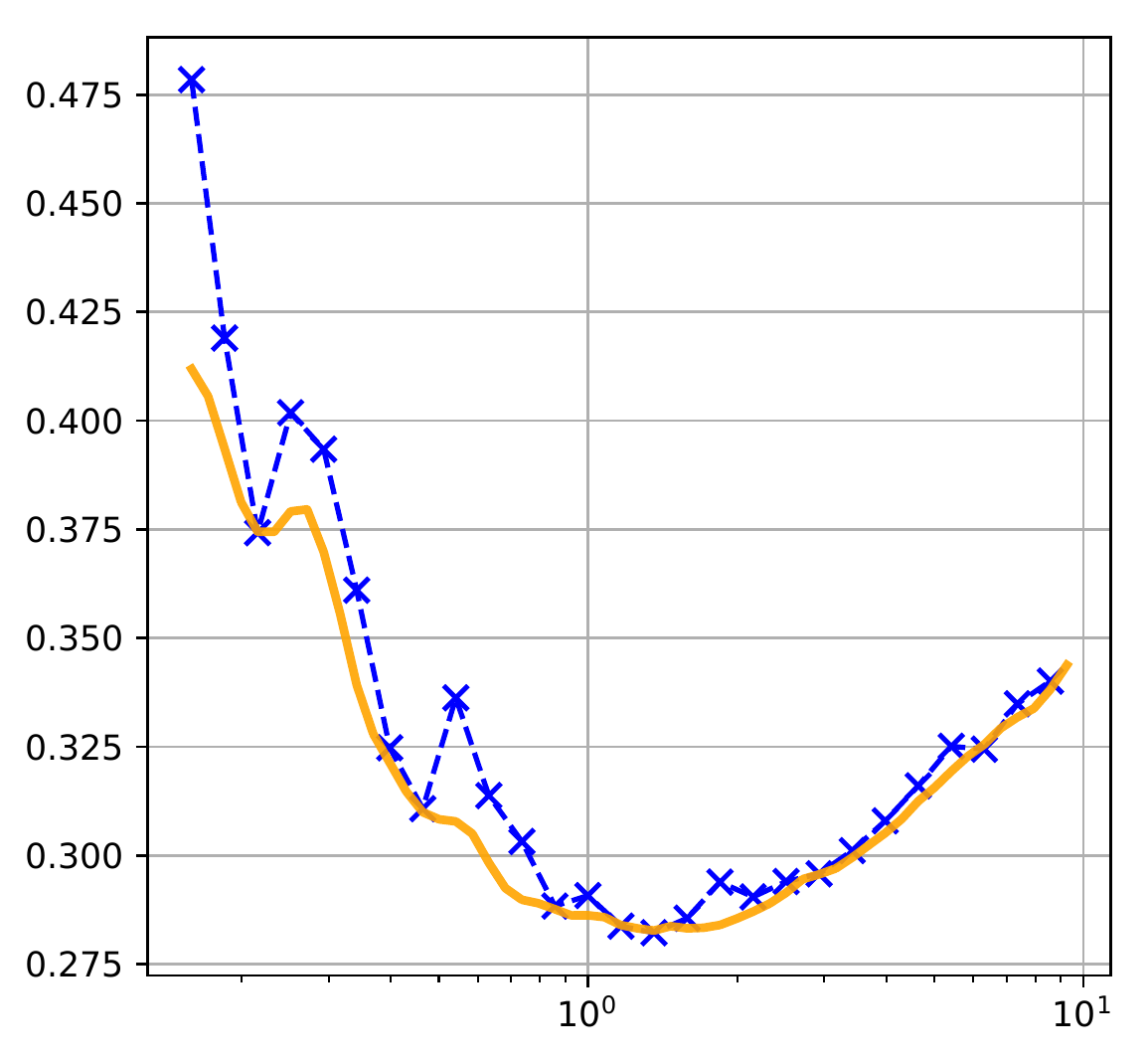}
            & \includegraphics[scale=0.34]{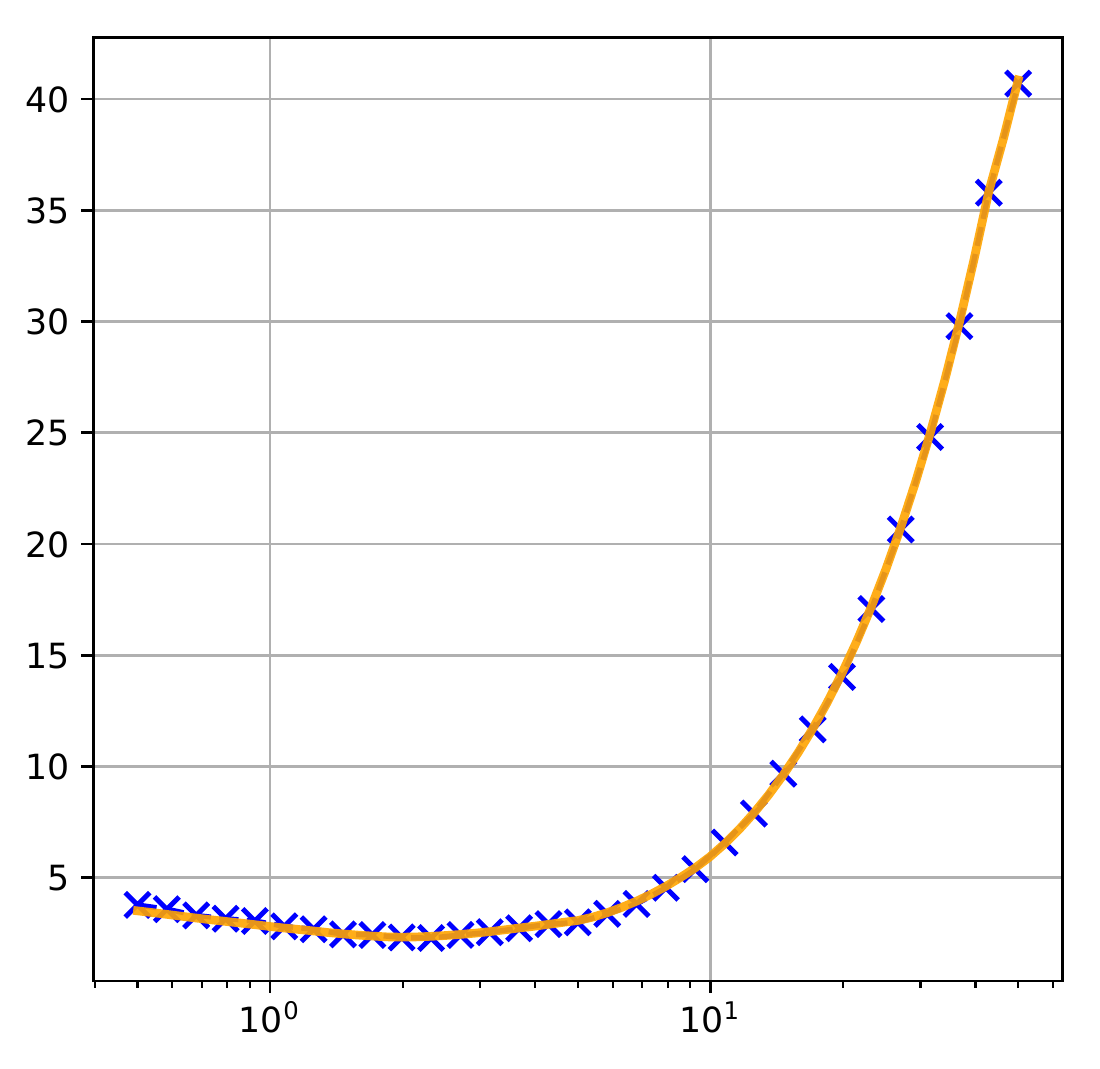}
            & \includegraphics[scale=0.34]{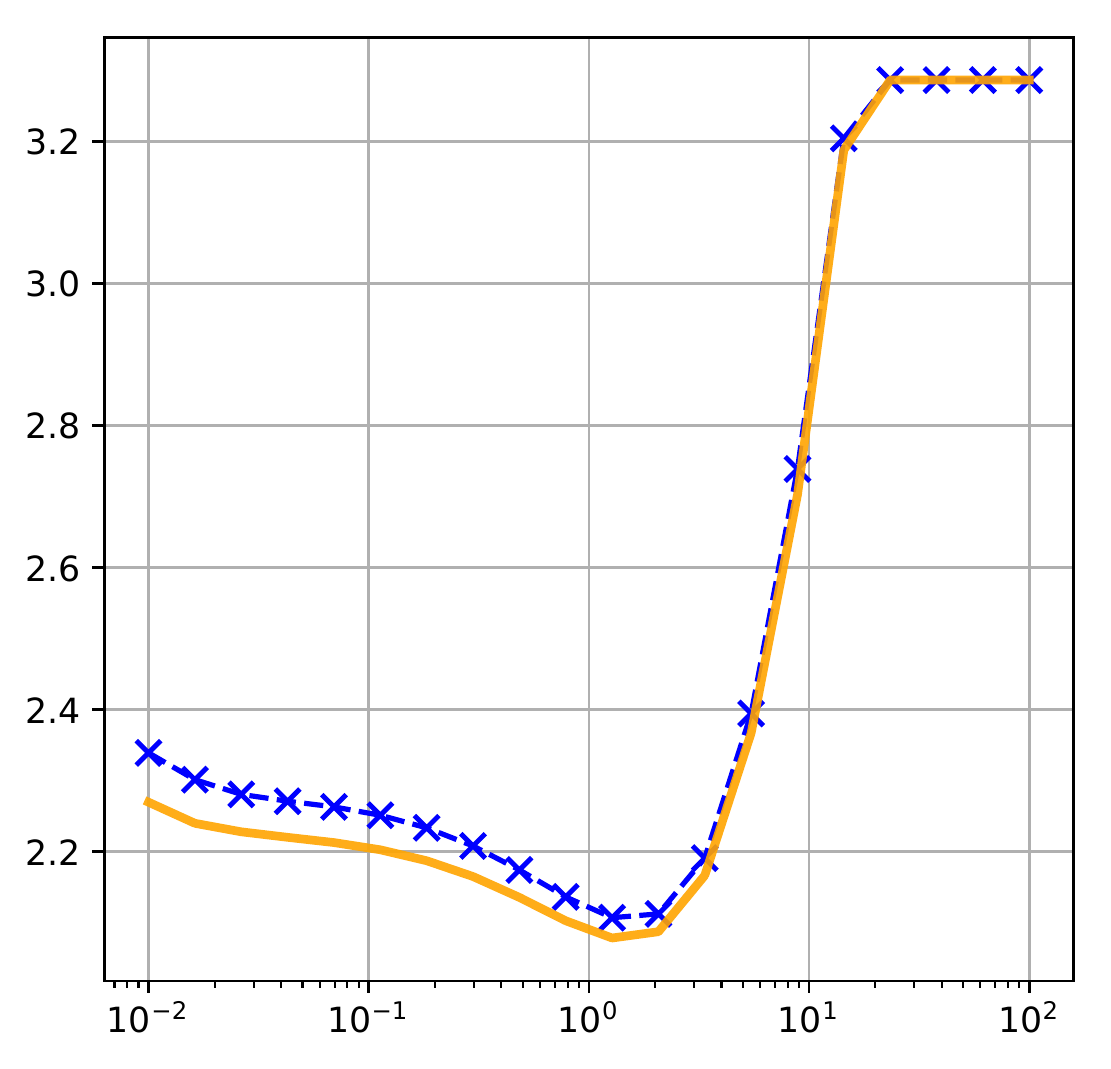} \\
            & \multicolumn{1}{c}{{\fontsize{7}{6} $n=300, p=600$}}
            & \multicolumn{1}{c}{{\fontsize{7}{6} $n=200, p=400$}}
            & \multicolumn{1}{c}{{\fontsize{7}{6} $n=200, p_1=p_2=20$}}
            & \multicolumn{1}{c}{{\fontsize{7}{6} $n=300, p=600$}} \\
            &&&& \\
            \hline
            \hline
            & \multicolumn{1}{c}{\small $\ell_\infty$ norm}
            & \multicolumn{1}{c}{\small positive ridge}
            & \multicolumn{1}{c}{\small psd ridge} \\
            \hline
            \rotatebox{90}{\small \hspace{1.6cm} $n > p$}
            & \includegraphics[scale=0.34]{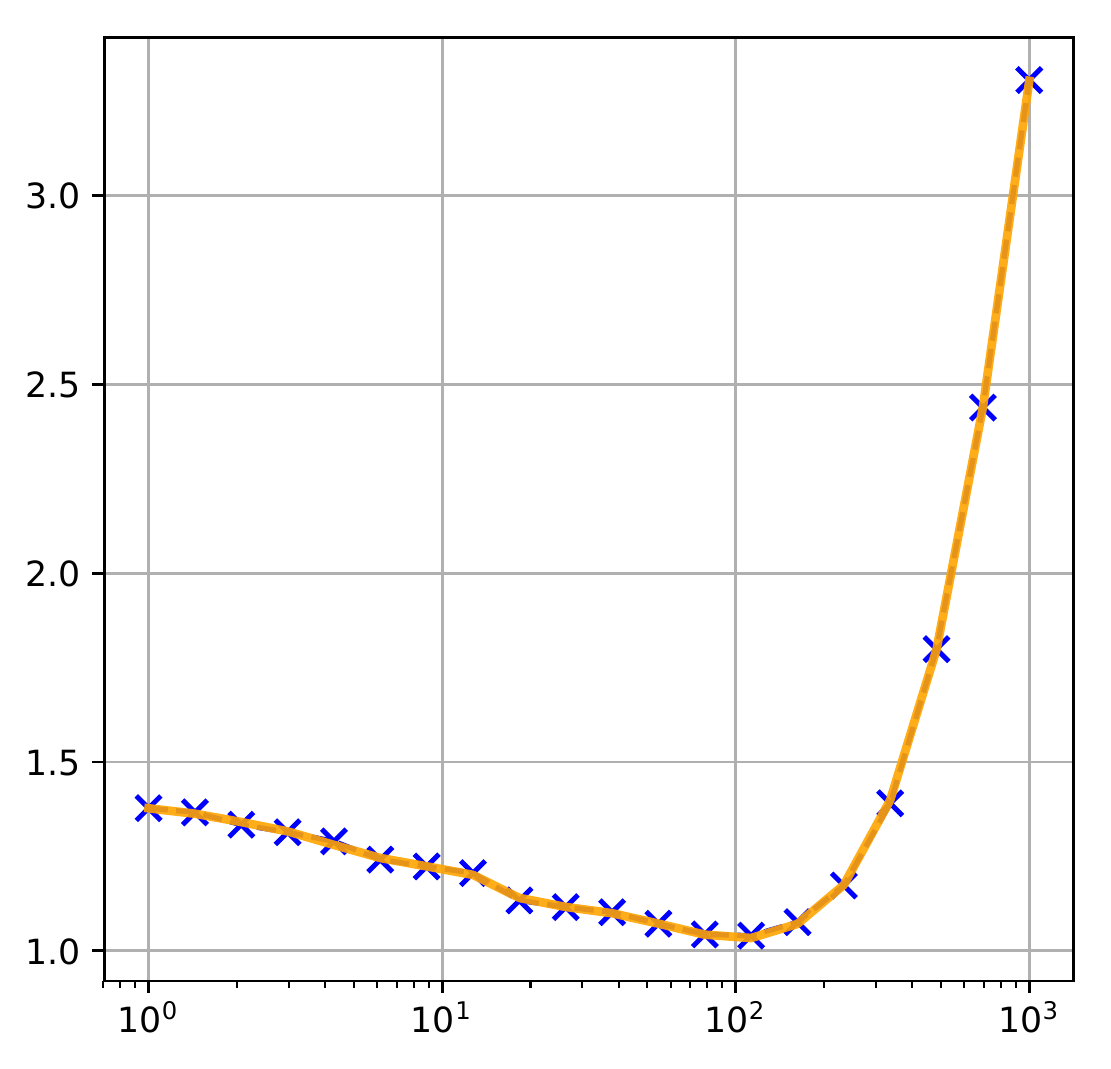}
            & \includegraphics[scale=0.34]{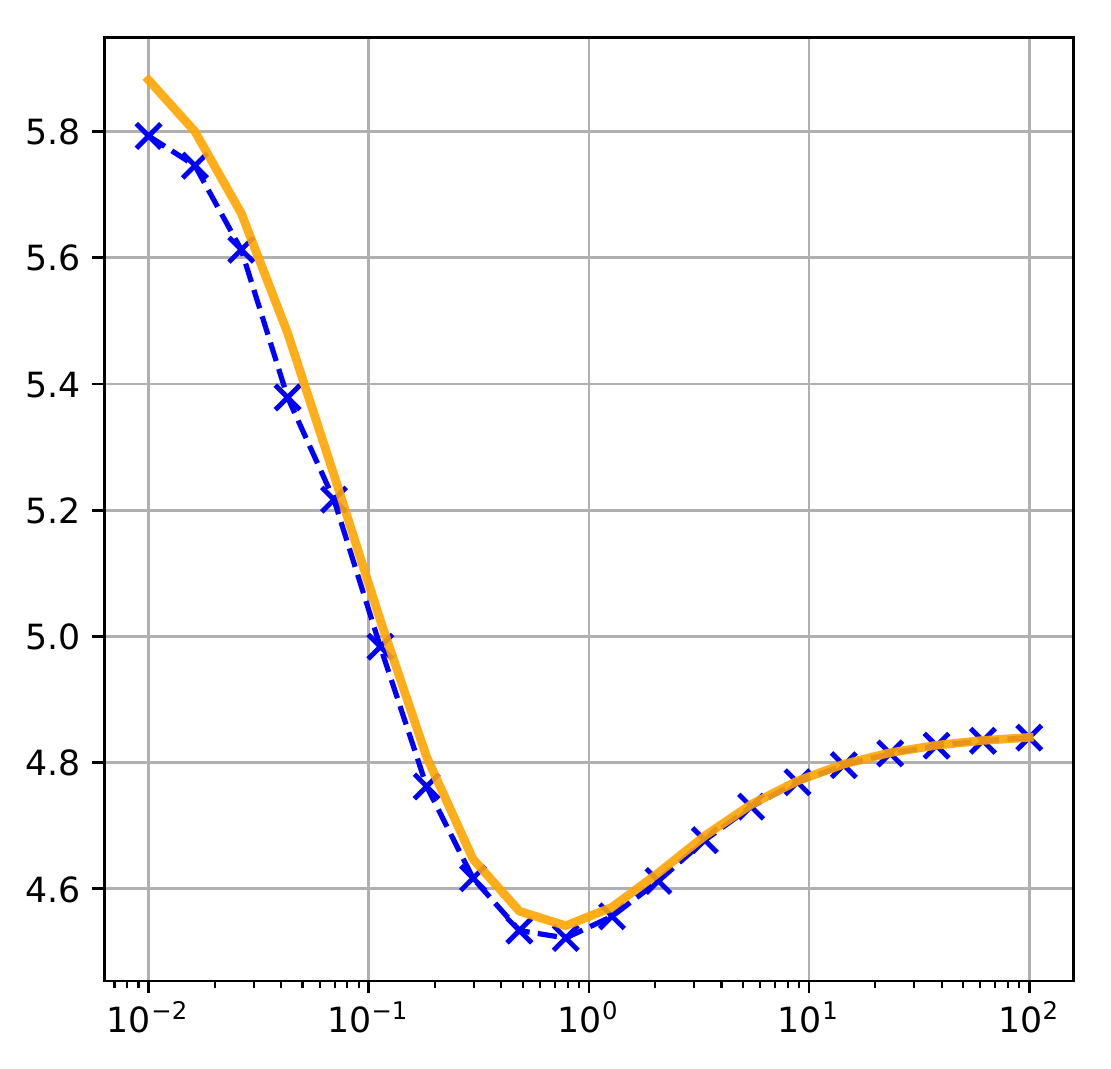}
            & \includegraphics[scale=0.34]{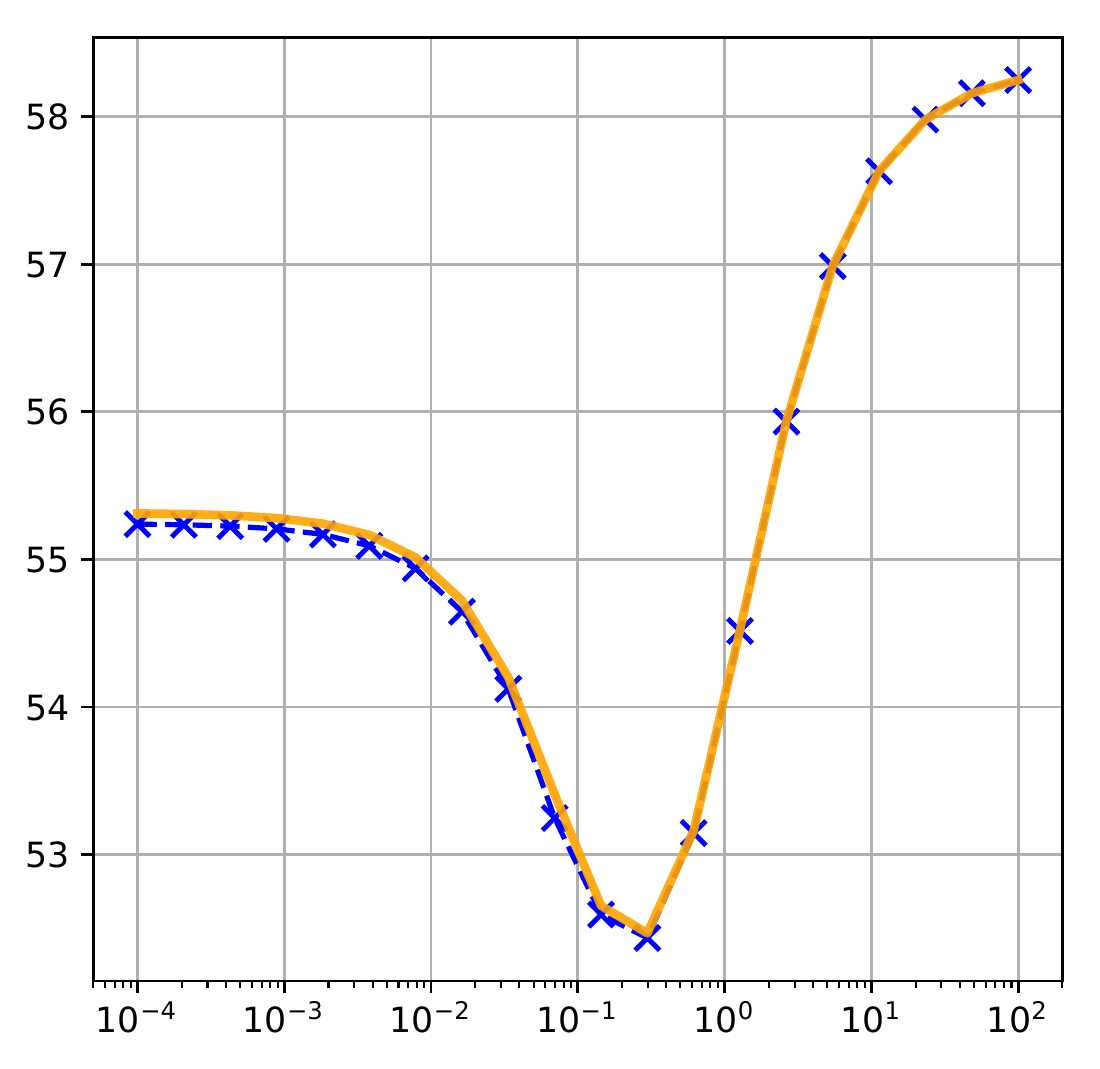}
            & \\
            & \multicolumn{1}{c}{{\fontsize{7}{6} $n=900, p=450$}}
            & \multicolumn{1}{c}{{\fontsize{7}{6} $n=300, p=150$}}
            & \multicolumn{1}{c}{{\fontsize{7}{6} $n=300, p_1=p_2=10$}}
            & \\
            \rotatebox{90}{\small \hspace{1.6cm} $n < p$}
            & \includegraphics[scale=0.34]{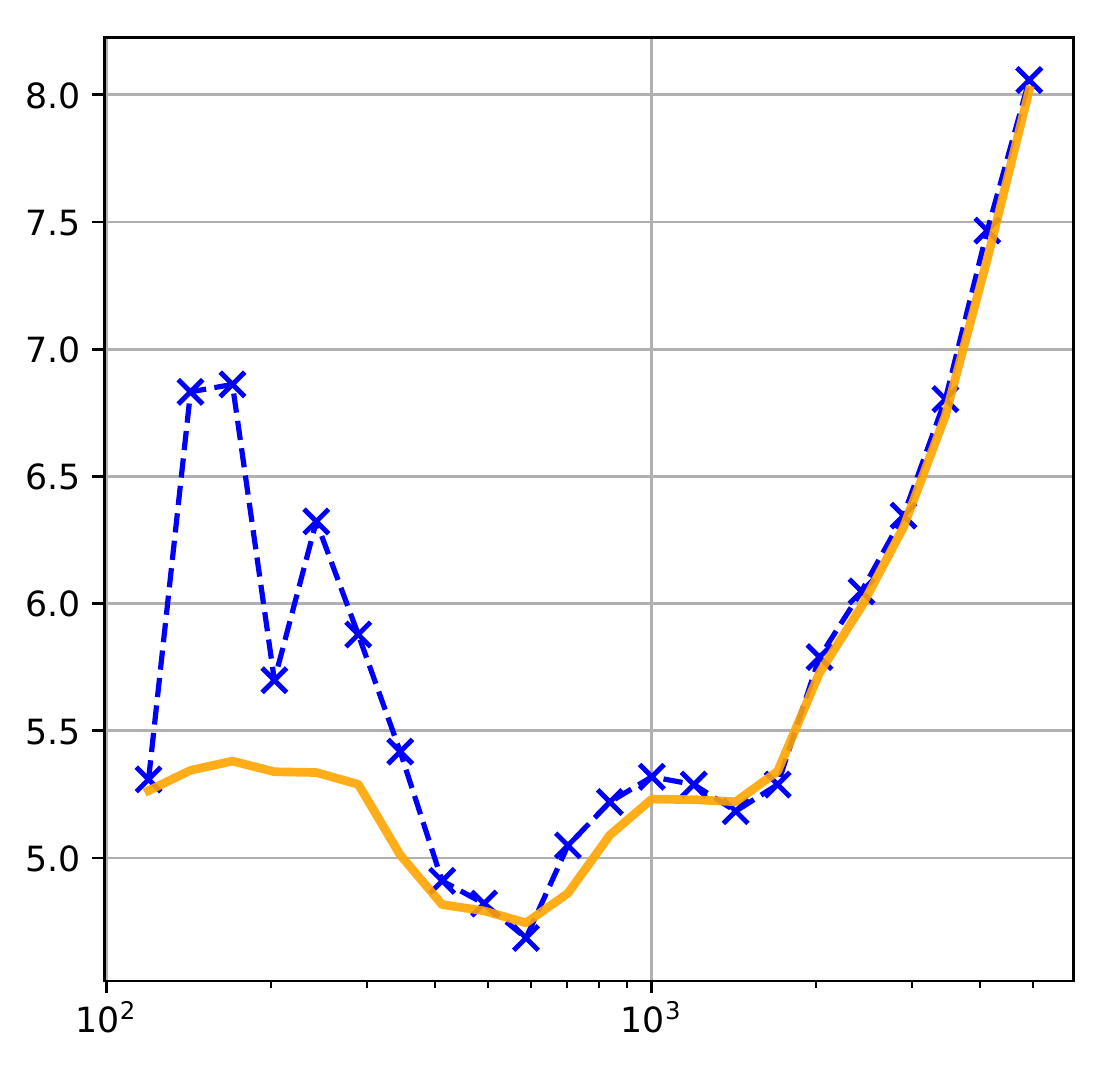}
            & \includegraphics[scale=0.34]{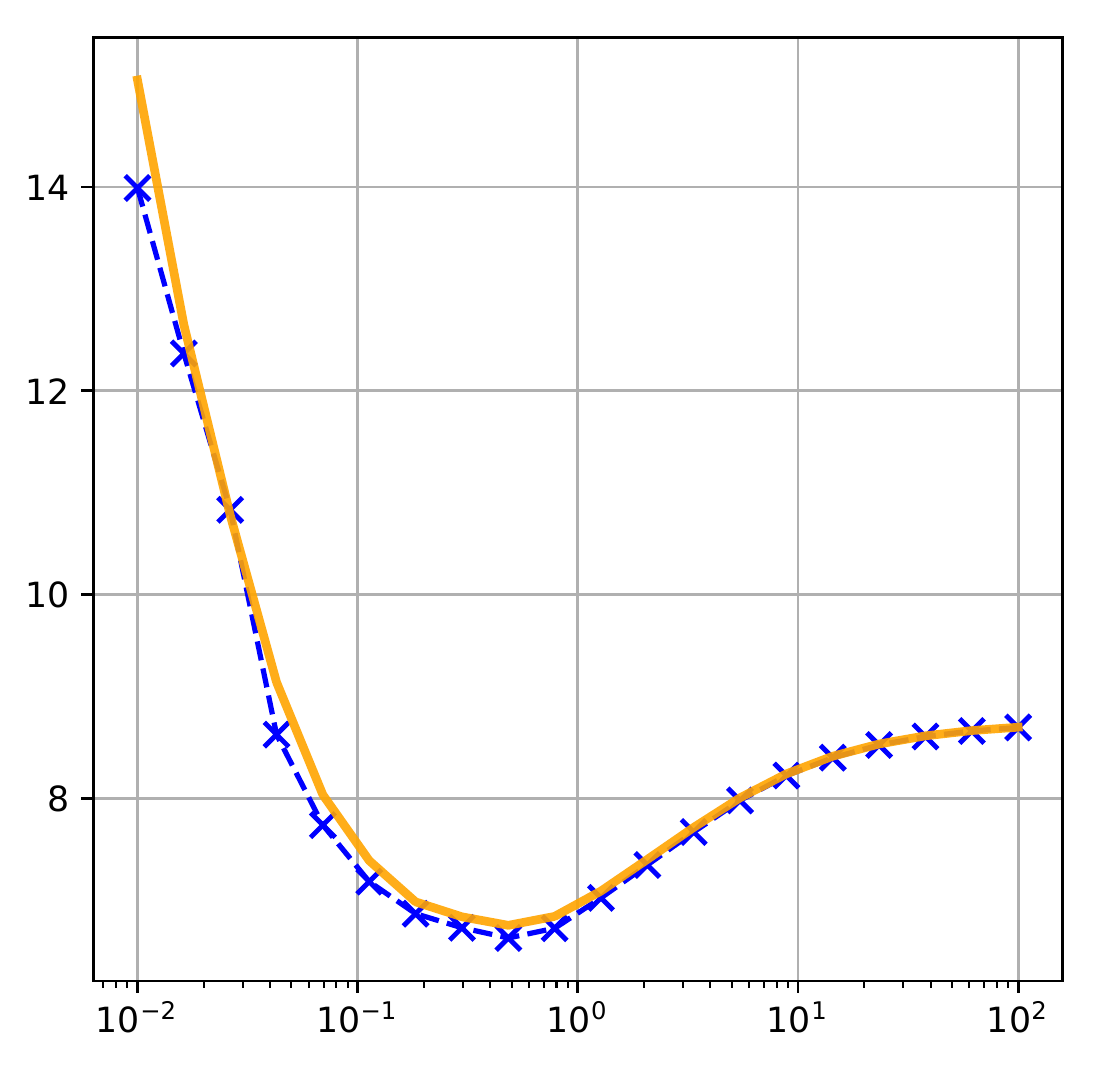}
            & \includegraphics[scale=0.34]{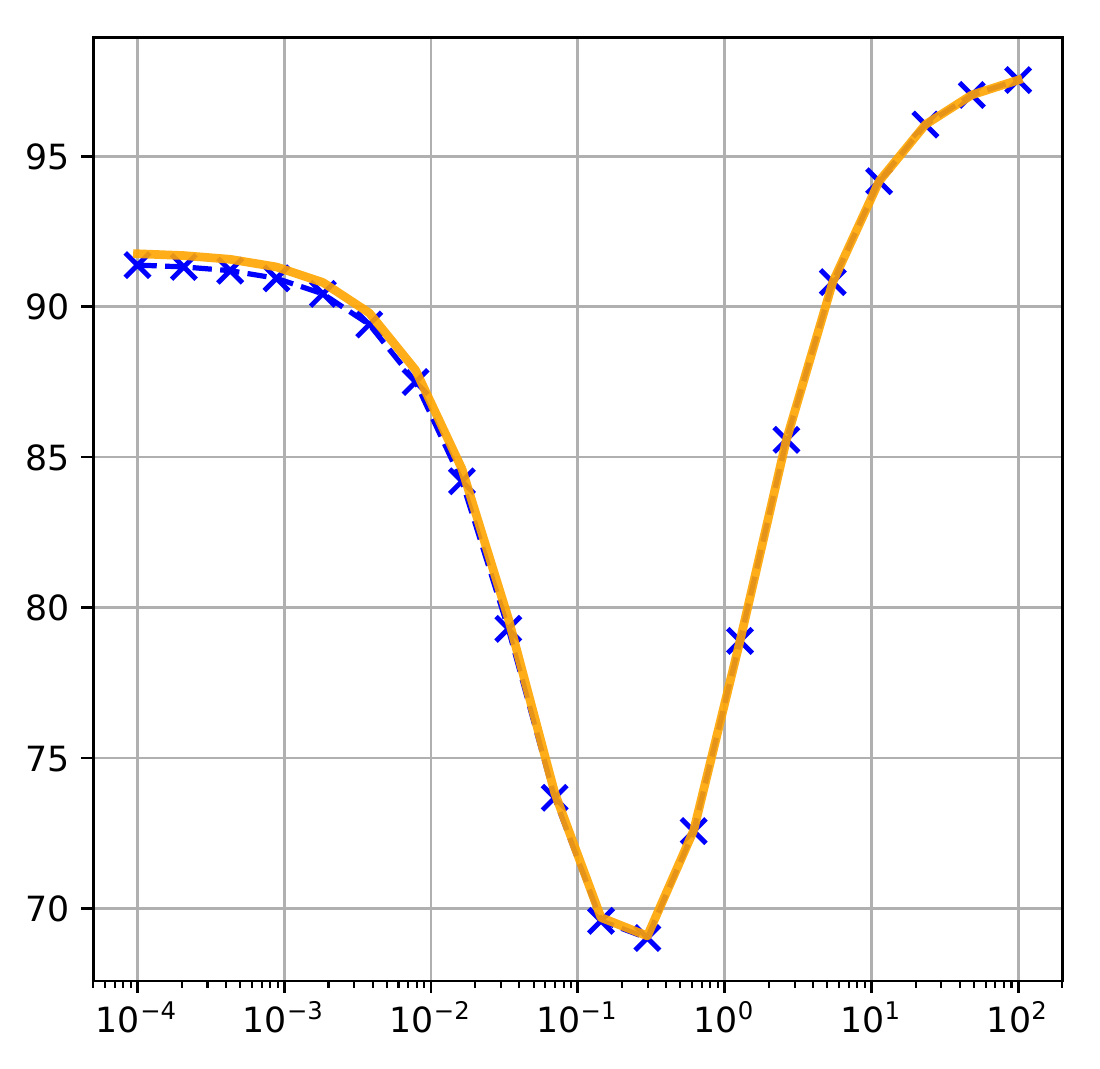}
            & \\
            & \multicolumn{1}{c}{{\fontsize{7}{6} $n=900, p=1800$}}
            & \multicolumn{1}{c}{{\fontsize{7}{6} $n=300, p=600$}}
            & \multicolumn{1}{c}{{\fontsize{7}{6} $n=300, p_1=p_2=20$}}
            & 
        \end{tabular}
        
        \caption{Risk estimates from ALO versus LOOCV. The $x$-axis is the
        tuning parameter value on $\log$-scale, the $y$-axis is the risk
        estimate. The comparison is based on SVM, fused LASSO, nuclear norm,
        group LASSO, $\ell_\infty$ norm, ridge regression on positive quadrant
        and positive semidefinite cone constrained matrix sensing. Different
        settings for the number of observations $n$ and the number of features
        $p$ are considered. For nuclear norm and positive semidefinite matrix
        cone constraints, $p_1, p_2$ are dimensions of a matrix.}\label{fig:risk-alo-loo1}
    \end{center}
\end{figure}

\subsubsection{Twisting the Model} \label{ssec:simdetailscorr}
In this section, we summarize the details of the simulations that are reported in Figure \ref{fig:risk-alo-loo2}. 
In our simulations, we use the setting where $n = 300$, $p = 600$, and the true model is sparse with $k =
60$ non-zeros. These non-zeros are i.i.d. $\mathcal{N}(0, 1)$.  

In the misspecification example, the elements of $\Xv$ are i.i.d. $\mathcal{N}(0, 1 / k)$. $\yv$ is generated according to
the following non-linear model:
\begin{equation*}
    y_j = f(\xv_j^\top \betav + \epsilon_j),
\end{equation*}
where $\bm{\epsilon} \sim \mathcal{N} (\mathbf{0}, 0.25 \Iv_{n})$, and  the function $f$ is given by:
\begin{equation*}
    f(x) = \begin{cases}
        \sqrt{x} & \text{ if } x \geq 0, \\
        -\sqrt{-x} & \text{ otherwise.}
    \end{cases}
\end{equation*}

In the heavy-tailed noise example, the elements of $\Xv$ are i.i.d. $\mathcal{N}(0, 1 / k)$. $\yv$ is generated according to
\begin{equation*}
\yv = \Xv\betav + \bm{\epsilon},
\end{equation*}
where the ``heavy-tailed'' noise $\epsilon_j$ is generated
according to a Student-$t$ distribution with three degrees of freedom, and
rescaled such that its variance is $\sigma^2=0.25$.

In the correlated design example, $\yv$ is generated according to
\begin{equation*}
\yv = \Xv\betav + \bm{\epsilon},
\end{equation*}
where $\bm{\epsilon} \sim \mathcal{N} (\mathbf{0}, 0.25 \mathbf{I})$, and the
``correlated design'' $\Xv$ is generated with each row $\xv_j$ being
sampled independently according to a multivariate normal distribution $\xv_j
\sim \mathcal{N}(0, \Cv / k)$, where $\Cv$ is the Toeplitz matrix, given by:
\begin{equation*}
    \Cv = \begin{pmatrix}
        \rho & \rho^2 & \ldots & \rho^p \\
        \rho^2 & \rho & \ldots & \rho^{p - 1} \\
        \vdots & \ldots & \ddots & \vdots \\
        \rho^p & \rho^{p - 1} & \ldots & \rho
    \end{pmatrix}.
\end{equation*}
$\rho$ is set to $0.8$ in our experiments. For all settings, we consider a
sequence of 25 tuning parameters from $3.16\times 10^{-3} \sim 3.16\times
10^{-2}$, equally spaced under log-scale.

All models were solved using the \texttt{glmnet}  package in
Matlab \cite{qian2013glmnet}. We identify the zero locations of $\estim{\betav}$ by checking
$|\beta_j| > 10^{-8}$.

\begin{figure}[!htbp]
    \begin{center}
        \setlength\tabcolsep{2pt}
        \renewcommand{\arraystretch}{0.3}
        \begin{tabular}{r|rrr}
            & \multicolumn{1}{c}{\small misspecification}
            & \multicolumn{1}{c}{\small heavy-tailed noise}
            & \multicolumn{1}{c}{\small correlated design} \\
            \hline
            \rotatebox{90}{\small \hspace{1.2cm} lasso risk} &
            \includegraphics[scale=0.34]{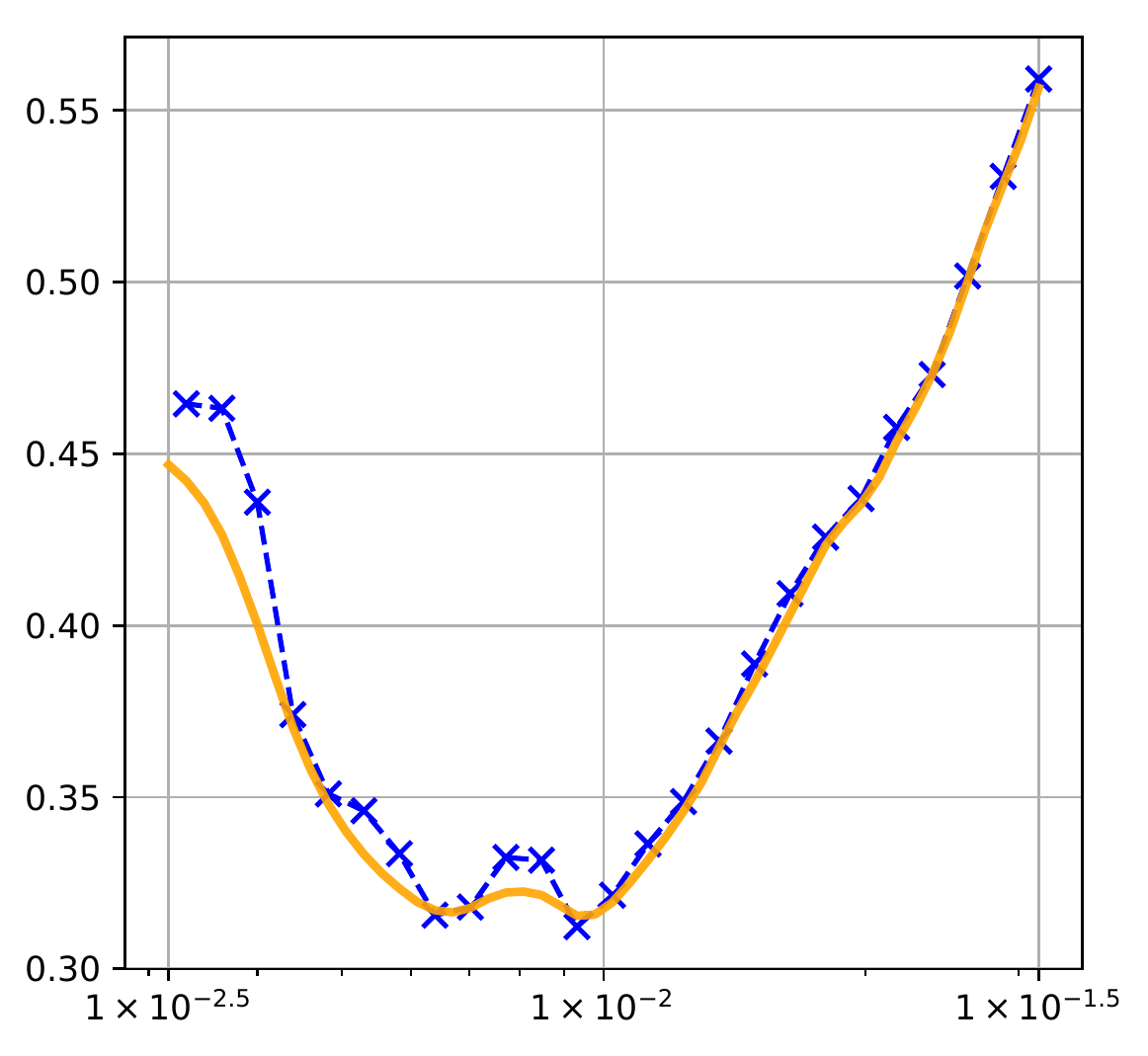} & 
            \includegraphics[scale=0.34]{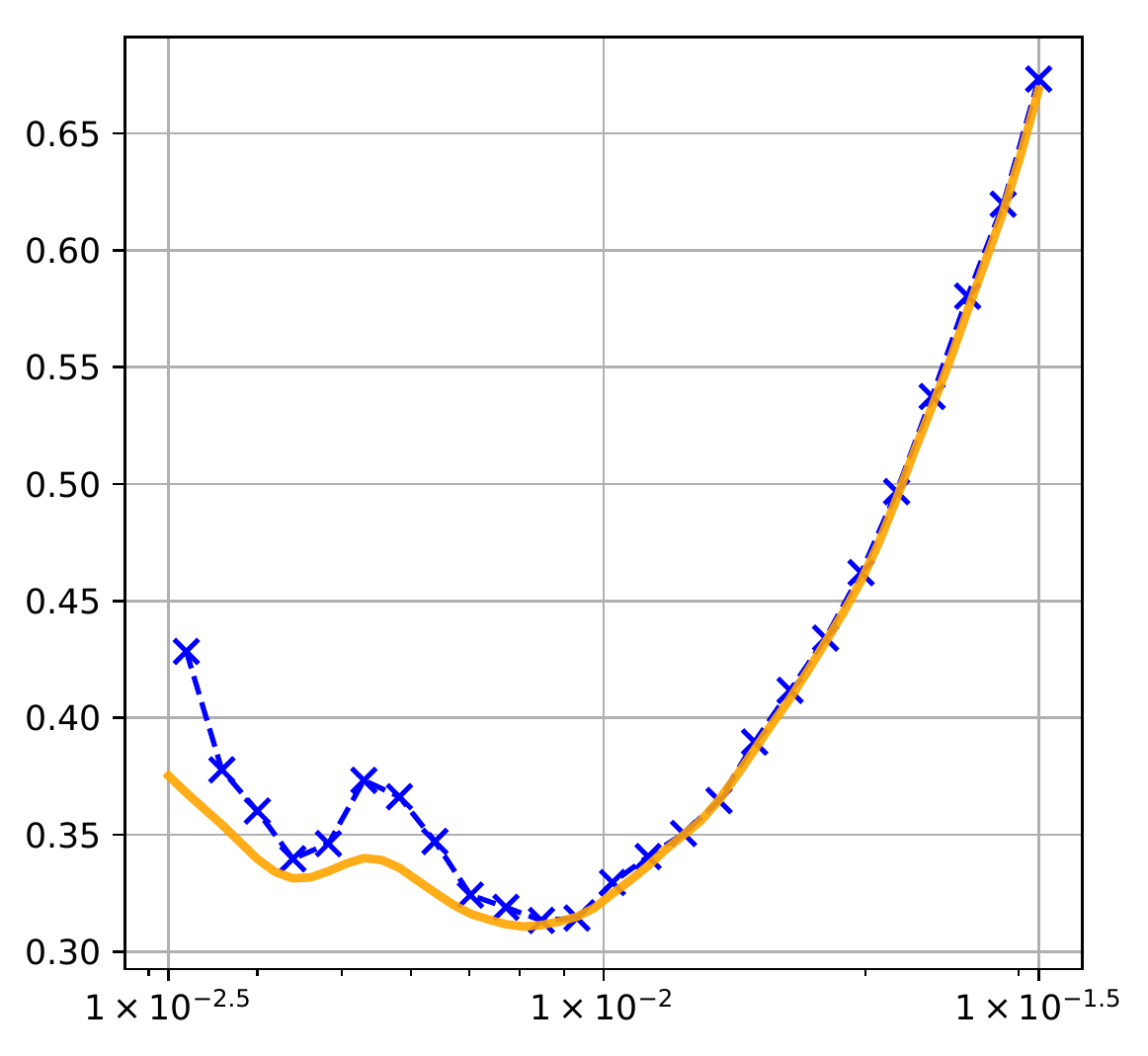} & 
            \includegraphics[scale=0.34]{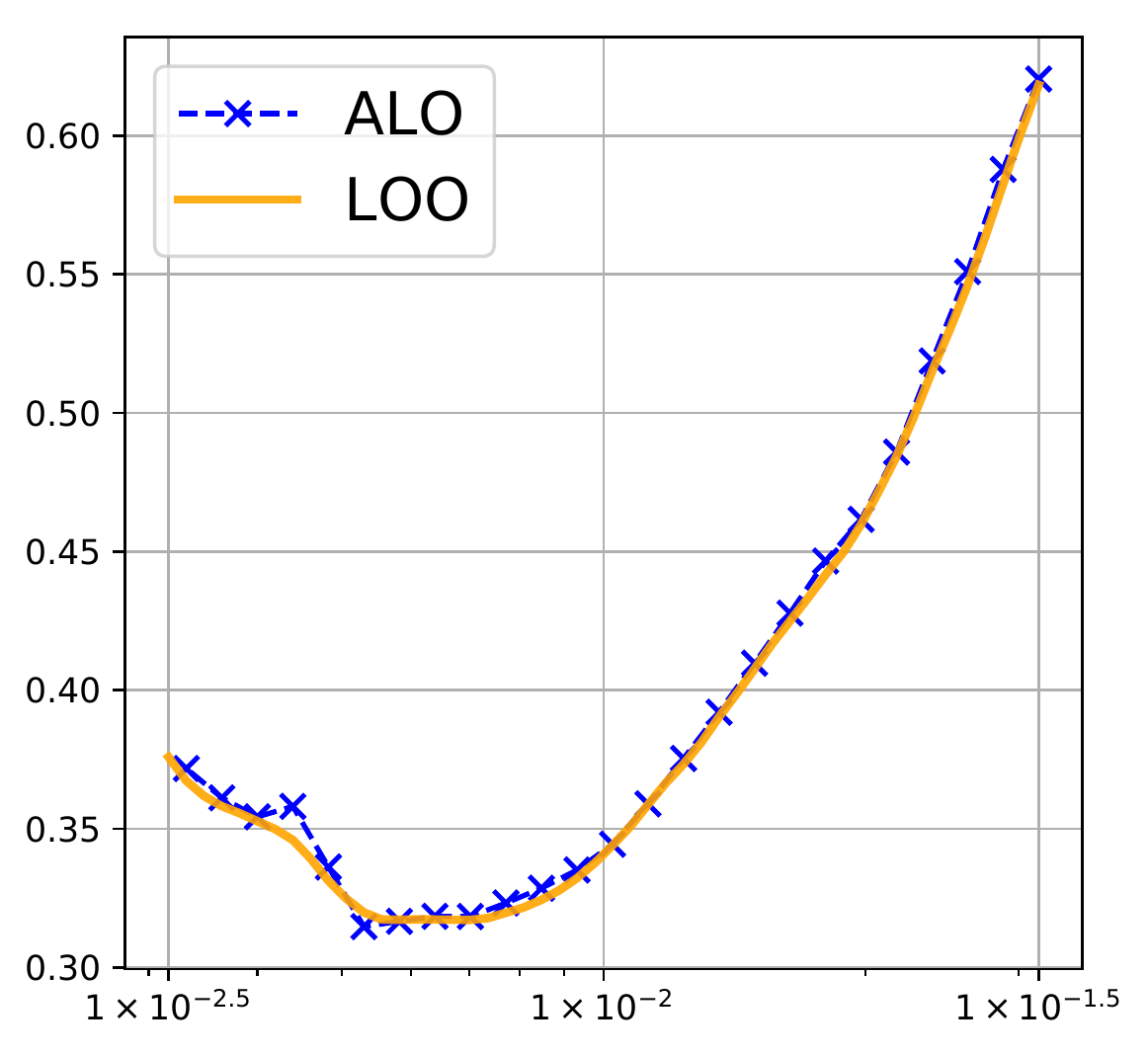}
    \end{tabular}
        \caption{Risk estimates from ALO versus LOOCV. The $(x, y)$-axes has the same meaning as Figure
            \ref{fig:risk-alo-loo1}. We consider the risk estimates of LASSO
            under model mis-specification, heavy-tailed noise and correlated design
            scenarios. We use $n=300$, $p=600$ and $k=30$ for all three where $k$ is
            the number of nonzeros in the true $\betav$.} \label{fig:risk-alo-loo2}
    \end{center}
\end{figure}

\subsubsection{IID Guassian Design with Intercept}
\label{sssec:experiments-iid-intercept}
In this section, we explain the details of the simulations whose results are presented in Figure \ref{fig:risk-alo-loo-intercept}. The details of the three models are listed below.

\paragraph{LASSO}
We generate the model using $\yv = \Xv\betav + \epsilonv$. For $\Xv$ we have $X_{j,k}
\overset{iid}{\sim} \frac{1}{\sqrt{n}}\mathcal{N}(0, 1)$; For $\betav$, we
randomly pick $k$ locations and sample them from \texttt{Uniform[-3, 3]}, with
the rest set to 0; $\epsilon_j \overset{iid}{\sim} 0.8 \mathcal{N}(0, 1)$.
Finally we use $n=400$, $p=200, 800$ and $k=100$.

\paragraph{SVM}
The data is generated based on the logistic regression model
$y_j \sim \texttt{Bernoulli} (p_j)$ with $\log\frac{p_j}{1 - p_j} =
\xv_j^\top\betav_0 + \epsilon_j$. Again $X_{j,k} \overset{iid}{\sim}
\frac{1}{\sqrt{n}}\mathcal{N}(0, 1)$, $\beta_j \overset{iid}{\sim}
\texttt{Uniform[-3, 3]}$ and $\epsilon_j \overset{iid}{\sim} 0.5 \mathcal{N}(0,
1)$. We choose $n=300$ and $p=150, 600$.

\paragraph{Ridge regression on postive quadrant}
Similar to the LASSO case, we generate the model using $\yv = \Xv\betav + \epsilonv$.
$Xv$ is generated in the same way. For $\betav$, we have $\beta_j
\overset{iid}{\sim} \texttt{Uniform[-1, 3]}$; $\epsilon_j \overset{iid}{\sim}
0.5 \mathcal{N}(0, 1)$. Finally we use $n=300$ and $p=150, 600$.

\begin{figure}[!htbp]
    \begin{center}
        \setlength\tabcolsep{2pt}
        \renewcommand{\arraystretch}{0.3}
        \begin{tabular}{r|rrr}
            & \multicolumn{1}{c}{\small lasso}
            & \multicolumn{1}{c}{\small svm}
            & \multicolumn{1}{c}{\small positive ridge} \\
            \hline
            \rotatebox{90}{\small \hspace{1.6cm} $n > p$}
            & \includegraphics[scale=0.34]{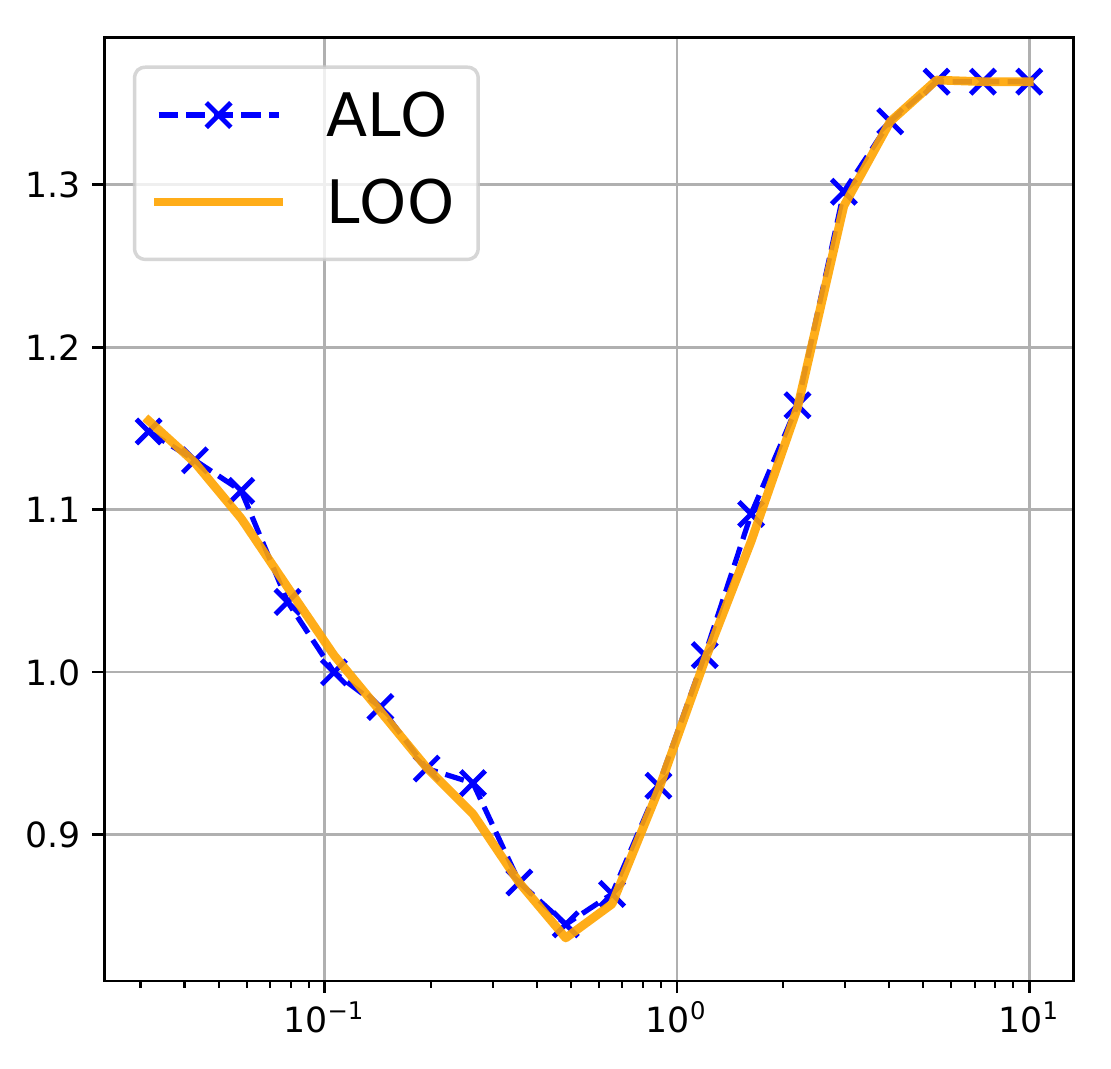}
            & \includegraphics[scale=0.34]{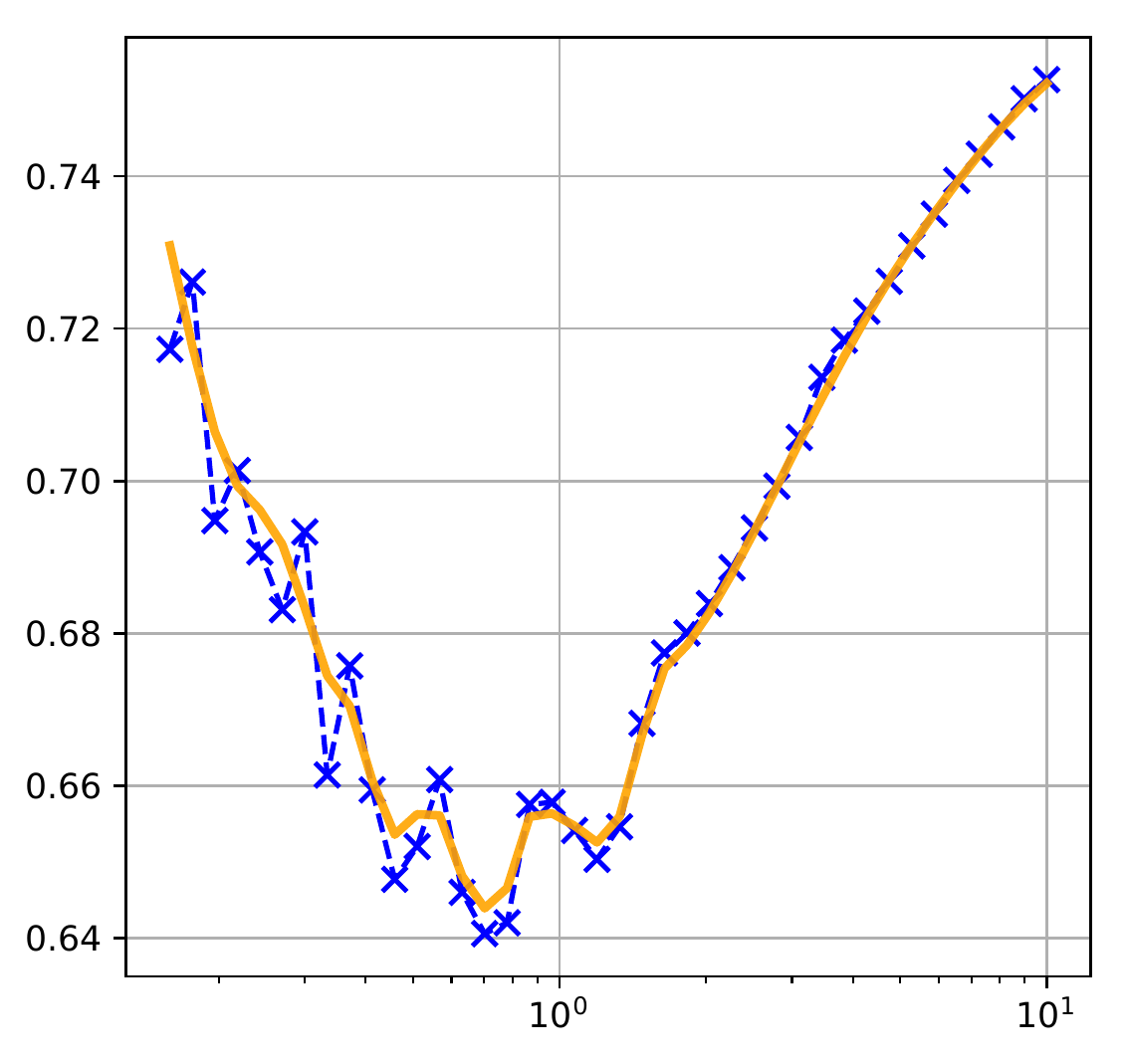}
            & \includegraphics[scale=0.34]{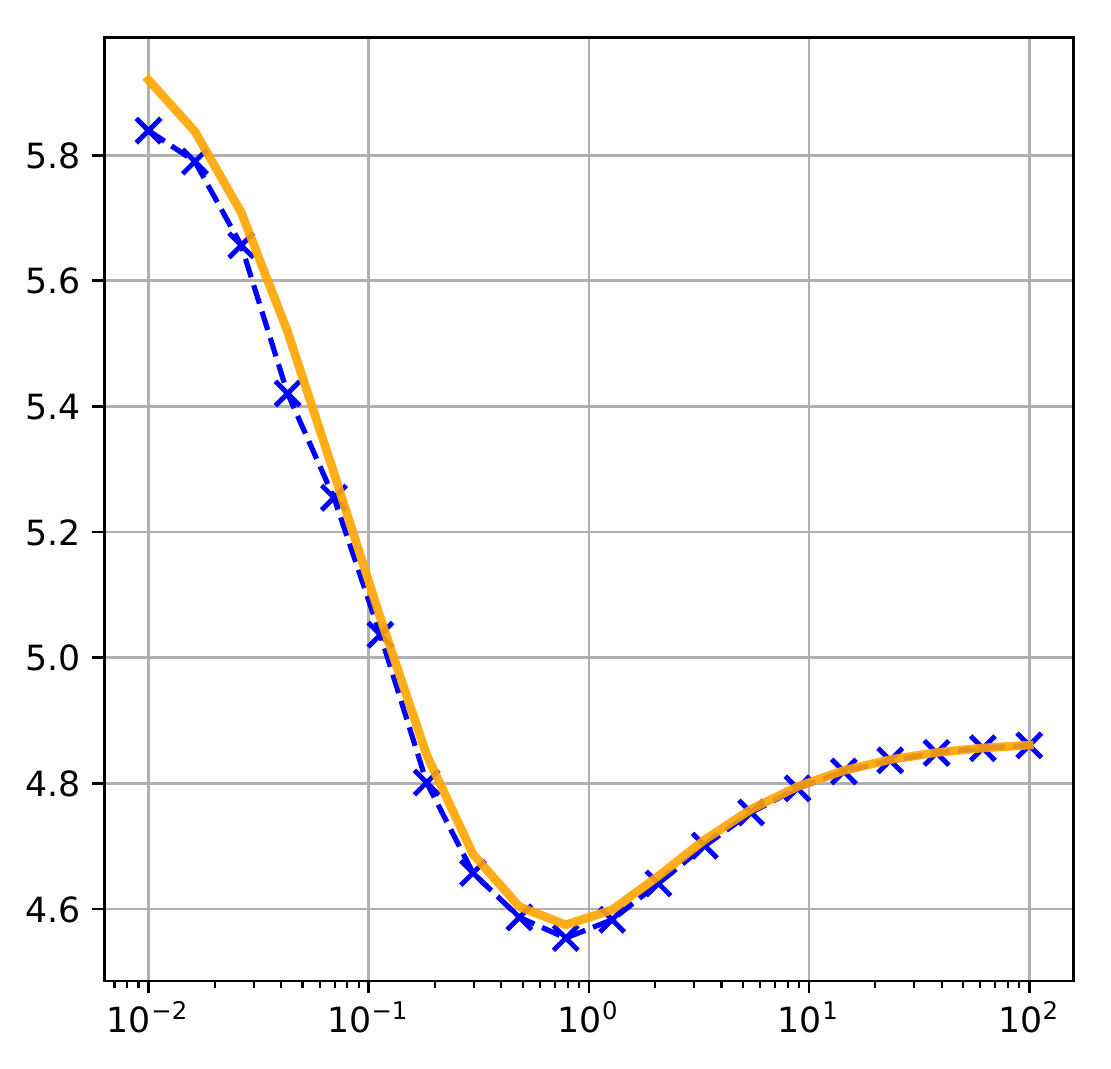} \\
            & \multicolumn{1}{c}{{\fontsize{7}{6} $n=400, p=200$}}
            & \multicolumn{1}{c}{{\fontsize{7}{6} $n=300, p=150$}}
            & \multicolumn{1}{c}{{\fontsize{7}{6} $n=300, p=150$}} \\
            \rotatebox{90}{\small \hspace{1.6cm} $n < p$}
            & \includegraphics[scale=0.34]{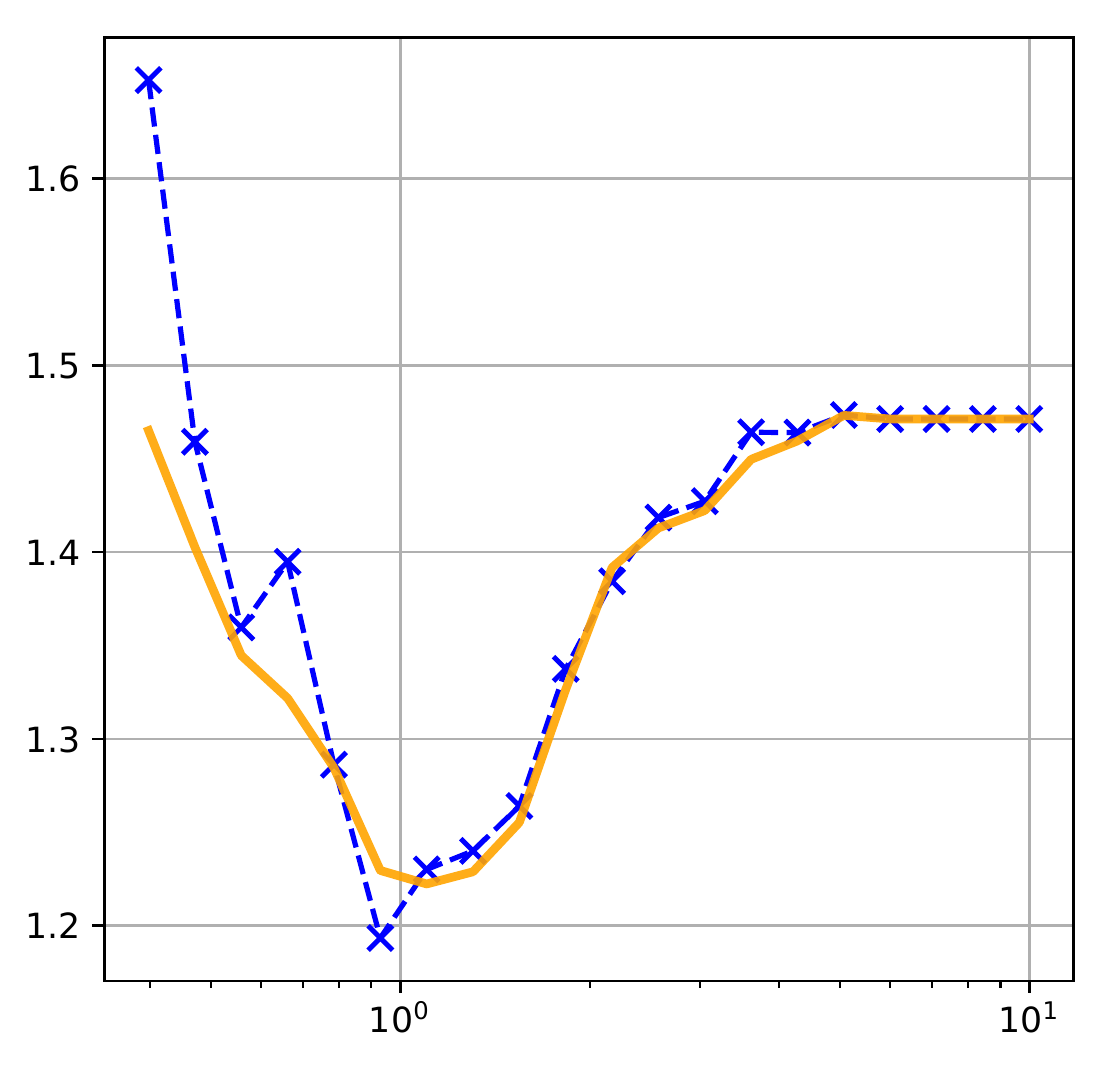}
            & \includegraphics[scale=0.34]{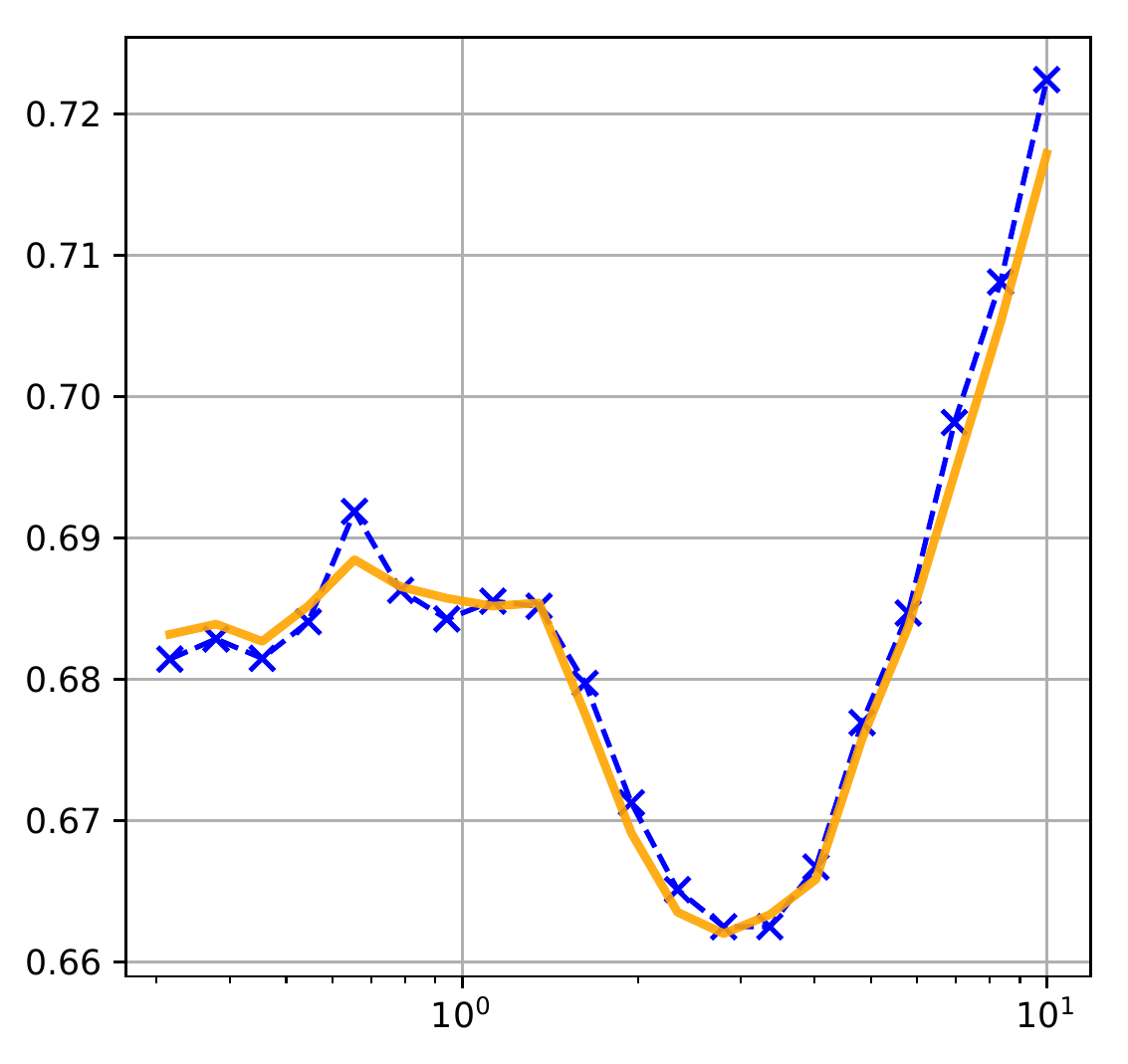}
            & \includegraphics[scale=0.34]{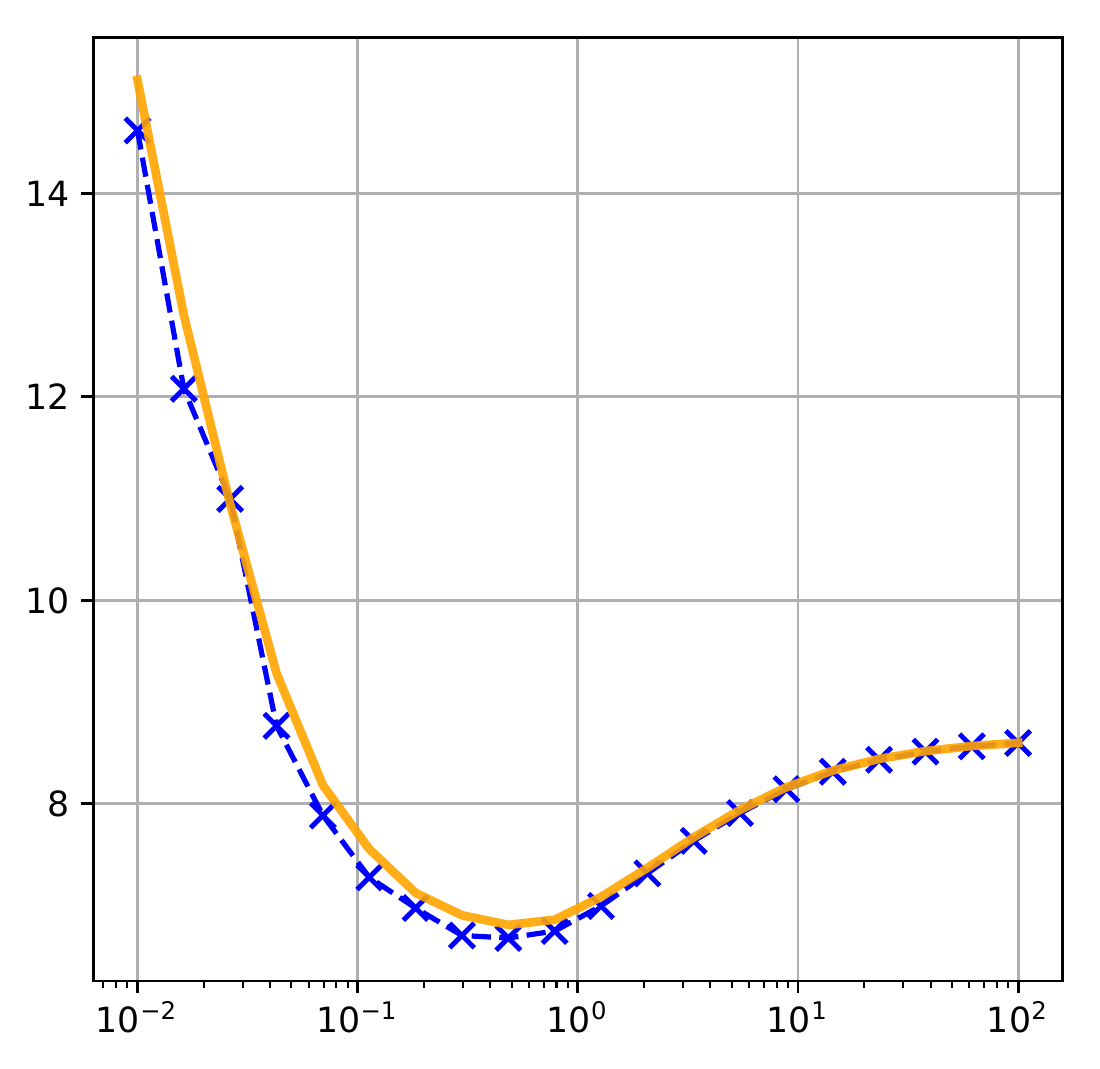} \\
            & \multicolumn{1}{c}{{\fontsize{7}{6} $n=400, p=800$}}
            & \multicolumn{1}{c}{{\fontsize{7}{6} $n=300, p=600$}}
            & \multicolumn{1}{c}{{\fontsize{7}{6} $n=300, p=600$}}
        \end{tabular}
        \caption{Risk estimates from ALO versus LOOCV on models involving
            intercepts. The $(x, y)$-axes are interpreted in the same way as
            Figure \ref{fig:risk-alo-loo1}. The comparison is based on LASSO, SVM and Ridge
            regression with positive quadrant constraint, corresponding to
            nonsmooth regularizer, nonsmooth loss and contrained problem
            respectively.}\label{fig:risk-alo-loo-intercept}
    \end{center}
\end{figure}

\subsection{Timing comparison between ALO and LOOCV}
\label{ssec:experiment-alo-timing}
Our next experiment compares the computational complexity of ALO with that of
LOOCV. In Table \ref{tab:timing}, we provide the timing of LASSO for different
values of $n$ and $p$. The time required by ALO, which involves a single fit
and a matrix inversion (in the
construction of $\Hv$ matrix), is in all experiments no more than twice that of a single fit. As expected, averaged time for LOOCV is close to $n$ times the time required
for a single fit. 

\subsubsection{Details of the Simulation}
For comparing the timing of ALO with that of LOOCV, we consider the LASSO
problem with correlated design similar to the one we introduced in Section
\ref{ssec:simdetailscorr}. Specifically, each row of the design matrix has a
Toeplitz covariance matrix with $\rho = 0.8$. The true coefficient vector
$\betav$ has $\frac{\min(n, p)}{2}$ nonzero components, with each nonzero
component of $\betav$ being selected independently from $\pm 1$ with
probability $0.5$. The noise $\epsilon \sim \mathcal{N}(0, 0.5\Iv_n)$. For each
pair of $(n, p)$, we choose a sequence of 50 tuning parameters ranging from
$\lambda_0$ to $10^{-2.5}\lambda_0$, where $\lambda_0 = \|\Xv^\top
\yv\|_{\infty}$. Note that for this choice of $\lambda$ all the regression
coefficients are equal to zero.

The timing of one single fit on the full dataset, the ALO risk estimates and
the LOOCV risk estimates are reported in Table \ref{tab:timing}. To obtain the timing of a single fit we run the corresponding function
of glmnet along the entire tuning parameter path and record the total time
consumed. This process is then repeated for 10 random seeds to obtain the
average timing. Every time an estimate is obtained we use our formula to obtain
ALO. Hence, the time reported for ALO in Table \ref{tab:timing} is again
obtained from an average of $10$ Monte Carlo samples. To obtain the computation
time of LOOCV, we only use $5$ random seeds.

\begin{table}[!htbp]
    \begin{center}
        \caption{Timing (in \textit{sec}) of one single fit, ALO and LOOCV. In
        the upper and lower tables, we fix $n=800$ and $p=800$ respectively.}
        \label{tab:timing}
        \begin{tabular}{l|llll}
            $p$ & 200 & 400 & 800 & 1600 \\
            \hline
            single fit & $0.035 \pm 0.001$ & $0.13 \pm 0.003$ & $0.56
            \pm 0.02$ & $0.60 \pm 0.01$ \\
            ALO & $0.060 \pm 0.001$ & $0.21 \pm 0.003$ & $0.77 \pm
            0.02$ & $0.89 \pm 0.01$ \\
            LOOCV & $27.52 \pm 0.03$ & $107.4 \pm
            0.5$ & $437.9 \pm 2.9$ & $479 \pm 2$ \\
            \hline
        \end{tabular}
        \begin{tabular}{l|llll}
            $n$ & 200 & 400 & 800 & 1600 \\
            \hline
            single fit & $0.055 \pm 0.002$ & $0.19 \pm 0.006$ & $0.56
            \pm 0.02$ & $0.76 \pm 0.02$ \\
            ALO & $0.065 \pm 0.001$ & $0.24 \pm 0.001$ & $0.77 \pm
            0.02$ & $1.20 \pm 0.01$ \\
            LOOCV & $11.44 \pm 0.049$ & $74.7 \pm
            0.5$ & $437.9 \pm 2.9$ & $1249 \pm 3$
        \end{tabular}
    \end{center}
\end{table}

\subsection{Evaluating the Accuracy of ALO on Real-World Data}
\label{ssec:real-world-data}
In this section, we apply our ALO methods to three real-world datasets: Gisette
digit recognition \cite{guyon2005result}, the tumor colon tissues gene
expression \cite{alon1999broad} and the South Africa heart disease data
\cite{rossouw1983coronary, EOSL:chapter4}. All
the three datasets have binary response, so we consider classification
algorithms. The information of the three datasets is listed in Table
\ref{tab:real-data-meta-info} below. The column of number of effective features
records the number of features after data preprocessing, including removing
duplicates and missing columns.
\begin{table}[!htbp]
    \begin{center}
        \caption{Information of the three datasets.}
        \label{tab:real-data-meta-info}
        \begin{tabular}{l|p{1.6cm}p{1.6cm}p{1.6cm}l}
            dataset & \# samples & \# features & \# effective features & model used \\
            \hline
            gisette & 6000 & 5000 & 4955 & SVM \\
            tumor colon & 62 & 2000 & 1909 & logistic + LASSO \\
            heart disease & 462 & 9 & 9 & logistic + LASSO
        \end{tabular}
    \end{center}
\end{table}

For gisette, since $n=6000$ is too large for LOOCV, we randomly subsample 1000
observations and apply linear SVM on it. For the tumor colon tissues and South
Africa heart disease dataset, we apply logistic regression with LASSO penalty.
The results are shown in Figure \ref{fig:real-data}. The accuracy of ALO is
verified on gisette and the heart disease dataset. However, the behavior of ALO
is more complicated for the tumor colon tissues dataset. First ALO gives very
close estimates to LOOCV for relatively large tuning values, but deviates from
LOOCV risk estimates and bends upward after $\lambda$ decreases to a certain
value. Second, we note that the optimal tuning is still correctly captured by
ALO.
\begin{figure}[!htbp]
    \begin{center}
        \setlength\tabcolsep{2pt}
        \renewcommand{\arraystretch}{0.3}
        \begin{tabular}{r|rrr}
            & \multicolumn{1}{c}{\small gisette}
            & \multicolumn{1}{c}{\small heart disease}
            & \multicolumn{1}{c}{\small colon tumor} \\
            \hline
            \rotatebox{90}{\small \hspace{1.1cm} lasso risk} &
            \includegraphics[scale=0.4]{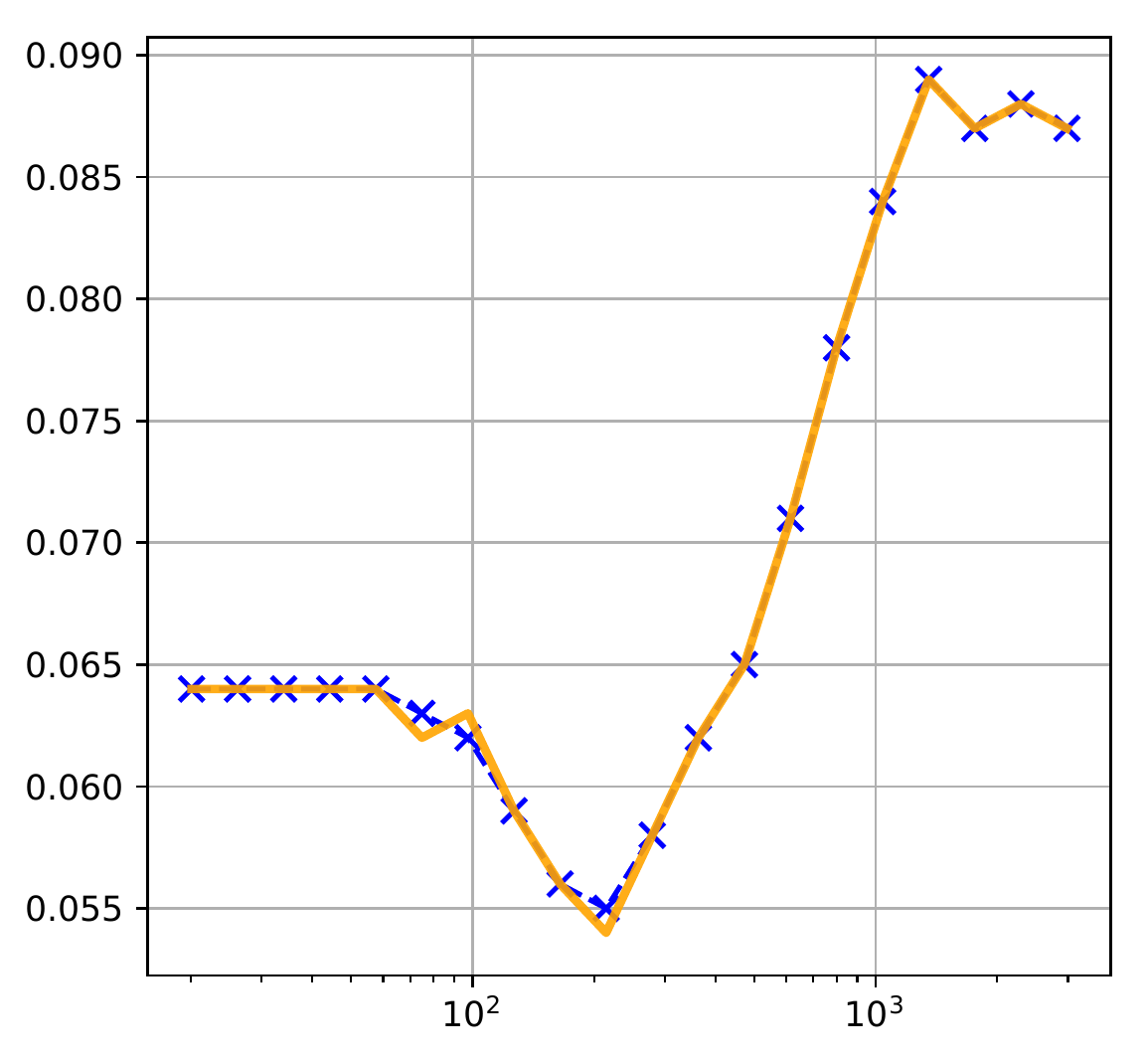} & 
            \includegraphics[scale=0.4]{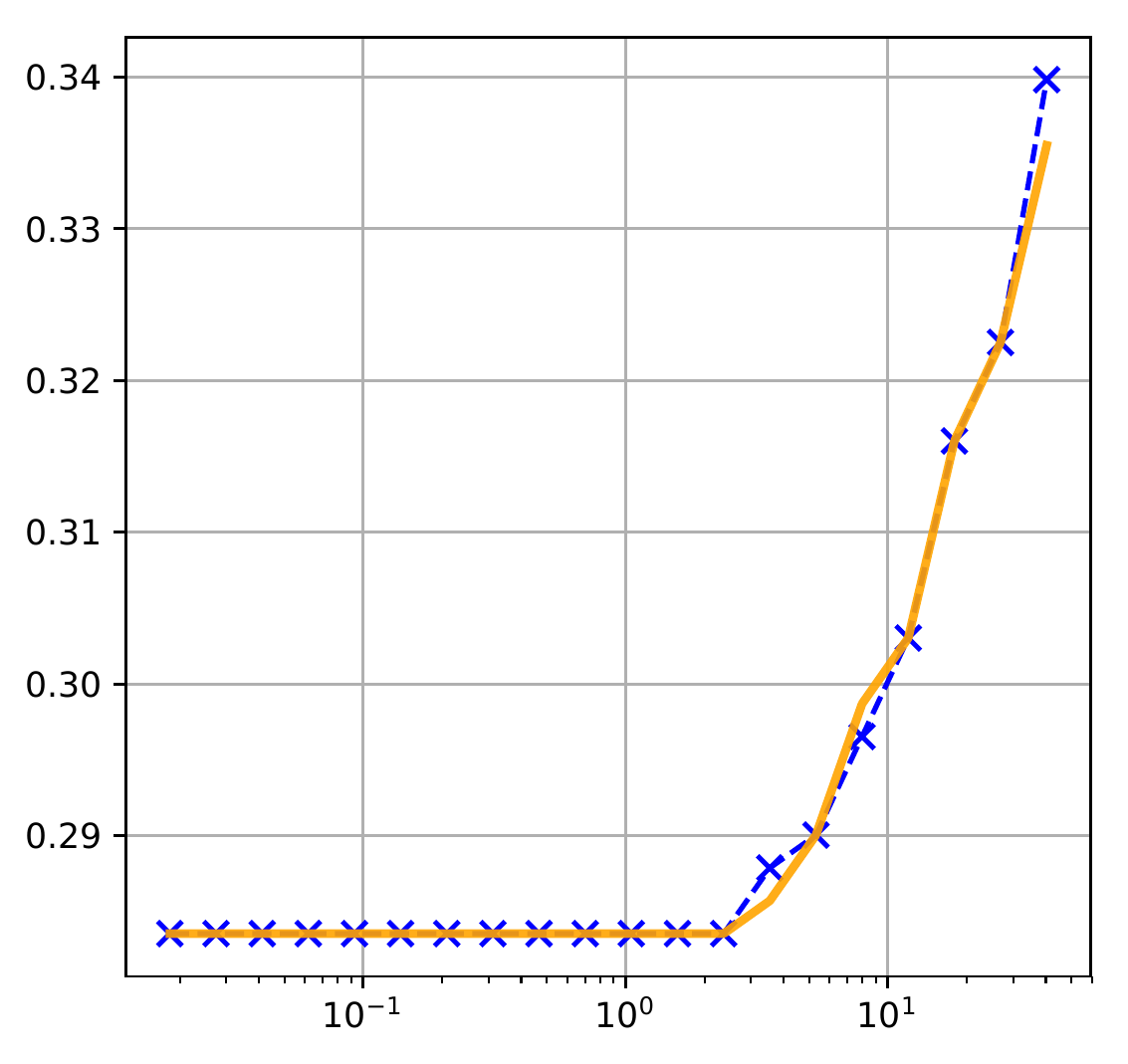} & 
            \includegraphics[scale=0.4]{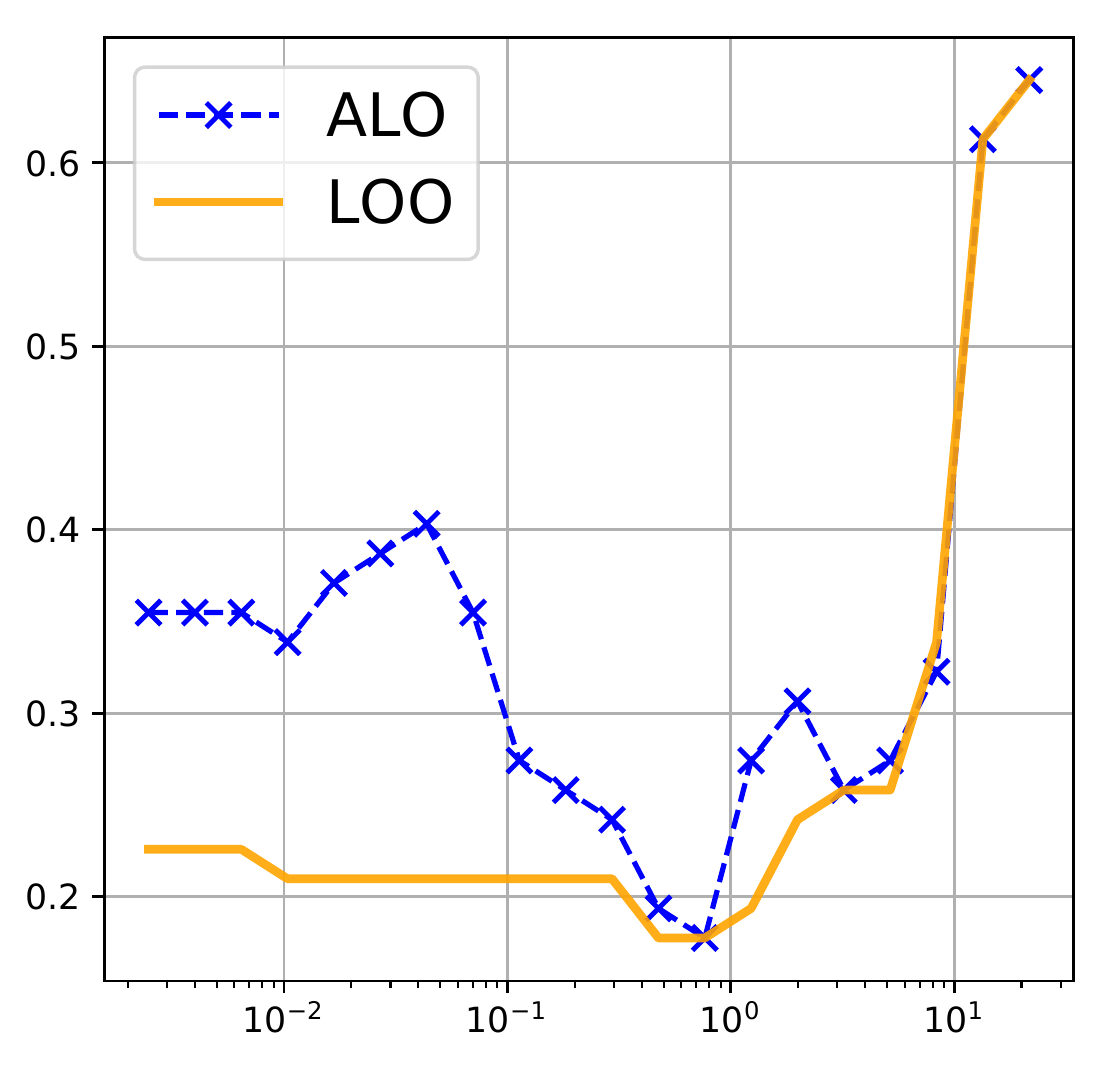} \\
        \end{tabular}
        \caption{Risk estimates of from ALO versus LOOCV for the three
            datasets: gisette, South Africa coronary heart disease and colon
            tumor gene expression. The $x$-axis is the tuning parameter value
            $\lambda$ on $\log$-scale, the $y$-axis is the risk estimates under
            0-1 loss.} \label{fig:real-data}
    \end{center}
\end{figure}

There are a few factors which may affect the performance of ALO. First, as
implied by the theoretical guarantee on smooth models, the closeness between
ALO and LOOCV is a high-dimensional phenomenon, which takes place for
relatively large $n$ and $p$. From our simulation in Section \ref{sec:experiments} and the real-data
examples in this section, we can see that when $\frac{n}{p}$ is not much
smaller than 1 (compared to the $\frac{n}{p}$-ratio in the colon tissue
dataset), a few hundreds of observation and features are enough to guarantee
the accuracy of ALO risk estimates. Also note that the deviation of ALO estimates
tends to happen when the tuning $\lambda$ becomes smaller than a certain value,
typically in the case of $n < p$. For most nonsmooth regularizers, small tuning values induce dense solutions. In most
high dimensional datasets, these dense solutions are often not favorable. Furthermore, from our experiments, this deviation mostly happens after correctly
capturing the optimal tuning values. We should again emphasize that the deviations decrease as $n$ and $p$ grow. 

\section{Discussion} \label{sec:discussion}

\paragraph{Determining the active set}
For most of the nonsmooth models we need to identify certain
set of indices (we call it active set in the rest of this section). They either determine the direction along which the objective
function changes smoothly (such as the set $V, S$ in \eqref{eq:indexset} and
the set $A$ in \eqref{eq:nonsmooth-reg:2}), or characterize the face on the
dual norm ball where the optimum locates (such as the set $E$ in Section
\ref{ssec:application-l_inf}, \ref{ssec:application-group-lasso},
\ref{ssec:application-slope}). 

The identification of the active set can potentially depend on the algorithms used to
optimize the objective function. For example, if we use the coordinate
descent or proximal gradient descent algorithm to solve LASSO, then sparsity is
automatically imposed. In this case, one may just pick the nonzero locations
directly. However for some other models (as we will see in the following
example for $\ell_\infty$ norm penalty), the active set depends on the optimzer
in an indirect way and cannot be explicitly identified straightforwardly. A
generic solution is to set a threshold value to extract the active set. However
we observe that this threshold may slowly vary for different values of tuning
parameter. Ideally, one would like to employ algorithms, such as the proximal
gradient descent in the case of LASSO, that can return the active set and do
not leave the decision of the threshold to the user.

Below we introduce an idea which avoids this thresholding step by employing a
proper optimization algorithm to solve the dual problem and construct the
active set explicitly. We use the $\ell_\infty$-minimization problem discussed
in Section \ref{ssec:application-l_inf} as an example. Similar idea may be used
for some other problems too. As we discussed in Section
\ref{ssec:application-l_inf}, we need to identify the set of indices $E=\{j:
\Xv_j^\top \estim{\uv} = 0\}$, where $\estim{\uv}$ is the dual optimizer
\begin{equation*}
    \estim{\uv} = \argmin_{\uv} \| \yv - \uv\|_2^2,
    \quad
    \text{subject to } \| \Xv^\top \uv\|_1 \leq \lambda.
\end{equation*}

According to the primal dual correspondence $\yv - \Xv\estim{\betav} =
\estim{\uv}$. After obtaining the primal optimizer $\estim{\betav}$, we may
check the value of $\Xv_j^\top(\yv - \Xv\estim{\betav})$ for each $1\leq j \leq
p$ and select the ones that are exactly equal to $0$. However, due to the
non-exactness of the solution we do not expect to observe any exact $0$.
Nevertheless one may directly solve the dual problem in an appropriate way so
that exact zeros can be obtained. Let $\zv = \Xv^\top\uv$, the dual problem can
be translated to
\begin{equation*}
    \estim{\uv} = \argmin_{\uv} \| \yv - \uv\|_2^2,
    \quad
    \text{subject to } \| \zv \|_1 \leq \lambda
    \text{ and } \Xv^\top \uv = \zv.
\end{equation*}

Note that the optimum $\estim{\zv} = \Xv^\top\estim{\uv} = \Xv^\top(\yv -
\Xv\estim{\betav})$. Thus we may identify the set $E$ directly from
$\estim{\zv}$. To make this possible, we need to adopt an optimization
algorithm which exploits the $\ell_1$ constraints on $\zv$ so that exact zeros can
be obtained. A natural choice is the ADMM algorithm \cite{boyd2011distributed},
which iterates in the following way
\begin{align*}
    \uv^{t+1} =& \big( \Iv + \rho \Xv\Xv^\top \big)^{-1} \big(\yv + \rho\Xv\zv^t -
    \Xv\muv^t\big) \\
    \zv^{t+1} =& \projv_{\{\zv:\|\zv\|_1 \leq \lambda\}} \Big( \Xv^\top
    \uv^{t+1} + \frac{\muv^t}{\rho} \Big) \\
    \muv^{t+1} =& \muv^t + \rho \big( \Xv^\top\uv^{t+1} - \zv^{t+1}\big)
\end{align*}
where $\rho > 0$ is a stepsize parameter manually picked. $\muv^t$ is the
Lagrange multiplier.

The projection update on $\zv^{t+1}$ automatically imposes sparsity. Once
the algorithm converges with certain precision, the set of indices can be picked
easily by identifying the zero locations in $\estim{\zv}$. We would like to emphasize that this trick occurs at the optimization stage,
and does not change our ALO algorithm itself. Also it requires the
availablility of fast algorithms of projection to certain convex set
($\ell_1$-norm ball in this example).

\paragraph{ALO risk estimation for small tuning}
From the simulations in Section \ref{sec:experiments}, we observe that when $n
< p$, as the value of the tuning parameter $\lambda$ goes below a certain
threshold, for some of the models including fused LASSO and $\ell_\infty$ norm
minimization, ALO risk estimates skews upward against the LOOCV risk estimates.

Recall we need to construct a $\Hv$ matrix in the ALO formula and for all these
models, $\Hv = \Wv(\Wv^\top \Wv)^{-1} \Wv^\top$ for some matrix $\Wv \in
\mathbb{R}^{n \times k}$ where $k$ is determined by the face on the dual norm
ball at which the $\Xv^\top \estim{\uv}$ locates. It is obvious that $k \leq
p$. Thus when $n > p$, $\Wv$ has full column-rank and $H_{ii}$ are bounded away
from 1. But in the case of $n < p$, as one decreases the value of $\lambda$,
denser and denser solutions are produced. When $k$ gets close to $n$,
$H_{ii}$ will be closer and closer to 1, which in turn leads to large values of
ALO estimates. However, we should emphasize on two points: (i) in all these
cases, the optimal tunings are above the bad regions and are accurately
captured by ALO. (ii) As the problem size increases this issue alleviates.
Nevertheless, an interesting direction for future research is to find new
modifications for ALO that are capable of approximating LOOCV more accurately
even when $\lambda$ is small and $n<p$ are not very large.

\paragraph{Summary}
The low bias of the leave-one-out cross validation (LOOCV) makes it one of the
most appealing risk estimation techniques in high-dimensional settings, where
the number of predictors is comparable with the number of observations.
However, the high computational complexity of this method poses a major
obstacle in most real-world applications. In this paper, we proposed three
different methods for approximating LOOCV. These approaches are based on
primal, dual and the proximal formulation of learning problems. Different
approaches show their adavantages in different problems. Our
approximations inherit desirable properties of LOOCV, while dramatically reduce
its computational complexity.

We proved the equivalence of these methods when the loss function and the
regularizer are twice differentiable. This equivalence enabled us to prove the
accuracy of our approximation for large high-dimensional datasets. We also
showed how our approximation schemes can be used for non-differentiable losses
and regularizers. We use our approaches to obtain a risk estimate for several
popular non-differentiable learning problems. Our empirical results prove the
excellent performance of our approximation techniques.

\section{Proofs of our main results} \label{sec:proof}

\subsection{Proofs of Theorems \ref{thm:primal-dual-equivalence},
\ref{thm:primal-dual-equivalence-2} and Lemma \ref{prox:equivalence}}
\label{append:sec:primal-dual-equiv}

In this section, we prove the equivalence between the primal and dual methods in the case
where the loss and regularizer are twice differentiable.
Let $\ell$, $\ell^*$, $R$ and $R^*$ be twice differentiable. The following lemma
plays a key role in our analysis:
\begin{lemma}\label{lemma:hessianprimaldual}
    Let $f$ be a proper closed convex function, such that both $f$ and $f^*$ are twice differentiable.
    Then, we have for any $\xv$ in the domain of $f$ and any $\uv$ in the domain of $f^*$:
    \begin{align*}
        \nabla^2 f^*(\nabla f(\xv)) =& [\nabla^2 f(\xv)]^{-1}, \nonumber \\
        \nabla^2 f(\nabla f^*(\uv)) =& [\nabla^2 f^*(\uv)]^{-1}. 
    \end{align*}
\end{lemma}
\begin{proof}
    This lemma is a known result in convex optimization. However, since the
    proof is short and for the sake of completeness we include the proof here.
    For $f$ a proper closed convex function, we have by Theorem 23.5 of
    \cite{rockafellar1970convex} that for all $\xv, \xv^*$:
    \begin{equation*}
        \xv^* \in \partial f(\xv) \Rightarrow \xv \in \partial f^*(\xv^*).
    \end{equation*}
    
    In particular, if $f$ and $f^*$ are differentiable, we obtain:
    \begin{equation*}
        \xv = \nabla f^*(\nabla f(\xv)).
    \end{equation*}
    
    Taking derivative in $\xv$ once more, we obtain that:
    \begin{equation*}
        \Iv = [\nabla^2 f^*( \nabla f(\xv))] [\nabla^2 f(\xv)],
    \end{equation*}
    
    which immediately gives:
    \begin{equation*}
        \nabla^2 f^*(\nabla f(\xv)) = [\nabla^2 f(\xv)]^{-1}.
    \end{equation*}

    The proof of the second part is immediate by applying the existing
    result to $f^*$.
\end{proof}

\begin{proof}[Proof of Theorem \ref{thm:primal-dual-equivalence}]
As discussed in Section \ref{ssec:primdualequ}, we
construct quadratic surrogates by Taylor expansion. Hence, we have the following expressions for $\surrog{\ell}$ and $\surrog{R}$:
    \begin{align*}
        \surrog{\ell}(z_j; y_j) =& \frac{1}{2} \ddot{\ell}(\xv_j^\top \estim{\betav}; y_j) (z_j - \xv_j^\top \estim{\betav})^2
        + \dot{\ell}(\xv_j^\top \estim{\betav}; y_j) (z_j - \xv_j^\top \estim{\betav}) + c, \\
        \surrog{R}(\betav) =& \frac{1}{2} (\betav - \estim{\betav})^\top [\nabla^2 R(\estim{\betav})] (\betav - \estim{\betav})
        + [\nabla R(\estim{\betav})]^\top (\betav - \estim{\betav}) + d,
    \end{align*}
    where $c, d \in \RR$ are constants that do not affect the location of the
    optimizer. We now compute the convex conjugate of $\surrog{\ell}$ and
    $\surrog{R}$, and we obtain that:
    \begin{align}
        \surrog{\ell}^*(w_j; y_j) =& \frac{1}{2} \frac{1}{\ddot{\ell}(\xv_j^\top
        \estim{\betav}; y_j)} (w_j - \dot{\ell}(\xv_j^\top \estim{\betav}; y_j))^2
        + (\xv_j^\top \estim{\betav}) (w_j - \dot{\ell}(\xv_j^\top
        \estim{\betav}; y_j)) + c', \label{eq:equiv:dual-quad-primal-loss}\\
        \surrog{R}^*(\muv) =& \frac{1}{2} (\muv - \nabla R(\estim{\betav}))^\top
        [\nabla^2 R(\estim{\betav})]^{-1} (\muv - \nabla R(\estim{\betav}))
        + \estim{\betav}^\top (\muv - \nabla R(\estim{\betav})) +
        d',\label{eq:equiv:dual-quad-primal-reg}
    \end{align}
    where again $c', d' \in \RR$ are constants. Now, we wish to relate \eqref{eq:equiv:dual-quad-primal-loss} and
    \eqref{eq:equiv:dual-quad-primal-reg} to $\surrog{\ell}_D^*$ and
    $\surrog{R}_D^*$. By substituting the primal-dual correspondence described in
    \eqref{eq:primal-dual-correspondence}, for components of
    \eqref{eq:equiv:dual-quad-primal-loss} and
    \eqref{eq:equiv:dual-quad-primal-reg}, we obtain that:
    \begin{align}
        \surrog{\ell}^*(w_j; y_j)
        &= \frac{1}{2} \frac{1}{\ddot{\ell}(\dot{\ell}^*(-\estim{\duals}_j; y_j); y_j)} (w_j + \estim{\duals}_j)^2
            + \dot{\ell}^*(-\estim{\duals}_j; y_j) (w_j + \estim{\duals}_j) +
            c', \label{eq:equivalence:primal-substituted-l} \\
        \surrog{R}^*(\muv)
        &= \frac{1}{2}(\muv - \Xv^\top \estim{\dualv})^\top[\nabla^2 R(\nabla
    R^*(\Xv^\top \estim{\dualv}))]^{-1} (\muv - \Xv^\top \estim{\dualv})
    \nonumber \\
        &\quad + [\nabla R^* (\Xv^\top \estim{\dualv})]^\top (\muv - \Xv^\top \estim{\dualv}) + d'.
    \label{eq:equivalence:primal-substituted-R}
    \end{align}
    
    To conclude, we note that according to Lemma \ref{lemma:hessianprimaldual} we have
    \begin{equation}\label{eq:equivalence:hessian-correspondence}
    \begin{gathered}
        \ddot{\ell}(\dot{\ell}^*(-\estim{\duals}_j; y_j); y_j) =
        (\ddot{\ell}^*(-\estim{\duals}_j; y_j))^{-1}, \\
        \nabla^2 R(\nabla R^*(\Xv^\top \estim{\dualv})) = [\nabla^2 R^*(\Xv^\top \estim{\dualv})]^{-1}.
    \end{gathered}
    \end{equation}
    
    Substitute \eqref{eq:equivalence:hessian-correspondence} in
    \eqref{eq:equivalence:primal-substituted-l} and
    \eqref{eq:equivalence:primal-substituted-R} we obtain the dual of the
    quadratic surrogate equals
    \begin{align}\label{eq:quadratic-dual}
        \frac{1}{2}\sum_j \surrog{\ell}^*(-\theta_j; y_j) +
        \surrog{R}^*(\Xv^\top\theta)
        &=
        \frac{1}{2} \sum_j \ddot{\ell}^*(-\estim{\duals}_j; y_j) \Big(-\theta_j +
        \estim{\duals}_j + \frac{\dot{\ell}^*(-\estim{\theta}_j;
        y_j)}{\ddot{\ell}^*(-\estim{\theta}_j; y_j)}\Big)^2 \nonumber \\
        &\quad +
        \frac{1}{2}(\Xv^\top\thetav - \Xv^\top \estim{\dualv}) \nabla^2
        R^*(\Xv^\top \estim{\dualv}) (\Xv^\top\thetav - \Xv^\top
        \estim{\dualv}) \nonumber \\
        &\quad + [\nabla R^*(\Xv^\top \estim{\dualv})]^\top
        (\Xv^\top\thetav - \Xv^\top \estim{\dualv}) + c'.
    \end{align}
    
    Note that the formula given in \eqref{eq:quadratic-dual} exactly
    corresponds to the second-order Taylor expansion of
    \eqref{eq:dual-general-optimal}.
\end{proof}

Now, we would like to prove Theorem \ref{thm:primal-dual-equivalence-2}. 
\begin{proof}[Proof of Theorem \ref{thm:primal-dual-equivalence-2}]
    We noted in Section \ref{ssec:dual:general-case} that our dual method as
    described explicitly approximates the loss by its quadratic expansion
    at the optimal value. We may thus assume without loss of generality that the loss is given by
    $\ell(\mu; y) = (\mu - y)^2 / 2$. In this case, as stated in Section \ref{ssec:dual:general-case}, we have
    that
    \begin{equation*}
        \estim{\dualv} = \proxv_{g}(\yv),
    \end{equation*}
    where we have defined $g(\uv) = R^*(\Xv^\top \uv)$. In addition, we note that the augmented observation
    vector $\yv_{a}$ must have its $i$\tsup{th} observation lie on the leave-$i$-out regression line by definition,
    and in particular we have that:
    \begin{equation*}
        [\proxv_{g}(\yv_{a})]_i = 0.
    \end{equation*}

    This motivated us to solve for $\leavei{\surrog{y}_{i}}$ by linearly
    expanding $\proxv_g$ and considering the intersection of its $i$\tsup{th}
    coordinate with 0. Specifically, the desired $\leavei{\surrog{y}_i}$ is
    obtained from the solution of the following linear equation in $z$:
    \begin{equation}\label{eq:dual-optim-linear}
        [\proxv_{g}(\yv) + \Jv_{\proxv_g}(\yv)\ev_i(z - y_i)]_i = 0.
    \end{equation}
    where $\Jv_{\proxv_g}(\yv)$ denotes the Jacobian matrix of $\proxv_g$ at
    $\yv$. We show that if $R^*$ is replaced with its quadratic surrogate $\surrog{R}^*$ as defined in
    Theorem \ref{thm:primal-dual-equivalence}, then:
    \begin{equation*}
        [\proxv_{\surrog{g}}(\surrog{\yv}_a)]_i = 0,
    \end{equation*}
    where $\surrog{g}(\uv) = \surrog{R}^*(\Xv^\top\uv)$, and $\surrog{\yv}_a$ denotes the vector $\yv$,
    except with its $i$\tsup{th} coordinate replaced by the ALO value $\leavei{\surrog{\yv}}_i$.
    Let us note that as $\surrog{g}$ is quadratic, its proximal map $\proxv_{\surrog{g}}$ is linear,
    and the equation may thus be solved directly by a single Newton's step. As
    a linear map is characterized by its intercept and slope, compared with
    \eqref{eq:dual-optim-linear}, it remains to show that:
    \begin{align}
        \proxv_{g}(\yv) &= \proxv_{\surrog{g}}(\yv),
        \label{eq:equiv:check-prox-intercept} \\
        \Jv_{\proxv_{g}}(\yv) &= \Jv_{\proxv_{\surrog{g}}}(\yv).
        \label{eq:equiv:check-prox-coef}
    \end{align}
    
    We note that \eqref{eq:equiv:check-prox-intercept} is immediate from the definition of $\tilde{g}$, as
    both the left and right hand sides are equal to the dual optimal
    $\hat{\dualv}$. In order to show \eqref{eq:equiv:check-prox-coef}, since $\surrog{g}$ is quadratic, we may compute its proximal map
    exactly. From the previous section, we have that:
    \begin{equation*}
        \surrog{g}(\dualv)
        =
        \frac{1}{2}(\thetav - \estim{\dualv})^\top\Xv[\nabla^2 R(\nabla
        R^*(\Xv^\top \estim{\dualv}))]^{-1}
        \Xv^\top(\thetav - \estim{\dualv}) + [\nabla R^*
        (\Xv^\top\estim{\dualv})]^\top\Xv^\top(\thetav - \estim{\dualv}),
    \end{equation*}
    
    We minimize $\frac{1}{2}\|\yv - \thetav\|_2^2 + \surrog{g}(\thetav)$ in
    $\thetav$ and get
    \begin{equation*}
        \proxv_{\surrog{g}}(\yv)
        = (\Iv + \Xv [\nabla^2 R(\nabla R^*(\Xv^\top\estim{\thetav}))]^{-1} \Xv^\top)^{-1}(\yv -
        \Xv\nabla R^*(\Xv^\top\estim{\thetav})),
    \end{equation*}
    
    Note that the primal dual correspondence implies $\estim{\betav} = \nabla
    R^*(\Xv^\top\estim{\thetav})$. In particular we may compute the Jacobian of
    $\proxv_{\surrog{g}}$ at $\yv$ as $(\Iv + \Xv [\nabla^2
    R(\estim{\betav})]^{-1} \Xv^\top)^{-1}$. On the other hand, according to part (ii) of Lemma \ref{lem:proxproperties} we know that the proximal operator $\proxv_g$ is exactly the resolvent of
    the subgradient $\partial g$, i.e.,
    \begin{equation*}
        \proxv_g = (I + \partial g)^{-1},
    \end{equation*}
    and in particular we have
    \begin{equation*}
        \proxv_g(\yv) + \nabla g(\proxv_g(\yv)) = \yv.
    \end{equation*}
   Taking derivative again with respect to $\yv$ and applying the chain rule, we obtain
    \begin{equation*}
        \Jv_{\proxv_g}(\yv)(\Iv + \nabla^2 g(\proxv_g(\yv))) = \Iv,
    \end{equation*}
    and hence 
    \begin{equation*}
        \Jv_{\proxv_g}(\yv) = (\Iv + \nabla^2 g(\proxv_g(\yv))^{-1}.
    \end{equation*}
      Now, note that we have $\proxv_g(\yv) = \estim{\dualv}$, and that:
    \begin{equation*}
        \nabla^2 g(\estim{\dualv}) = \Xv [\nabla^2 R^*(\Xv^\top\estim{\dualv})] \Xv^\top.
    \end{equation*}
      We are thus done by Lemma \ref{lemma:hessianprimaldual}.
\end{proof}

\begin{proof}[Proof of Lemma \ref{prox:equivalence}]
As is clear from \eqref{eq:surrogate}, for $\tilde{\Jv}$ we have 
\begin{equation*}
    \surrog{\Jv}
    = \big[\Iv+ \nabla^2 R(\estim{\betav})\big]^{-1}.
\end{equation*}

Now let us look at $\Jv$. Using the definition $\proxv_R(\uv) =
\argmin_{\zv\in\mathbb{R}^p} \frac{1}{2} \|\uv -\zv\|_2^2 + R(\zv)$, we have
the following holds
\begin{equation*}
    \proxv_R(\uv)- \uv + \nabla R(\proxv_R(\uv))= \bm{0}.
\end{equation*}

Taking derivatives on both sides of the above equation, we obtain $\Jv(\uv) -
\Iv + \nabla^2 R\big(\proxv_R(\uv)\big) \Jv(\uv) = \bm{0}$. This leads to
\begin{equation}\label{eq:proxjac_calc}
    \Jv(\uv) = \big[\Iv + \nabla^2 R\big(\proxv_R(\uv)\big)\big]^{-1}.
\end{equation}

Note that the Jacobian should be calculated at $\uv = \estim{\betav} -
\sum_{j=1}^n \dot{\ell}(\xv_j^\top\estim{\betav}; y_j)\xv_j$, which implies
that $\proxv_R(\uv) =  \estim{\betav}$. Plugging this in
\eqref{eq:proxjac_calc} we obtain that  $\Jv = \surrog{\Jv}$.
\end{proof}


\subsection{Proof of Primal Approximation Approach}\label{sec:primal-approx}
In this section, we prove the results of our primal approach on nonsmooth models presented
in Section \ref{sec:primal-smoothing}. Since we use a kernel smoothing
strategy, we start with some useful preliminary results on kernel smoothing. We then
discuss nonsmooth regularizer and  nonsmooth loss respectively.

\subsubsection{Properties of Kernel Smoothing}
\label{append:ssec:property-kernel-smoothing}

Consider the following smoothing strategy for a convex
function $f: \mathbb{R} \rightarrow \RR$:
\begin{equation}\label{eq:smoothingffunc}
    f_h(z) = \frac{1}{h}\int f(u)\phi((z - u)/h) du,
\end{equation}
where $\phi$ satisfies the conditions clarified in Section
\ref{ssec:nonsmooth-loss}. Let $K := \{ v_1, \ldots, v_k \}$ denote the set of
zeroth-order singularities of the
function $f$. Denote by $\dot{f}_-(v)$ and $\dot{f}_+(v)$ the left and right
derivative of $f$ at $v$. Our next lemma summarizes some of the basic properties of
$f$ that may be used in the proofs of Theorem \ref{thm:nonsmooth-loss-approx}
and \ref{thm:nonsmooth-reg-approx} of the main text.

\begin{lemma}\label{lemma:kernel-smooth-property}
The smooth function $f_h$ satisfies the following properties:
\begin{enumerate}
    \item
        $f_h(z) \geq f(z)$ for all $z \in \mathbb{R}$;
    \item For all $z \in K^C$, for all $h$ small enough:
        \begin{equation*}
            \dot{f}_h(z) = \frac{1}{h}\int \dot{f}(u)\phi((z - u)/h) du
            ,\quad
            \ddot{f}_h(z) = \frac{1}{h}\int \ddot{f}(u)\phi((z - u)/h) du.
        \end{equation*}
    \item For all $z \in K$:
        \begin{equation*}
            \lim_{h \rightarrow 0} \dot{f}_h(z) = \frac{\dot{f}_{-}(z) + \dot{f}_{+}(z)}{2}
            ,\quad
            \lim_{h \rightarrow 0} \ddot{f}_h(z) = +\infty.
        \end{equation*}
    \item
        If $f$ is locally Lipschiz in the sense that, for any $A > 0$, and for any
        $x, y \in [-A, A]$, we have $|f(x) - f(y)| \leq L_A|x - y|$, where
        $L_A$ is a constant that only depends on $A$; then $f_h(z)$ converges
        to $f(z)$ uniformly on any compact set.
\end{enumerate}
\end{lemma}

\begin{proof}
    For part 1, by the normalization property of $\phi$, we can treat $\phi$ as a probability
    density. Consider the random variable $U \sim \frac{1}{h}\phi(\frac{z -
    u}{h})$. From the convexity of $f$ and Jensen's inequality we have
    \begin{equation*}
        f_h(z) = \mathbb{E}f(U) \geq f(\mathbb{E}U) = f(z).
    \end{equation*}
   
    For part 2, note that 
    \begin{equation*}
        \dot{f}_h(z)
        =
        \frac{1}{h^2}\int f(u)\dot{\phi}((z - u) / h) du
        =
        \int \dot{f}(u)\frac{1}{h}\phi((z - u) / h) du.
    \end{equation*}
    
    A similar computation gives the stated equation for $\ddot{f}_h(z)$.
    
    For part 3, when $z \in K$, we have by compact support of $\phi$ that
    as $h \rightarrow 0$:
    \begin{align*}
        \dot{f}_h(z)
        &=
        \frac{1}{h^2}\int_{z-hC}^{z} f(u)\dot{\phi}((z - u) / h) du +
        \frac{1}{h^2}\int_{z}^{z+hC} f(u)\dot{\phi}((z - u) / h) du \nonumber \\
        &=
        \int_{-C}^{0} \dot{f}(z - hw)\phi(w) dw + \int_{0}^{C} \dot{f}(z - hw)\phi(w) dw \nonumber \\
        & \rightarrow
        \int_{-C}^{0} \dot{f}_+(z)\phi(w) dw
        + \int_{0}^{C} \dot{f}_{-}(z)\phi(w) dw \nonumber \\
        &= \frac{\dot{f}_+(z) + \dot{f}_-(z)}{2}.
    \end{align*}
    To obtain the last equality we have used the symmetry of $\phi$.  A similar computation for the second-order derivative yields:
    \begin{align*}
        \ddot{f}_h(z)
        &=
        \frac{1}{h^3}\int_{z-hC}^{z} f(u)\ddot{\phi}((z - u) / h) du +
        \frac{1}{h^3}\int_{z}^{z+hC} f(u)\ddot{\phi}((z - u) / h) du \nonumber
        \\
        &=
        \frac{1}{h}\phi(0)(\dot{f}_+(z) - \dot{f}_-(z)) + \int_0^C \ddot{f}(z -
        hw)\phi(w)dw + \int_{-C}^0 \ddot{f}(z - hw)\phi(w)dw \nonumber \rightarrow
        \infty.
    \end{align*}
    The last claim holds because $\dot{f}_+(z) > \dot{f}_-(z)$.

    For part 4, for any compact set $\mathcal{C}$ which can be covered by a large enough set
    $[-A, A]$ for some $A > 0$, we have
    \begin{equation*}
        \sup_{z \in \mathcal{C}}|f_h(z) - f(z)|
        \leq
        \sup_{z \in \mathcal{C}}\int_{-C}^C|f(z - hw) - f(z)|\phi(w)dw
        \leq 2hCL_{A+C}
        \rightarrow 0
        ,\quad
        \text{ as }h \rightarrow 0
    \end{equation*}
\end{proof}

Having established the basic properties of our smoothing strategy,
we apply them to non-smooth regularizers and non-smooth losses in the next two sections.

\subsubsection{Proof of Theorem \ref{thm:nonsmooth-reg-approx}}
\label{append:ssec:primal-nonsmooth-reg}

Consider the penalized regression problem:
\begin{equation}\label{eq:nonsmooth-reg-main}
    \estim{\betav} = \argmin_{\betav} \sum_{j = 1}^n \ell(\xv_j^\top\betav; y_j)
    + \lambda \sum_l r(\beta_l).
\end{equation}
with $\ell$ and $r$ being twice differentiable and nonsmooth functions
respectively. Let $r_h$ be the smoothed version of $r$ constructed as
in \eqref{eq:smoothingffunc}. Define
\begin{equation*}
    \estim{\betav}_h=\arg\min_{\betav}\sum_j \ell(\xv_j^\top\betav; y_j) + \lambda \sum_l r_h(\beta_l).
\end{equation*}
As before, let $K$ denote the set of all zeroth-order singularities of $r$.

Let us look at Assumption \ref{assum:nonsmooth-reg-assump1}. Note that 1 and 4
hold for all the popular regularizers. The second one also holds in almost all
applications. Finally, note that at $\estim{\beta}_l = v \in K$, we always have
$g_r(\estim{\beta}_l) \in [\dot{r}_-(v), \dot{r}_+(v)]$. Hence, assumption 3
implies that $g_r(\estim{\beta}_l) \neq \dot{r}_-(v)$ and $g_r(\estim{\beta}_l)
\neq \dot{r}_+(v)$. Note  the event $g_r(\estim{\beta}_l) = \dot{r}_-(v)$ or
$g_r(\estim{\beta}_l) = \dot{r}_+(v)$ only holds when  $\estim{\beta}_l \in K$,
but very small perturbation of data pushes $\estim{\beta}_\ell$ out of $K$.
Such events happen in rare (detectable) occasions, and do not pose any serious
limitation to our $\alo$ formulas.

\begin{lemma}\label{lemma:nonsmooth-reg-boundedness}
    Suppose that Assumption \ref{assum:nonsmooth-reg-assump1} holds.
    There exists $M > 0$ that only depends on $r, \ell$ and $\lambda$, such that
    we have for any $h \leq 1$:
    \begin{equation*}
        \norm{\estim{\betav}}_\infty, \norm{\estim{\betav}_h}_\infty < M.
    \end{equation*}
\end{lemma}
\begin{proof}
    Let $h \leq 1$, then the minimizer of the smoothed version $\estim{\betav}_h$ satifies
    \begin{align*}
        \lambda \sum_{l=1}^p r([\estim{\betav}_h]_l) &\overset{(a)}{\leq} \lambda \sum_{l=1}^p r_h([\estim{\betav}_h]_l) 
        \leq \sum_i \ell(y_i; 0) + \lambda p r_h(0) \nonumber \\
        &= \sum_i \ell(y_i; 0) + \lambda p \int_{-C}^C r(hw) \phi(w) dw \nonumber \\
        & \leq \sum_i \ell(y_i; 0) + \lambda p \sup_{|w|\leq C} r(w).
    \end{align*}
Note that Inequality (a) is due to Lemma \ref{lemma:kernel-smooth-property}(i). The convexity and coerciveness of $r$ imply that there exists an $M$, such that $ \norm{\estim{\betav}_h}_\infty \leq M$. Similarly, the minimizer $\estim{\betav}$ of the original problem satisfies
    \begin{equation*}
        \lambda \sum_{l=1}^p r([\estim{\betav}]_l) \leq \sum_i \ell(y_i; 0) + \lambda p r(0)
        \leq \sum_i \ell(y_i; 0) + \lambda p \sup_{|w|\leq C} r(w),
    \end{equation*}
    and hence $\norm{\estim{\betav}}_\infty \leq M$. 
\end{proof}

\begin{lemma}\label{lemma:nonsmooth-reg-estimator-converge}
    Suppose that Assumption \ref{assum:nonsmooth-reg-assump1} holds. Then the
    smoothed version converges to the original problem in the sense that
    \begin{equation*}
        \norm{\estim{\betav}_h - \estim{\betav}}_2 \rightarrow 0 \text{ as } h \rightarrow 0.
    \end{equation*}
\end{lemma}
\begin{proof}
    By the local Lipschitz condition of $r$, we have for any $z \leq M$ and $h \leq 1$:
    \begin{equation}\label{eq:nonsmooth-reg-estimator-proof-lipschitz}
        0 \leq r_h(z) - r(z)
        =
        \int_{-C}^C [r(z - hw) - r(z)]\phi(w)dw \leq 2CL_{M + C}h.
    \end{equation}
     Let $P_h(\betav) := \sum_j \ell(\xv_j^\top\betav; y_j) + \lambda \sum_l
    r_h(\beta_l)$ denote the primal objective value. Then, \eqref{eq:nonsmooth-reg-estimator-proof-lipschitz}
    implies that
    \begin{equation}\label{eq:upperbounderror}
        \sup_{\|\betav\|_\infty \leq M} \abs{P(\betav) - P_h(\betav)} \leq 2hpCL_{M+C}.
    \end{equation}
      By Lemma \ref{lemma:nonsmooth-reg-boundedness} $\estim{\betav}_h$ is in a
    compact set. Hence, any of its subsequences
    contains a convergent sub-subsequence. Let us abuse the notation and denote
    by $\estim{\betav}_h$ one such convergent sub-subsquence, that is, assume that
    $\estim{\betav}_h \rightarrow \estim{\betav}_0$. We have
    \begin{equation*}
        P(\estim{\betav}_0)
        =
        \lim_{h\rightarrow 0} P(\estim{\betav}_h)
        \overset{(a)}{=}
        \lim_{h\rightarrow 0} P_h(\estim{\betav}_h)
        \overset{(b)}{\leq}
        \lim_{h\rightarrow 0} P_h(\estim{\betav})
        \overset{(c)}{=}
        \lim_{h\rightarrow 0} P(\estim{\betav}).
    \end{equation*}
Inequality (a) is due to \eqref{eq:upperbounderror}. Inequality (b) also holds since $\estim{\betav}_h$ is the minimizer of $P_h(\cdot)$. Finally, Inequality (c) is also due to \eqref{eq:upperbounderror}. The uniqueness of the minimizer implies $\estim{\betav}_0 = \estim{\betav}$.
    As the above holds along any convergent sub-subsequence, we have that:
    \begin{equation*}
        \norm{\estim{\betav}_h - \estim{\betav}}_2 \rightarrow 0 \text{ as } h \rightarrow 0.
    \end{equation*}
\end{proof}

\begin{lemma}[Convergence of the subgradients]\label{lemma:nonsmooth-reg-gradient-converge}
    Suppose that Assumption \ref{assum:nonsmooth-reg-assump1} holds.
    Recall that we use $R(\betav) = \sum_{l=1}^p r(\beta_l)$. We have
    \begin{equation*}
        \|\nabla R_h(\estim{\betav}_h) - \gv_R(\estim{\betav})\|_2
        \rightarrow 0
        ,\quad
        \text{ as } h  \rightarrow 0,
    \end{equation*}
    where $g_R(\estim{\betav})$ is the subgradient of $R$ at $\estim{\betav}$.
\end{lemma}
\begin{proof}
    By the first-order optimality conditions and the continuity of $\ell$, we
    have that as $h \rightarrow 0$:
    \begin{equation*}
        \|\nabla R_h(\estim{\betav}_h) - \gv_R(\estim{\betav}) \|_2
        =
        \Big\|\sum_j \ell(\xv_j^\top\estim{\betav}; y_j) - \sum_j
        \ell(\xv_j^\top\estim{\betav}_h; y_j)\Big\|_2
        \rightarrow
        0.
    \end{equation*}
\end{proof}

\begin{lemma}[Convergence of the Hessian]\label{lemma:nonsmooth-reg-hessian-converge}
    Suppose that Assumption \ref{assum:nonsmooth-reg-assump1} holds. We have that
    as $h \rightarrow 0$:
    \begin{equation*}
    \ddot{r}_h(\estim{\beta}_{h,i})
    \rightarrow
    \begin{cases}
        \ddot{r}(\estim{\beta}_i) & \text{ if } \estim{\beta}_i \notin K, \\
        +\infty & \text{ if } \estim{\beta}_i \in K.
    \end{cases}
    \end{equation*}
\end{lemma}
\begin{proof}
    Let us first consider the case $\estim{\beta}_i \notin K$.
    As $\RR \setminus K$ is open,
    there exists $\delta > 0$ such that
    $[\estim{\beta}_i - \delta, \estim{\beta}_i + \delta] \subset \mathbb{R}
    \backslash K$.
    Since $\estim{\beta}_{h,i} \rightarrow \estim{\beta}_i$ as $h \rightarrow 0$,
    we have for $h$ small enough that:
    \begin{equation*}
        [\estim{\beta}_{h,i} - hC, \estim{\beta}_{h,i} + hC]
        \subset [\estim{\beta}_i - \delta, \estim{\beta}_i + \delta]
        \subset \mathbb{R}\backslash K.
    \end{equation*}
    Since $\ddot{r}$ is smooth on $[\estim{\beta}_i - \delta, \estim{\beta}_i +
    \delta]$, by the dominated convergence theorem, we have as $h \rightarrow 0$:
    \begin{equation*}
        \ddot{r}_h(\estim{\beta}_{h,i})
        =
        \int_{-C}^C \ddot{r}(\estim{\beta}_{h,i} - hw)\phi(w)dw
        \rightarrow
        \int_{-C}^C \ddot{r}(\estim{\beta}_i)\phi(w)dw
        =
        \ddot{r}(\estim{\beta}_i)
    \end{equation*}
    Now, let us consider the case where $\estim{\beta}_i \in K$.
    By Lemma \ref{lemma:nonsmooth-reg-gradient-converge},
    we have that $\dot{r}_h(\estim{\beta}_{h,i}) \rightarrow g_r(\estim{\beta}_i)$,
    from which we deduce:
    \begin{equation*}
        |\estim{\beta}_{h,i} - \estim{\beta}_i| < hC.
    \end{equation*}
    Indeed, if we had $\estim{\beta}_i \geq \estim{\beta}_{h,i} + hC$, then this would imply:
    \begin{equation*}
        \dot{r}_h(\estim{\beta}_{h,i}) = \int_{-C}^C \dot{r}(\estim{\beta}_{h,i} - hw)\phi(w)dw
        \leq
        \dot{r}_{-}(\estim{\beta}_i)
        < g_r(\estim{\beta}_i),
    \end{equation*}
    which is in contradiction with $\dot{r}_h(\estim{\beta}_{h,i}) \rightarrow g_r(\estim{\beta}_i)$. The same happens if $\estim{\beta}_i \leq \estim{\beta}_{h,i} - hC$.
    To conclude, note that as $h \rightarrow 0$:
    \begin{align*}
        \ddot{r}_h(\estim{\beta}_{h,i})
        &= \int_{\estim{\beta}_{h, i}-hC}^{\estim{\beta}_i}
        r(u)\frac{1}{h^3}\ddot{\phi}\Big(\frac{\estim{\beta}_{h, i} - u}{h}\Big) du
        + \int_{\estim{\beta}_i}^{\estim{\beta}_{h, i}+hC}
        r(u)\frac{1}{h^3}\ddot{\phi}\Big(\frac{\estim{\beta}_{h, i} - u}{h}\Big) du \nonumber \\   
        &= \frac{1}{h}\phi\Big(\frac{\estim{\beta}_{h, i} -
        \estim{\beta}_i}{h}\Big)(\dot{r}_+(\estim{\beta}_i) - \dot{r}_-(\estim{\beta}_i))
        + \int_{\frac{\estim{\beta}_{h, i} - \estim{\beta}_i}{h}}^C \ddot{r}(\estim{\beta}_{h, i} - hw)\phi(w)dw \nonumber \\
        & \quad + \int_{-C}^{\frac{\estim{\beta}_{h, i}
        - \estim{\beta}_i}{h}} \ddot{r}(\estim{\beta}_{h, i} - hw)\phi(w)dw \nonumber \\
        &\rightarrow +\infty.
    \end{align*}
\end{proof}

\begin{lemma}\label{lemma:woodbury-block}
    Consider a sequence of matrices $\Av_n, n \in \mathbb{N}$, and let $\Av_n =
    \begin{bmatrix}
        \Av_{1n} & \Av_{2n} \\
        \Av_{3n} & \Av_{4n}
    \end{bmatrix}$ where $\Av_{1n}, \Av_{4n}$ are invertible for all $n$. Additionally,
    suppose that $\Av_{in} \rightarrow \Av_i, i = 1, 2, 3$, and $\Av_{4n}^{-1} \rightarrow \zerov$ as $n \rightarrow \infty$.
    Then we have as $n \rightarrow \infty$ that:
    \begin{equation*}
        \Av_n^{-1} \rightarrow
        \begin{bmatrix}
            \Av_1^{-1} & \zerov \\
            \zerov & \zerov
        \end{bmatrix}.
    \end{equation*}
\end{lemma}
\begin{proof}
    By the block matrix inversion lemma, we have
    \begin{align*}
        \Av_n^{-1} =&
        \begin{bmatrix}
            (\Av_{1n} - \Av_{2n}\Av_{4n}^{-1}\Av_{3n})^{-1} & -(\Av_{1n} -
            \Av_{2n}\Av_{4n}^{-1}\Av_{3n})^{-1}\Av_{2n}\Av_{4n}^{-1} \\
            -\Av_{4n}^{-1}\Av_{3n}(\Av_{1n} - \Av_{2n}\Av_{4n}^{-1}\Av_{3n})^{-1} &
            \Av_{4n}^{-1}\Av_{3n}(\Av_{1n} -
            \Av_{2n}\Av_{4n}^{-1}\Av_{3n})^{-1}\Av_{2n}\Av_{4n}^{-1} + \Av_{4n}^{-1}
        \end{bmatrix} \nonumber \\
        \rightarrow&
        \begin{bmatrix}
            \Av_{1}^{-1} & \zerov \\
            \zerov & \zerov
        \end{bmatrix}.
    \end{align*}
\end{proof}

\begin{proof}[Proof of Theorem \ref{thm:nonsmooth-reg-approx}]
     We remind the reader that we have
    \begin{equation*}
    \surrogi{\betav}_h := \estim{\betav}_h
    + \Big[\sum_{j\neq i} \xv_j\xv_j^\top\ddot{\ell}(\xv_j^\top\estim{\betav}_h; y_j) + \nabla^2 R_h(\estim{\betav}_h)\Big]^{-1}
    \xv_i\dot{\ell}(\xv_i^\top\estim{\betav}_h; y_i).
\end{equation*}
We have proved in Lemma \ref{lemma:nonsmooth-reg-estimator-converge} that $\estim{\betav}_h \rightarrow \estim{\betav}$. Hence, the only remaining step is to simplify the limit of the matrix $\Big[\sum_{j\neq i} \xv_j\xv_j^\top\ddot{\ell}(\xv_j^\top\estim{\betav}_h; y_j) + \nabla^2 R_h(\estim{\betav}_h)\Big]^{-1}$. We remind the reader that $\nabla^2 R_h(\estim{\betav}_h)$ is a diagonal matrix, and according to Lemma \ref{lemma:nonsmooth-reg-hessian-converge} if $\estim{\betav}_{h,i} \notin A$, then $\ddot{r}_h(\estim{\beta}_{h,i}) \rightarrow \infty$. Hence, we can use Lemma \ref{lemma:woodbury-block} and simplify $\Big[\sum_{j\neq i} \xv_j\xv_j^\top\ddot{\ell}(\xv_j^\top\estim{\betav}_h; y_j) + \nabla^2 R_h(\estim{\betav}_h)\Big]^{-1}$ to $[\sum_{j\neq i} \xv_{j, A}\xv_{j, A}^\top\ddot{\ell}(\xv_{j,A}^\top\estim{\betav}_A; y_j)+
    \nabla^2 R(\estim{\betav}_{A})]^{-1}$. 
\end{proof}

\subsubsection{Proof of Theorem \ref{thm:nonsmooth-loss-approx}}
\label{append:ssec:primal-nonsmooth-loss}

Consider nonsmooth loss $\ell$ and its smoothed version $\ell_h$. $R$ is assumed
to be smooth. Let us consider:
\begin{equation*}
    \begin{aligned}
        P(\betav) &= \sum_{j=1}^n \ell(\xv_j^\top\betav; y_j) + R(\betav), \\
        P_h(\betav) &= \sum_{j = 1}^ n \ell_h(\xv_j^\top\betav; y_j) + R(\betav).
    \end{aligned}
\end{equation*}

We use notations $\estim{\betav} = \argmin_{\betav} P(\betav)$
and $\estim{\betav}_h =\argmin_{\betav} P_h(\betav)$ to denote the optimizers.
Let $K = \{ v_1, \dotsc, v_k \}$ denote the zeroth-order singularities
of $\ell$, and let $V = \{ i: \xv_i^\top \estim{\betav} \in K \}$ be the set of indices
of observations at such singularities.

In Asumption \ref{assum:nonsmooth-loss-assump1}, 1 and 5 hold for all the
problems of interest. Assumption 3 also holds for almost all practical
problems. The discussion of assumption 4 is similar to discussion of part (3) of Assumption \ref{assum:nonsmooth-reg-assump1}. Hence, we skip it. Note that the second assumption is also required for the stability of our solution. If it does not hold, removing one data point can dramatically change the solution and make our approximations inaccurate. 

\begin{lemma}\label{lemma:nonsmooth-loss-boundedness}
    Suppose that Assumption \ref{assum:nonsmooth-loss-assump1} holds.
    There exists $M > 0$ that only depends on $r, \ell$ and $\lambda$,
    such that for all $h \leq 1$, we have:
    \begin{equation*}
        \norm{\estim{\betav}}_\infty \leq M \text{ and } \norm{\estim{\betav}_h}_\infty \leq M.
    \end{equation*}
\end{lemma}
\begin{proof}
    Let $h \leq 1$, then $\estim{\betav}_h$ satisfies
    \begin{align*}
        R(\estim{\betav}_h)
        \leq& \sum_j \ell_h(0; y_j) + p R(0) \nonumber \\
        =& \sum_j \int_{-C}^C \ell(hw; y_j) \phi(w) dw + p R(0)
        \leq \sum_j \sup_{|w|\leq C}\ell(w; y_i) + p R(0).
    \end{align*}
      The convexity and coerciveness of $R$ implies that there exists a $M$, such that
    for all $h \leq 1$, $\|\estim{\betav}_h\|_2 \leq M$. Similarly, for $\estim{\betav}$ we have
    \begin{equation*}
        R(\estim{\betav})
        \leq \sum_j \ell(0; y_j) + p R(0)
        \leq \sum_j \sup_{|w|\leq C}\ell(w; y_i) + p R(0),
    \end{equation*}
 and hence $\|\estim{\betav}\|_2 \leq M$.
\end{proof}

\begin{lemma}\label{lemma:nonsmooth-loss-estimator-converge}
    Suppose that Assumption \ref{assum:nonsmooth-loss-assump1} holds. We have that as
    $h \rightarrow 0$:
    \begin{equation*}
        \|\estim{\betav}_h - \estim{\betav}\|_2 \rightarrow 0.
    \end{equation*}
\end{lemma}
\begin{proof}
    Let $M_x = \max_i\|\xv_i\|_2$.
    By the local Lipschitz condition of $\ell$, we have that for any $\|\betav\|_2 \leq M$ and $h \leq 1$ 
    \begin{align*}
        0 &\leq \ell_h(y_i; \xv_i^\top\betav) - \ell(y_i; \xv_i^\top\betav) \\
          &= \int_{-C}^C [\ell(y_i; \xv_i^\top\betav - hw) - \ell(y_i;\xv_i^\top\betav)]\phi(w)dw \\
          &\leq 2CL_{M_xM + C}h.
    \end{align*}
 Note that the first inequality is a result of Lemma \ref{lemma:kernel-smooth-property}(i).  
    This implies that
    \begin{equation}\label{eq:uppersmootheddiff}
        \sup_{\|\betav\|_2 \leq M} |P(\betav) - P_h(\betav)| \leq
        2nhCL_{M_xM+C}.
    \end{equation}
    From Lemma \ref{lemma:nonsmooth-loss-boundedness}, we know
    $\estim{\betav}_h$ is in a compact set, thus any of its subsequence
    contains a convergent sub-subsequence. Again abuse the notation and
    let $\estim{\betav}_h$ denote this convergent sub-subsequence. Suppose that
    $\estim{\betav}_h \rightarrow \estim{\betav}_0$. We have
    \begin{equation*}
        P(\estim{\betav}_0)
        =
        \lim_{h\rightarrow 0} P(\estim{\betav}_h)
        \overset{(a)}{=}
        \lim_{h\rightarrow 0} P_h(\estim{\betav}_h)
        \overset{(b)}{\leq}
        \lim_{h\rightarrow 0} P_h(\estim{\betav})
        \overset{(c)}{=}
        \lim_{h\rightarrow 0} P(\estim{\betav}).
    \end{equation*}
    Note that Equality (a) is due to \eqref{eq:uppersmootheddiff}. Inequality (b) is due to Lemma \ref{lemma:kernel-smooth-property}(i), and finally Equality (c) is due to\eqref{eq:uppersmootheddiff}.  The uniqueness implies that $\estim{\betav}_0 = \estim{\betav}$. Since this
    holds along any sub-subsequence, we deduce that $ \| \estim{\betav}_h - \estim{\betav} \|_2 \rightarrow 0$.
\end{proof}

\begin{lemma}[Convergence of gradients]\label{lemma:nonsmooth-loss-gradient-converge}
    Suppose that Assumption \ref{assum:nonsmooth-loss-assump1} holds. Then, we have that
    for any $j$, as $h \rightarrow 0$
    \begin{equation*}
        \| \dot{\ell}_h(\xv_j^\top\estim{\betav}_{h}) -
        g_\ell(\xv_j^\top\estim{\betav}) \|_2 \rightarrow 0.
    \end{equation*}
\end{lemma}
\begin{proof}
    for $j \notin V$, the result is immediate. For $j \in V$, we have that
    as $h \rightarrow 0$:
    \begin{equation*}
        \Big\|\sum_{j \in V} \xv_j\dot{\ell}_h(\xv_j^\top\estim{\betav}_h; y_j) -
        \sum_{j \in V} \xv_j g_\ell(\xv_j^\top\estim{\betav}; y_j) \Big\|_2
        \rightarrow
        0.
    \end{equation*}
    This combined with Assumption \ref{assum:nonsmooth-loss-assump1}(ii)  proves the result.
\end{proof}

\begin{lemma}[Convergence of Hessian]\label{lemma:nonsmooth-loss-hessian-converge}
    Suppose that Assumption \ref{assum:nonsmooth-loss-assump1} holds. Then, we have
    that for any $j$, as $h \rightarrow 0$
    \begin{equation*}
    \ddot{\ell}_h(\xv_j^\top\estim{\betav}_h; y_j)
    \rightarrow
    \begin{cases}
        \ddot{\ell}(\xv_j^\top\estim{\betav}; y_j) & \text{ if } j \notin V, \\
        +\infty & \text{ if } j \in V. \\
    \end{cases}
    \end{equation*}
\end{lemma}
\begin{proof}
The result follows through a similar argument as
    in the proof of Lemma \ref{lemma:nonsmooth-reg-hessian-converge}
    for $j \notin V$.
    For $j \in V$, we have by Lemma \ref{lemma:nonsmooth-loss-gradient-converge} that
    as $h \rightarrow 0$:
    \begin{equation*}
        \dot{\ell}_h(\xv_j^\top\estim{\betav}_h; y_j) \rightarrow g_\ell(\xv_j^\top\estim{\betav}; y_j).
    \end{equation*}
    Following a similar reasoning as in the proof of Lemma \ref{lemma:nonsmooth-reg-hessian-converge},
    we have that:
    \begin{equation*}
        |\xv_j^\top\estim{\betav}_h - \xv_j^\top\estim{\betav}| < hC.
    \end{equation*}
    Finally, we note that as $h \rightarrow 0$:
    \begin{equation*}
        \ddot{\ell}_h(\xv_j^\top\estim{\betav}_h; y_j)
        \geq
        \frac{1}{h}\phi\Big(\frac{\xv_j^\top\estim{\betav}_h -
        \xv_j^\top\estim{\betav}}{h}\Big)(\dot{\ell}_+(\xv_j^\top\estim{\betav}) -
        \dot{\ell}_-(\xv_j^\top\estim{\betav}))
        \rightarrow
        +\infty.
    \end{equation*}
\end{proof}

\begin{proof}[Proof of Theorem \ref{thm:nonsmooth-loss-approx}]
Recall $V = \{i: \xv_i^\top \estim{\betav} \in K\}$ and $S=[1:n]\backslash V$.
Let $\Hv_h$ be the matrix in ALO for smooth loss and smooth regularizer when
using $\ell_h$. Let $\Lv_h=\diag[\{\ddot{\ell}_h(\xv_j^\top\estim{\betav}; y_j)\}_j]$,
$\Lv_S=\diag[\{\ddot{\ell}(\xv_j^\top\estim{\betav}; y_j)\}_{j \in S}]$.
$\Lv_{h, S}$ and $\Lv_{h, V}$ are similarly defined. Recall
\begin{equation*}
    \Hv_h = \Xv(\lambda \nabla^2 R + \Xv^\top \Lv_h\Xv)^{-1}\Xv^\top.
\end{equation*}

We then have
\begin{align*}
    &(\lambda \nabla^2 R + \Xv^\top \Lv_h \Xv)^{-1} \nonumber \\
    =&
    (\underbrace{\lambda \nabla^2 R + \Xv_{S,\cdot}^\top \Lv_{h,S} \Xv_{S,\cdot}}_{\Yv_h} + \Xv_{V,\cdot}^\top \Lv_{h,V} \Xv_{V,\cdot})^{-1} \nonumber \\
    =&
    \Yv_h^{-1} - \Yv_h^{-1}\Xv_{V,\cdot}^\top (\Lv_{h,V}^{-1} +
    \Xv_{V,\cdot}\Yv_h^{-1}\Xv_{V,\cdot}^\top )^{-1}\Xv_{V,\cdot} \Yv_h^{-1}.
\end{align*}

As a result, we have
\begin{align*}
    &(\lambda \nabla^2 R + \Xv^\top \Lv_h X)^{-1}\Xv_{V,\cdot}^\top \nonumber \\
    =&
    \Yv_h^{-1}\Xv_{V,\cdot}^\top - \Yv_h^{-1}\Xv_{V,\cdot}^\top (\Lv_{h,V}^{-1} + \Xv_{V,\cdot}\Yv_h^{-1}\Xv_{V,\cdot}^\top )^{-1}\Xv_{V,\cdot} \Yv_h^{-1}\Xv_{V,\cdot}^\top \nonumber \\
    =&
    \Yv_h^{-1}\Xv_{V,\cdot}^\top (\Iv_p - (\Lv_{h,V}^{-1} + \Xv_{V,\cdot}\Yv_h^{-1}\Xv_{V,\cdot}^\top )^{-1}\Xv_{V,\cdot} \Yv_h^{-1}\Xv_{V,\cdot}^\top ) \nonumber \\
    =&
    \Yv_h^{-1}\Xv_{V,\cdot}^\top (\Lv_{h,V}^{-1} +
    \Xv_{V,\cdot}\Yv_h^{-1}\Xv_{V,\cdot}^\top )^{-1}\Lv_{h,V}^{-1}.
\end{align*}

Similarly we can get
\begin{align*}
    \Xv_{V,\cdot}(\lambda \nabla^2 R + \Xv^\top \Lv_h \Xv)^{-1} =& \Lv_{h,V}^{-1} (\Lv_{h,V}^{-1} + \Xv_{V,\cdot}\Yv_h^{-1}\Xv_{V,\cdot}^\top )^{-1}\Xv_{V,\cdot} \Yv_h^{-1} \\
    \Xv_{V,\cdot}(\lambda \nabla^2 R + \Xv^\top \Lv_h
    \Xv)^{-1}\Xv_{V,\cdot}^\top =& \Lv_{h,V}^{-1} - \Lv_{h,V}^{-1}
    (\Lv_{h,V}^{-1} + \Xv_{V,\cdot}\Yv_h^{-1}\Xv_{V,\cdot}^\top)^{-1} \Lv_{h,V}^{-1}.
\end{align*}

By Lemma \ref{lemma:nonsmooth-loss-hessian-converge}, $\Yv_h \rightarrow \Yv :=
\lambda \nabla^2 R + \Xv_{S,\cdot}^\top \Lv_S \Xv_{S,\cdot}$, $\Lv_{h,V}^{-1} \rightarrow
\mathbf{0}$, we have
\begin{align*}
    \Hv_{h, S,S}\Lv_{h,S} \rightarrow& \Xv_{S,\cdot} (\Yv^{-1} -
    \Yv^{-1}\Xv_{V,\cdot}^\top (\Xv_{V,\cdot},
    \Yv^{-1}\Xv_{V,\cdot}^\top )^{-1}\Xv_{V,\cdot} \Yv^{-1})\Xv_{S,\cdot}^\top
    \Lv_S, \\
    \Hv_{h, S,V}\Lv_{h,V} \rightarrow& \Xv_{S,\cdot} \Yv^{-1}\Xv_{V,\cdot}^\top
    (\Xv_{V,\cdot}\Yv^{-1}\Xv_{V,\cdot}^\top )^{-1}, \\
    \Hv_{h, V,S}\Lv_{h,S} \rightarrow& \mathbf{0} \\
    \Hv_{h, V,V}\Lv_{h,V} \rightarrow& \Iv_V.
\end{align*}

This is not enough, however, noticing that in the final formula of the smooth
case, we need $\frac{H_{h, ii}}{1 - L_{h, ii}H_{h, ii}}$ but for $i\in V$,
$1 - L_{h,ii}H_{h,ii}\rightarrow 0$ and $H_{h, ii} \rightarrow 0$. So
further we have
\begin{align*}
    & \Lv_{h,V}(\Iv_V - \Hv_{h, VV}\Lv_{h,V}) \nonumber \\
    =&
    \Lv_{h,V}(\Iv_V - (\Lv_{h,V}^{-1} - \Lv_{h,V}^{-1} (\Lv_{h,V}^{-1} +
    \Xv_{V,\cdot}\Yv_h^{-1}\Xv_{V,\cdot}^\top )^{-1}\Lv_{h,V}^{-1}) \Lv_{h,V}) \nonumber \\
    =&
    (\Lv_{h,V}^{-1} + \Xv_{V,\cdot}\Yv_h^{-1}\Xv_{V,\cdot}^\top )^{-1} \nonumber \\
    \rightarrow&
    (\Xv_{V,\cdot}\Yv^{-1}\Xv_{V,\cdot}^\top )^{-1}.
\end{align*}

As a result, we have
\begin{equation*}
    \frac{H_{h,ii}}{1 - L_{h,ii}H_{h,ii}}
    \rightarrow
    \left\{
        \begin{array}{ll}
            \frac{\xv_i^\top (\Yv^{-1} - \Yv^{-1}\Xv_{V,\cdot}^\top (\Xv_{V,\cdot}
            \Yv^{-1}\Xv_{V,\cdot}^\top )^{-1}\Xv_{V,\cdot} \Yv^{-1})\xv_i}{1 - \xv_i (\Yv^{-1} -
            \Yv^{-1}\Xv_{V,\cdot}^\top (\Xv_{V,\cdot} \Yv^{-1}\Xv_{V,\cdot}^\top )^{-1}\Xv_{V,\cdot}
            \Yv^{-1})\xv_i \ddot{\ell}_i}, & i \in S, \\
            \frac{1}{[(\Xv_{V,\cdot}\Yv^{-1}\Xv_{V,\cdot}^\top )^{-1}]_{ii}}, &
            i \in V. \\
        \end{array}
    \right.
\end{equation*}

For $\dot{\ell}_h(\xv_i^\top\estim{\betav}_h; y_i)$, as $h\rightarrow 0$, Lemma
\ref{lemma:nonsmooth-loss-gradient-converge} implies the limit value the smooth
gradients would converge to. Notice that for $j \in V$, we solve for the
subgradient by applying least square formula to the 1st order optimality
equation. The final results easily follow.
\end{proof}


\subsection{Proof of Lemma \ref{lem:smoothedprox} }\label{sec:appendixprox}

We prove this lemma under a more general setting, since smoothing idea can also be applied to non-separable regularizers. Let $\proxv_R: \mathbb{R}^p \rightarrow \mathbb{R}^p$ denote the proximal operator of a convex function $R: \mathbb{R}^p \rightarrow \mathbb{R}^p$. Let $\phi: \mathbb{R} \rightarrow \mathbb{R}^+ \cup \{0\}$ denote an infinitely many times differentiable and symmetric function whose support is $[-1,1]$. Furthermore, assume that $\phi$ is normalized such that $\int \phi(t)dt =1$. Construct 
\[
\phiv(\uv) = \phi(u_1) \times \phi(u_2) \times \ldots \times \phi(u_p).   
\] 
Using this function we define 
\[
\proxv_R^\alpha(\uv) = \int_{\tv \in \mathbb{R}^p}  \proxv_R(\tv)  \alpha \phiv(\alpha (\uv -\tv)) d\tv. 
\]
Note that for notational simplicity we use $\alpha := \frac{1}{h}$ in our calculations. It is straightforward to see that $\proxv_R^\alpha(\uv)$ is infinitely many times differentiable. In the next two lemmas, we prove the properties mentioned in Lemma \ref{lem:smoothedprox} in a more general setting. 

\begin{lemma}
$\proxv_R^\alpha(\uv)$ is a proximal operator of a convex function. 
\end{lemma}
\begin{proof}
According to Lemma \ref{lem:proxproperties} part 5, if $\proxv_R^\alpha(\uv)$ is non-expansive and is a gradient of a convex function, then it is a proximal operator of a convex function too. We will hence prove that $\proxv_R^\alpha(\uv)$ is non-expansive and is the gradient of a convex function. First, note that 
\begin{align*}
    \big\| \proxv_R^\alpha(\uv) - \proxv_R^\alpha(\vv) \big\|_2
    =&
    \bigg \| \int_{\tv \in \mathbb{R}^p} \proxv_R (\uv - \tv)
    \alpha \phiv(\alpha \tv) d\tv - \int_{\tv \in \mathbb{R}^p}
    \proxv_R (\vv - \tv) \alpha \phiv(\alpha \tv) d\tv \bigg\|_2 \nonumber \\
    =&
    \int_{\tv \in \mathbb{R}^p} \big\| \proxv_R (\uv - \tv) - \proxv_R (\vv -
    \tv) \big\|_2 \alpha \phiv(\alpha \tv) d\tv \nonumber \\
    \leq&
    \|\uv-\vv\|_2 \int_{\tv \in \mathbb{R}^p} \alpha \phiv(\alpha \tv) d\tv
    = \|\uv-\vv\|_2.
\end{align*}

To confirm the fact that $\proxv_h^\alpha$ is the gradient of a convex function, we should prove that for every $\uv, \vv \in \mathbb{R}^p$ and $c \in \mathbb{R}$, $\vv^\top \proxv_h^\alpha(\uv + c \vv)$ is an increasing function of $c$. First note that
\begin{equation*}
   \proxv_R^\alpha(\uv)
    = \int_{\tv \in \mathbb{R}^p} \proxv_R(\tv)\alpha \phiv(\alpha(\uv - \tv)) d\tv
    = \int_{\tv \in \mathbb{R}^p} \proxv_R(\uv - \tv)\alpha \phiv(\alpha\tv) d\tv.
\end{equation*}

For $c_1 > c_2$, we have
\begin{align*}
    & \vv^\top [\proxv_R^\alpha(\uv + c_1\vv) - \proxv_R^\alpha (\uv + c_2\vv)] \nonumber \\
    =& \int_{\tv \in \mathbb{R}^p} \vv^\top [\proxv(\uv + c_1\vv - \tv)
    - \proxv(\uv + c_2\vv - \tv)]\alpha \phiv(\alpha\tv) d\tv \nonumber \\
    \geq& \frac{1}{c_1 - c_2}\int_{\tv \in \mathbb{R}^p} \big\|
    \proxv(\uv + c_1\vv - \tv) - \proxv(\uv + c_2\vv - \tv) \big\|_2^2
    \alpha \phiv(\alpha\tv) d\tv  \geq 0.
\end{align*}

The first inequality follows from the nonexpansiveness of the proximal
operator. This justifies the monotonicity of $\proxv_R^\alpha$ along any
direction $\vv$.
\end{proof}

\begin{lemma}\label{lem:boundproxdif}
The approximation error of $\proxv_R^\alpha(\uv)$ satisfies
\begin{equation*}
    \|\proxv_R^\alpha(\uv) - \proxv(\uv)\|_2
    \leq
    \frac{p}{\alpha}\int_{-1}^{1} |u|\phiv (u)du. 
\end{equation*}
\end{lemma}
\begin{proof}
\begin{eqnarray*}
    \|\proxv_R^\alpha(\uv) - \proxv(\uv)\|_2
    &\leq&
    \int \|\proxv_R(\uv - \tv) - \proxv_R(\uv)\|_2 \alpha \phiv(\alpha
    \tv) d \tv  \nonumber \\
    &\leq&
    \int \|\tv\|_2 \alpha \phiv (\alpha \tv) d\tv \nonumber
    \\
    &\leq&
    \int \|\tv\|_1 \alpha \phiv (\alpha \tv) d \tv
    = \frac{p}{\alpha} \int |u| \phi(u)du
\end{eqnarray*}
We remind the reader that we have used $\alpha := 1/h$ in this proof. 
\end{proof}


\subsection{Proof of Theorem \ref{thm:prox-smoothing}}
\label{append:proof:thm:prox-smoothing}

Suppose that $\estim{\betav}_h$ and $\estim{\betav}$ are all in a compact set for
small enough $h$. Then we do the rest of the proof in two steps.

Step 1: We first prove $\|\estim{\betav}_h - \estim{\betav}\|_2 \rightarrow 0$.
Since $\estim{\betav}_h$ are in a compact set for small enough $h$, for any
subsequence of $\estim{\betav}_h$ there is a convergent subsubsequence. We abuse
notation and still use $\estim{\betav}_h$ for this convergent subsubsequence
and assume it converges to $\estim{\betav}_0$. Then,
\begin{align*}
    & \Big\| \estim{\betav}_0 - \proxv_R(\estim{\betav}_0 - \sum_{j=1}^n
    \xv_j\dot{\ell}( \xv_j^\top\estim{\betav}_0; y_j)) \Big\|_2 \nonumber \\
    \leq &
    \| \estim{\betav}_h - \estim{\betav}_0 \|_2
    +
    \Big\| \proxv_R^h(\estim{\betav}_h - \sum_{j=1}^n
    \xv_j\dot{\ell}( \xv_j^\top\estim{\betav}_h; y_j)) -
    \proxv_R^h(\estim{\betav}_0 - \sum_{j=1}^n
    \xv_j\dot{\ell}( \xv_j^\top\estim{\betav}_0; y_j)) \Big\|_2 \nonumber \\
    & +
    \Big\| \proxv_R^h(\estim{\betav}_0 - \sum_{j=1}^n
    \xv_j\dot{\ell}( \xv_j^\top\estim{\betav}_0; y_j)) -
    \proxv_R(\estim{\betav}_0 - \sum_{j=1}^n
    \xv_j\dot{\ell}( \xv_j^\top\estim{\betav}_0; y_j)) \Big\|_2 \nonumber \\
    \overset{(a)}{\leq} &
    2\| \estim{\betav}_h - \estim{\betav}_0 \|_2
    +
    \sum_{j=1}^n \|\xv_j\|_2 | \dot{\ell}( \xv_j^\top\estim{\betav}_h; y_j) -
    \dot{\ell}( \xv_j^\top\estim{\betav}_0; y_j) | + ph\int |u|\phi(u)du
    \nonumber \\
    \rightarrow &
    0, \;\;\; \text{ as }h \rightarrow 0
\end{align*}
To obtain Inequality (a) we have used non-expansiveness of $\proxv_R^h(\cdot)$ and Lemma \ref{lem:boundproxdif}. 
The last limit is due to the continuity of $\dot{\ell}$. As a result,
$\estim{\betav}_0$ also satisfies the first order condition
\begin{equation*}
    \estim{\betav}_0
    =
    \proxv_R(\estim{\betav}_0 - \sum_{j=1}^n
    \xv_j\dot{\ell}( \xv_j^\top\estim{\betav}_0; y_j)).
\end{equation*}
The uniqueness of $\estim{\betav}$ implies that $\estim{\betav}_0 =
\estim{\betav}$, which indicates $\estim{\betav}_h \rightarrow \estim{\betav}$.\\

Step 2: We prove $\Jv_{h,k} \rightarrow \Jv_{k}$ for $k=1,\ldots, p$. By the
2nd part of Assumption \ref{assump:prox-smoothing}, noticing
$\estim{\betav}_h \rightarrow \estim{\betav}$, we have for small enough $h$,
$\estim{\beta}_{h,k} - \sum_{j}x_{jk}\dot{\ell}(\xv_j^\top\estim{\betav}_h;
y_j)$ falls in either the interior of one of the intervals with form $(v_m +
\dot{r}_-(v_m), v_m + \dot{r}_+(v_m))$ or the interior of their complement. Also, according to part (iv) of Lemma \ref{lem:proxproperties} we have $0 \leq \frac{d}{dt}\prox_r(t) \leq 1$ (whenever the derivative is well-defined). Hence, by the dominated convergence theorem, we have
\begin{align*}
    |J_{h, k} - J_{k}|
    =&
    \Big|\dot{\prox}_r^h(\estim{\beta}_{h,k} -
    \sum_{j}x_{jk}\dot{\ell}(\xv_j^\top\estim{\betav}_h; y_j)) -
    \dot{\prox}_r(\estim{\beta}_k - \sum_j
    x_{jk}\dot{\ell}(\xv_j^\top\estim{\betav}; y_j))\Big|
    \nonumber \\
    \leq&
    \int \Big|\dot{\prox}_r(\estim{\beta}_{h,k} - \sum_{j}x_{jk}
    \dot{\ell}(\xv_j^\top\estim{\betav}_h; y_j) - hu) -
    \dot{\prox}_r(\estim{\beta}_k - \sum_j x_{jk} \dot{\ell}
    (\xv_j^\top\estim{\betav}; y_j))\Big|\phi(u)du \nonumber \\
    \rightarrow&
    0,\;\;\; \text{ as } h \rightarrow 0
\end{align*}

Notice that $J_k = 0$ when $k \notin E$, our conclusion follows.



\subsection{Proof of Theorem \ref{thm:nonsmooth-loss-intercept}}
\label{ssec:proof-of-nonsmooth-loss-intercept}
In this section we prove the ALO formula for models with nonsmooth losses and
intercepts. We start our discussion from the conclusion of Theorem
\ref{thm:nonsmooth-loss-approx}. Recall that $S=\big\{j: \estim{\beta}_0 +
\xv^\top_j \estim{\betav} = v_t, \text{ for some } t \in \{1, \ldots, k\}
\big\}$ and $V = [1,\ldots, n] \backslash S$ where $v_t$'s are the zeroth-order
singular points of the nonsmooth loss function. First, note that when the
intercept is involved, the matrix $\Yv$ takes the following form
\begin{align*}
    \Yv_1
    =&
    \begin{bmatrix}
        0 & \\ & \nabla^2 R(\estim{\betav})
    \end{bmatrix}
    + \begin{bmatrix}
        \bm{1}^\top \\ \Xv_{S,\cdot}^\top
    \end{bmatrix}
    \diag[\{\ddot{\ell} (\estim{\beta}_0 + \xv_j^\top\estim{\betav})\}_{j \in S}]
    [\bm{1}, \Xv_{S,\cdot}] \nonumber \\
    =&
    \begin{bmatrix}
        \sum_{j\in S} \ddot{\ell}(\hat{\beta}_0 + \xv_j^\top\hat{\betav}; y_j)
        & 
        \sum_{j\in S} \ddot{\ell}(\hat{\beta}_0 + \xv_j^\top\hat{\betav}; y_j) \Xv_j^\top \\
        \sum_{j\in S} \ddot{\ell}(\hat{\beta}_0 + \xv_j^\top\hat{\betav}; y_j) \Xv_j
        &
        \Xv_{S,\cdot}^\top \diag[\{\ddot{\ell}(\hat{\beta}_0 + \xv_j^\top\hat{\betav};
        y_j)\}_{j \in S}] \Xv_{S,\cdot} + \nabla^2 R(\estim{\betav})
    \end{bmatrix}
\end{align*}

Since $\ddot{\ell}(\hat{\beta}_0 + \xv_j^\top\hat{\betav}; y_j)$ may be zero
for all $j \in S$ (such as in the case of SVM), we cannot directly apply the
matrix inversion formula to simplify $\Yv_1^{-1}$. Nevertheless we can still
use the smoothing techniques in Section \ref{ssec:nonsmooth-loss} by replacing
$\ddot{\ell}(\hat{\beta}_0 + \xv_j^\top\hat{\betav}; y_j)$ with
$\ddot{\ell}_h(\hat{\beta}_0 + \xv_j^\top\hat{\betav}; y_j)$ and setting $h$
goes to 0. Now take
\begin{equation*}
    \Yv_{1,h}
    =
    \begin{bmatrix}
        \sum_{j\in S} \ddot{\ell}_h(\hat{\beta}_0 + \xv_j^\top\hat{\betav}; y_j)
        & 
        (\sum_{j\in S} \ddot{\ell}_h(\hat{\beta}_0 + \xv_j^\top\hat{\betav}; y_j))\Xv_j^\top \\
        (\sum_{j\in S} \ddot{\ell}_h(\hat{\beta}_0 + \xv_j^\top\hat{\betav}; y_j))\Xv_j
        &
        \Xv_{S,\cdot}^\top \diag[\{\ddot{\ell}_h(\hat{\beta}_0 + \xv_j^\top\hat{\betav};
        y_j)\}_{j\in S}] \Xv_{S,\cdot} + \nabla^2 R(\estim{\betav})
    \end{bmatrix}
\end{equation*}

Let $a_h = \sum_{j\in S} \ddot{\ell}_h(\hat{\beta}_0 +
\xv_j^\top\hat{\betav}; y_j)$, $\bv_h = \sum_{j \in S}
(\ddot{\ell}_h (\hat{\beta}_0 + \xv_j^\top\hat{\betav}; y_j)) \xv_j$,
$\Yv_h = \Xv_{S,\cdot}^\top \diag[\{\ddot{\ell}_h(\hat{\beta}_0 +
\xv_j^\top\hat{\betav}; y_j)\}_{j\in S}] \Xv_{S,\cdot}
+ \nabla^2 R(\estim{\betav})$. Now we have

\begin{align} \label{eq:nonsmooth-loss-intercept1}
    &\bigg([\bm{1}, \Xv_{V,\cdot}] \Yv_{1, h}^{-1}
    \begin{bmatrix}
        \bm{1}^\top \\ \Xv_{V,\cdot}^\top
    \end{bmatrix}\bigg)^{-1} \nonumber \\
    =&
    \Bigg([\bm{1}, \Xv_{V,\cdot}]
    \begin{bmatrix}
        \frac{1}{a_h - \bv_h^\top \Yv_h^{-1} \bv_h} & - \frac{\bv_h^\top \Yv_h^{-1}}{a_h - \bv_h^\top \Yv_h^{-1} \bv_h} \\
        - \frac{\Yv_h^{-1}\bv_h}{a _h- \bv_h^\top \Yv_h^{-1} \bv_h}
        &
        \Yv_h^{-1} + \frac{\Yv_h^{-1}\bv_h\bv_h^\top\Yv_h^{-1}}{a_h - \bv_h^\top \Yv_h^{-1} \bv_h}
    \end{bmatrix}
\begin{bmatrix} \bm{1}^\top \\ \Xv_{V,\cdot}^\top \end{bmatrix} \Bigg)^{-1} \nonumber \\
    =&
    \Big[\Xv_{V,\cdot}\Yv_h^{-1}\Xv_{V,\cdot}^\top + \frac{1}{a_h - \bv_h^\top\Yv_h^{-1}\bv_h}
    \big(\bm{1} - \Xv_{V,\cdot} \Yv_h^{-1} \bv_h\big)
    \big(\bm{1} - \Xv_{V,\cdot} \Yv_h^{-1} \bv_h\big)^\top \Big]^{-1} \nonumber \\
    =&
    \big[\Xv_{V,\cdot}\Yv_h^{-1}\Xv_{V,\cdot}^\top\big]^{-1}
    -
    \frac{\big[\Xv_{V,\cdot}\Yv_h^{-1}\Xv_{V,\cdot}^\top\big]^{-1}
    \big(\bm{1} - \Xv_{V,\cdot} \Yv_h^{-1} \bv_h\big)
    \big(\bm{1} - \Xv_{V,\cdot} \Yv_h^{-1} \bv_h\big)^\top
    \big[\Xv_{V,\cdot}\Yv_h^{-1}\Xv_{V,\cdot}^\top\big]^{-1}}
    {a_h - \bv_h^\top\Yv_h^{-1}\bv_h
    + \big(\bm{1} - \Xv_{V,\cdot} \Yv_h^{-1} \bv_h\big)^\top
    \big[\Xv_{V,\cdot}\Yv_h^{-1}\Xv_{V,\cdot}^\top\big]^{-1}
    \big(\bm{1} - \Xv_{V,\cdot} \Yv_h^{-1} \bv_h\big)} \nonumber \\
    \rightarrow&
    \big[\Xv_{V,\cdot}\Yv^{-1}\Xv_{V,\cdot}^\top\big]^{-1}
    -
    \frac{\big[\Xv_{V,\cdot}\Yv^{-1}\Xv_{V,\cdot}^\top\big]^{-1}
    \big(\bm{1} - \Xv_{V,\cdot} \Yv^{-1} \bv\big)
    \big(\bm{1} - \Xv_{V,\cdot} \Yv^{-1} \bv\big)^\top
    \big[\Xv_{V,\cdot}\Yv^{-1}\Xv_{V,\cdot}^\top\big]^{-1}}
    {a - \bv^\top\Yv^{-1}\bv
    + \big(\bm{1} - \Xv_{V,\cdot} \Yv^{-1} \bv\big)^\top
    \big[\Xv_{V,\cdot}\Yv^{-1}\Xv_{V,\cdot}^\top\big]^{-1}
    \big(\bm{1} - \Xv_{V,\cdot} \Yv^{-1} \bv\big)},
    \quad \text{as } h \rightarrow 0.
\end{align}
where $\Yv = \Xv_{S,\cdot}^\top \diag[\{\ddot{\ell}(\hat{\beta}_0 +
\xv_j^\top\hat{\betav}; y_j)\}_{j\in S}] \Xv_{S,\cdot} + \nabla^2
R(\estim{\betav})$ takes the same form as in Theorem
\ref{thm:nonsmooth-loss-approx}, $a = \sum_{j\in S} \ddot{\ell}(\hat{\beta}_0 +
\xv_j^\top\hat{\betav}; y_j)$, $\bv = \sum_{j \in S} \ddot{\ell}
(\hat{\beta}_0 + \xv_j^\top\hat{\betav}; y_j) \xv_j$, here we use $\ddot{\ellv}_S$ to denote
$\bv_h$ at $h=0$.

Next we look at how does the value of $W_{ii}$ changes where $i \in S$. Note
that $W_{ii}$'s are the limiting value of the diagonals of the following matrix $\Wv_{1, h}$:
\begin{equation*}
    \Wv_{1,h}
    =
    [\bm{1}, \Xv_{S,\cdot}] \Yv_{1,h}^{-1}
    \begin{bmatrix}
        \bm{1}^\top \\ \Xv_{S,\cdot}^\top
    \end{bmatrix}
    - [\bm{1}, \Xv_{S,\cdot}] \Yv_{1,h}^{-1}
    \begin{bmatrix}
        \bm{1}^\top \\ \Xv_{V,\cdot}^\top
    \end{bmatrix}
    \bigg( [\bm{1}, \Xv_{V,\cdot}] \Yv_{1,h}^{-1}
    \begin{bmatrix}
        \bm{1}^\top \\ \Xv_{V,\cdot}^\top
    \end{bmatrix} \bigg)^{-1}
    [\bm{1}, \Xv_{V,\cdot}] \Yv_{1,h}^{-1}
    \begin{bmatrix}
        \bm{1}^\top \\ \Xv_{S,\cdot}^\top
    \end{bmatrix}
\end{equation*}

After pluggin \eqref{eq:nonsmooth-loss-intercept1} in the above equation, and a
few messy simplification steps, we reach to the follow expression for the
limiting value of $\Wv_1$:
\begin{align*}
    \Wv_1 =& \lim_{h \rightarrow 0} \Wv_{1, h} \nonumber \\
    =&
    \Xv_{S,\cdot}\Yv^{-1}\Xv_{S,\cdot}^\top -
    \Xv_{S,\cdot}\Yv^{-1}\Xv_{V,\cdot}^\top
    \big[\Xv_{V,\cdot}\Yv^{-1}\Xv_{V,\cdot}^\top \big]^{-1}
    \Xv_{V,\cdot}\Yv^{-1}\Xv_{S,\cdot}^\top \nonumber \\
    &+
    \frac{\dv \dv^\top}
    {a - \bv^\top\Yv^{-1}\bv + \big(\bm{1} - \Xv_{V,\cdot} \Yv^{-1} \bv\big)^\top
    \big[\Xv_{V,\cdot}\Yv^{-1}\Xv_{V,\cdot}^\top\big]^{-1}
    \big(\bm{1} - \Xv_{V,\cdot} \Yv^{-1} \bv\big)}
\end{align*}
where $\dv = \Xv_{S,\cdot}\Yv^{-1}\Xv_{V,\cdot}^\top \big[\Xv_{V,\cdot}
\Yv^{-1} \Xv_{V,\cdot}^\top \big]^{-1} (\bm{1} - \Xv_{V,\cdot}\Yv^{-1}\bv) -
(\bm{1} - \Xv_{S,\cdot}\Yv^{-1}\bv)$.

Finally for the (sub)gradients $g_{\ell, i}$, everything remains the same,
specifically we have:
\begin{equation*}
    g_{\ell, i} = \dot{\ell}(\estim{\beta}_0 + \xv_i^\top \estim{\betav};
    y_i),
    \quad \text{for } i \in S;
    \qquad
    \gv_{\ell,V} = (\Xv_{V,\cdot}\Xv_{V,\cdot}^\top)^{-1}\Xv_{V,\cdot}
    \Bigg[\nabla R(\estim{\betav}) - \sum_{j \in
    S}\xv_j\dot{\ell}(\estim{\beta}_0 + \xv_j^\top \estim{\betav}; y_j)\Bigg].
\end{equation*}

\subsection{Proof of Nuclear Norm ALO Formula}
\label{append:sec:nuclear-norm}

In this section, we prove Theorem \ref{thm:matrix-nonsmooth-alo}.  Consider the following problem
\begin{equation*}\label{eq:append:matrix-smooth-main}
    \estim{\Bv} = \arg\min_{\Bv} \sum_{j=1}^n \ell(\langle \Xv_j,
    \Bv\rangle; y_j)^2 + \lambda R(\Bv).
\end{equation*}
where $R$ is a unitarily invariant function, which will be
explained and studied in more detail in Section
\ref{append:ssec:unitary-matrix-function}. This section is laid out as follows: in
Section \ref{append:ssec:unitary-matrix-function}, we briefly discuss basic properties
of unitarily invariant functions; In Section \ref{append:ssec:smooth-unitary} we do
ALO for smooth unitarily invariant penalties; In Section
\ref{append:ssec:nuclear-proof} we prove Theorem \ref{thm:matrix-nonsmooth-alo}
where nuclear norm is considered.

\subsubsection{Properties of Unitarily Invariant Functions}
\label{append:ssec:unitary-matrix-function}

Let $\Bv \in \RR^{p_1 \times p_2}$, and consider the SVD of $\Bv$ as $\Bv = \Uv
\diag[\sigmav] \Vv^\top$ with $\Uv \in \mathbb{R}^{p_1 \times p_1}$, $\Vv \in
\mathbb{R}^{p_2 \times p_2}$. We say that a function $R : \RR^{p_1 \times
p_2} \rightarrow \RR$ is unitarily invariant if there exists an absolutely
symmetric function $f : \RR^{\min(p_1, p_2)} \rightarrow \RR$ such that:
\begin{equation*}
    R(\Bv) = f(\sigmav),
\end{equation*}
where we say that $f: \RR^q \rightarrow \RR$ is absolutely symmetric if for
any $\xv \in \RR^q$, any permutation $\tau$ and signs $\bm{\epsilon} \in \{ -1, 1 \}^q$ we have:
\begin{equation*}
    f(x_1, \dotsc, x_q) = f(\epsilon_1 x_{\tau(1)}, \dotsc, \epsilon_q x_{\tau(q)}).
\end{equation*}

The properties of $R$ and $f$ are closely related, and in particular we will make
use of the following lemma relating their convexity, smoothness and derivatives, proved in
\cite{lewis1995convex}.

\begin{lemma}[\cite{lewis1995convex}]\label{lemma:unitary1}
    Let $R(\Bv) = f(\sigmav)$ with $\Bv=\Uv\diag[\sigmav] \Vv^\top$ its SVD.
    There is a one-to-one correspondence between unitarily
    invariant matrix functions $R$ and symmetric functions $f$. Furthermore the
    convexity and/or differentiability of $f$ are equivalent to the convexity
    and/or
    differentiability of $R$ respectively. If $R$ is differentiable, its
    derivative is given by:
    \begin{equation*}
        \nabla R(\Bv) = \Uv \diag[\nabla f(\sigmav)]\Vv^\top.
    \end{equation*}
    When $f$ is not differentiable, a similar result holds with gradient
    replaced by subdifferentials
    \begin{equation*}
        \partial R(\Bv) = \Uv \diag[\partial f(\sigmav)]\Vv^\top.
    \end{equation*}
\end{lemma}
Based on this lemma, we know that as long as $f$ is convex and/or
smooth, the corresponding matrix function will be convex and/or smooth. This enables
us to produce convex and smooth unitarily invariant approximation to non-smooth
unitarily invariant matrix regularizers. In addition to the gradient of the unitarily invariant matrix functions, we
also need their Hessians. The following Theorem  characterizes the hessian for a sub-class of unitarily invariant functions.
\begin{theorem}\label{thm:matrix-twice-diff}
    Consider a unitarily invariant function with form $R(\Bv) =
    \sum_{j=1}^{ \min(p_1, p_2)} f(\sigma_j)$, where $f$ is a
    smooth function on $\mathbb{R}$ and $\Bv=\Uv\diag[\sigmav]\Vv^\top$ is its
    SVD with $\Uv \in \mathbb{R}^{p_1 \times p_1}$, $\Vv \in \mathbb{R}^{p_2
    \times p_2}$. Further assume that all the $\sigma_j$'s
    are different from each other and nonzero. Let $p_3 = \min(p_1, p_2)$, $p_4
    = \max(p_1, p_2)$. Then the Hessian matrix $\nabla^2 R(\Bv) \in
    \mathbb{R}^{p_1p_2 \times p_1p_2}$ takes the following form
    \begin{equation}\label{eq:matrix-hessian-general}
        \nabla^2 R(\Bv)
        =
        \Qv
        \Bigg[\begin{array}{ccc}
                A_1 & 0 & 0 \\
                0 & A_2 & 0 \\
                0 & 0 & A_3
        \end{array}\Bigg]
        \Qv^\top,
    \end{equation}
    where the first block $A_1 \in \mathbb{R}^{p_3 \times p_3}$, is
    diagonal with $A_{1, (ss, ss)} = f''(\sigma_s)$, $1\leq s\leq p_3$. The
    second block $A_2 \in \mathbb{R}^{p_3(p_3 - 1) \times p_3(p_3 - 1)}$ satisfies the following properties: for
    $1 \leq s\neq t \leq p_3$, $A_{2,(st, st)} = A_{2,(ts, ts)} =
    \frac{\sigma_s f'(\sigma_s) - \sigma_t f'(\sigma_t)}{\sigma_s^2 -
    \sigma_t^2}$, $A_{2, (st, ts)} = A_{2, (ts, st)} = - \frac{\sigma_s
    f'(\sigma_t) - \sigma_t f'(\sigma_s)}{\sigma_s^2 - \sigma_t^2}$; The third
    block $A_3 \in \mathbb{R}^{(p_4 - p_3)p_3 \times (p_4 - p_3)p_3}$ satisfies $A_{3,
    (st, st)} = \frac{f'(\sigma_t)}{\sigma_t}$ for $1 \leq t \leq p_3 < s \leq
    p_4$. Except for these specified locations, all other components of $A_1,
    A_2, A_3$ are zero. $\Qv$ is an orthogonal matrix with $\Qv_{\cdot, st} =
    \vecop(\uv_s\vv_t^\top)$ where $\uv_s$, $\vv_t$ are the $s$\tsup{th}
    column of $\Uv$ and $t$\tsup{th} column of $\Vv$ respectively. $\vecop(\cdot)$
    denotes the vectorization operator, which aligns all the components of a
    matrix into a long vector.
\end{theorem}
\begin{remark}
    Since here we are talking about the Hessian matrix of functions on matrix
    space, we treat them as vectors. The correspondence between each block in
    \eqref{eq:matrix-hessian-general} and the components of the original
    matrix $\Bv$ are exhibited in Figure \ref{fig:matrix-general}.
    \begin{figure}[!htb]
        \begin{center}
            \begin{tikzpicture}[scale=1.0]
                \path [fill=cyan] (0, 0) rectangle (2, 1);
                \path [fill=green] (0, 1) rectangle (2, 3);
                \path [fill=orange] (0, 3) -- (0.1, 3) -- (2, 1.1) -- (2, 1) --
                (1.9, 1) -- (0, 2.9) -- (0, 3);
                \draw (0, 0) rectangle (2, 3);
                \draw [dashed, thick] (0, 1) -- (2, 1);
                \draw [dashed, thick] (0, 3) -- (2, 1);
                
                \draw [fill] (0.75, 2.25) circle [radius=0.04];
                \node [right] at (0.75, 2.25) {\tiny $(s_1, s_1)$};
                \draw [fill] (0.3, 2) circle [radius=0.04];
                \node [below] at (0.4, 2) {\tiny $(s_2, t_2)$};
                \draw [fill] (1, 2.7) circle [radius=0.04];
                \node [right] at (1, 2.7) {\tiny $(t_2, s_2)$};
                \draw [fill] (1.5, 0.5) circle [radius=0.04];
                \node [left] at (1.5, 0.5) {\tiny $(s_3, t_3)$};
                
                \draw [ultra thick, ->] (2.5, 1.5) -- (3.8, 1.5);
                
                \path [fill=lightgray] (5, -0.5) rectangle (9, 3.5);
                \path [fill=cyan] (8, 0.5) rectangle (9, -0.5);
                \path [fill=green] (6, 2.5) rectangle (8, 0.5);
                \path [fill=orange] (5, 3.5) rectangle (6, 2.5);
                \draw [dotted, thick] (5, 2.5) -- (8, 2.5);
                \draw [dotted, thick] (6, 0.5) -- (9, 0.5);
                \draw [dotted, thick] (6, 0.5) -- (6, 3.5);
                \draw [dotted, thick] (8, -0.5) -- (8, 2.5);
                \draw [dashed, thick] (5, 3.5) -- (9, -0.5);
                \draw (5, -0.5) rectangle (9, 3.5);
                
                \draw [fill] (5.3, 3.2) circle [radius=0.04];
                \draw [dotted, thick] (5, 3.2) -- (5.3, 3.2);
                \draw [dotted, thick] (5.3, 3.5) -- (5.3, 3.2);
                \node [left] at (5, 3.2) {\tiny $(s_1, s_1)$};
                \node [right, rotate=90] at (5.3, 3.5) {\tiny $(s_1, s_1)$};
                
                \draw [fill] (6.5, 2.0) circle [radius=0.04];
                \draw [fill] (7.4, 1.1) circle [radius=0.04];
                \draw [fill] (7.4, 2.0) circle [radius=0.04];
                \draw [fill] (6.5, 1.1) circle [radius=0.04];
                \draw [dotted, thick] (5, 2.0) -- (7.4, 2.0);
                \draw [dotted, thick] (5, 1.1) -- (7.4, 1.1);
                \draw [dotted, thick] (6.5, 3.5) -- (6.5, 1.1);
                \draw [dotted, thick] (7.4, 3.5) -- (7.4, 1.1);
                \node [left] at (5, 2.0) {\tiny $(s_2, t_2)$};
                \node [right, rotate=90] at (6.5, 3.5) {\tiny $(s_2, t_2)$};
                \node [left] at (5, 1.1) {\tiny $(t_2, s_2)$};
                \node [right, rotate=90] at (7.4, 3.5) {\tiny $(t_2, s_2)$};
                
                \draw [fill] (8.6, -0.1) circle [radius=0.04];
                \draw [dotted, thick] (5, -0.1) -- (8.6, -0.1);
                \draw [dotted, thick] (8.6, 3.5) -- (8.6, -0.1);
                \node [left] at (5, -0.1) {\tiny $(s_3, t_3)$};
                \node [right, rotate=90] at (8.6, 3.5) {\tiny $(s_3, t_3)$};
                
                \node [below] at (5.3, 3.2) {\scriptsize $a$};
                \node [below left] at (6.5, 2.0) {\scriptsize $b$};
                \node [right] at (7.4, 1.1) {\scriptsize $b$};
                \node [right] at (7.4, 2.0) {\scriptsize $c$};
                \node [below] at (6.5, 1.1) {\scriptsize $c$};
                \node [below] at (8.6, -0.1) {\scriptsize $d$};

                \draw [thin, ->] (9.6, 3) -- (6.1, 3);
                \draw [thin, ->] (9.6, 1.5) -- (8.1, 1.5);
                \draw [thin, ->] (9.6, 0) -- (9.1, 0);
                \node [right] at (9.6, 3) {$A_1$};
                \node [right] at (9.6, 1.5) {$A_2$};
                \node [right] at (9.6, 0) {$A_3$};

                \node [below] at (1, -0.5) {$\Uv^\top\Bv\Vv=\diag[\sigmav]$};
                \node [below] at (7, -0.5) {$\Qv^\top\nabla^2 R(\Bv)\Qv$};
            \end{tikzpicture}
        \end{center}
        \caption{An illustration of the correspondence between the structure of
            the original matrix and the structure of the Hessian matrix of $R$.
            As we have mentioned in Theorem \ref{thm:matrix-twice-diff},
            $a=f''(\sigma_{s_1})$,
            $b = \frac{\sigma_{s_2} f'(\sigma_{s_2}) - \sigma_{t_2}
            f'(\sigma_{t_2}) } {\sigma_{s_2}^2 - \sigma_{t_2}^2 }$,
            $c = - \frac{\sigma_{s_2} f'(\sigma_{t_2}) - \sigma_{t_2}
            f'(\sigma_{s_2}) } {\sigma_{s_2}^2 - \sigma_{t_2}^2}$; $d =
            \frac{f'(\sigma_{t_3})}{\sigma_{t_3}}$.} \label{fig:matrix-general}
    \end{figure}
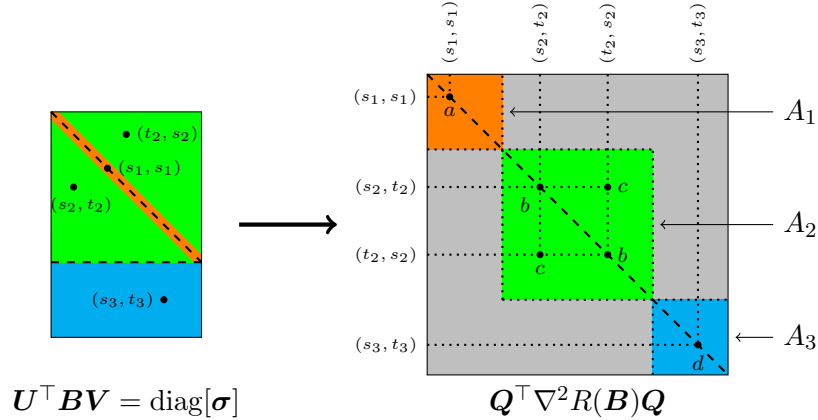
\end{remark}

\begin{proof}
    First by Lemma \ref{lemma:unitary1}, the gradient $\nabla R(\Bv)$ takes the
    following form:
    \begin{equation*}
        \nabla R(\Bv) = \Uv \diag[\{f'(\sigma_j)\}_{j}] \Vv^\top.
    \end{equation*}
    
    In order to find the differential of $\nabla R(\Bv)$, we use the similar
    techniques and notations as the ones used in Lemma IV.2 and Theorem IV.3 of
    \cite{candes2013unbiased}. To simplify our derivation, we assume $p_1 \geq
    p_2$. This does not affect the correctness of our final conclusion.

    We characterize the differential of the gradient as a linear form.
    Specifically, along a certain direction $\Deltav \in \mathbb{R}^{p_1 \times
    p_2}$, by Lemma IV.2 in
    \cite{candes2013unbiased}, we have
    \begin{equation}\label{eq:matrix-svd-differential}
        d\Uv[\Deltav] = \Uv\Omegav_{\Uv}[\Deltav]
        ,\quad
        d\Vv[\Deltav] = \Vv\Omegav_{\Vv}[\Deltav]^{\top}
        ,\quad
        d\sigma_s[\Deltav] = [\Uv^\top\Deltav\Vv]_{ss}.
    \end{equation}
    where $\Omegav_{\Uv}$ and $\Omegav_{\Vv}$ are assymmetric matrices (thus
    their diagonal values are 0) which can be found by solving the following
    linear system of equations:
    \begin{equation}\label{eq:matrix-svd-differential2}
        \left[
            \begin{array}{c}
                \Omegav_{\Uv, st}[\Delta] \\
                \Omegav_{\Vv, st}[\Delta]
            \end{array}
        \right]
        =
        -\frac{1}{\sigma_s^2 - \sigma_t^2}
        \left[
            \begin{array}{cc}
                \sigma_t & \sigma_s \\
                - \sigma_s & - \sigma_t
            \end{array}
        \right]
        \left[
            \begin{array}{c}
                (\Uv^\top \Deltav \Vv)_{st} \\
                (\Uv^\top \Deltav \Vv)_{ts}
            \end{array}
        \right]
        , \quad
        \text{if } s \neq t, s \leq p_2,
    \end{equation}
    and
    \begin{equation}\label{eq:matrix-svd-differential3}
        \Omegav_{\Uv, st}[\Delta]
        =
        \frac{(\Uv^\top\Deltav\Vv)_{st}}{\sigma_t}
        , \quad
        \text{if } s \neq t, s > p_2.
    \end{equation}

    The differential of $\nabla R(\Bv)$ along a certain direction $\Deltav$ can
    then be calculated through the chain,i.e., 
    \begin{align}\label{eq:matrix-reg-differential}
        &d\nabla R(\Bv)[\Deltav] \nonumber \\
        =&
        d\Uv[\Deltav]\diag[\{f'(\sigma_j)\}_{j}]\Vv^\top
        +
        \Uv\diag[\{f''(\sigma_j)d\sigma_j[\Deltav]\}_{j}]\Vv^\top
        +
        \Uv\diag[\{f'(\sigma_j)\}_{j}]d\Vv[\Deltav]^\top \nonumber \\
        =&
        \Uv\big(\Omegav_{\Uv}[\Deltav]\diag[\{f'(\sigma_j)\}_{j}]
        +
        \diag[\{f''(\sigma_j)d\sigma_j[\Deltav]\}_{j}]
        +
        \diag[\{f'(\sigma_j)\}_{j}]\Omegav_{\Vv}[\Deltav] \big)\Vv^\top.
    \end{align}

    In the original formula obtained from the primal approach, the Hessian is
    calculated under the canonical bases $\{\Ev_{st}\}_{s, t}$.\footnote{$\Ev_{st}$ is defined as a $p_1 \times p_2$ matrix with all of
    its components being 0 except the $(s, t)$ location being 1.} In order to
    simplify the calculation of the Hessian, we instead use the orthonormal bases
    $\{\uv_s\vv_t^\top\}_{s,t}$, and then transform back to
    $\{\Ev_{st}\}_{s, t}$. The $(kl, st)$ location of the Hessian matrix under $\{\uv_s\vv_t\}_{s, t}$
    bases can be calculated by
    \begin{equation}\label{eq:hess-component}
        \langle \uv_k\vv_l^\top, d\nabla R(\Bv)[\uv_s\vv_t^\top] \rangle.
    \end{equation}

    Plugging equation \eqref{eq:matrix-reg-differential}
    into \eqref{eq:hess-component} we obtain that
    \begin{align*}
        & \langle \uv_k\vv_l, d\nabla R(\Bv)[\uv_s\vv_t^\top] \rangle \nonumber
        \\
        =&
        \langle \Ev_{kl},
        \Omegav_{\Uv}[\uv_s\vv_t^\top]\diag[\{f'(\sigma_j)\}_j]
        +
        \diag[\{f''(\sigma_j)d\sigma_j[\uv_s\vv_t^\top]\}_{j}]
        +
        \diag[\{f'(\sigma_j)\}_j]\Omegav_{\Vv}[\uv_s\vv_t^\top] \rangle \nonumber \\
        =&
        \left\{
            \begin{array}{ll}
                f''(\sigma_t)d\sigma_t[\uv_t\vv_t^\top], & s = t = k = l,
                \\
                \Omegav_{\Uv, kl}[\uv_s\vv_t^\top]f'(\sigma_l)
                +
                f'(\sigma_k)\Omegav_{\Vv, kl}[\uv_s\vv_t^\top], & k \neq l, k
                \leq p_2, \\
                \Omegav_{\Uv, kl}[\uv_s\vv_t^\top]f'(\sigma_l), & 1\leq l \leq
                p_2 < k \leq p_1.
            \end{array}
        \right.
    \end{align*}

    By \eqref{eq:matrix-svd-differential}, we have $d\sigma_j[\uv_s\vv_t^\top]
    = [\Ev_{st}]_{jj} = \delta_{sj}\delta_{tj}$. In addition, $(\Uv^\top
    \uv_s\vv_t^\top\Vv^\top)_{kl} = (\Ev_{st})_{kl} = \delta_{sk}\delta_{tl}$,
    $(\Uv^\top \uv_s\vv_t^\top\Vv^\top)_{lk} = (\Ev_{st})_{lk} =
    \delta_{sl}\delta_{tk}$. Hence by \eqref{eq:matrix-svd-differential2} and
    \eqref{eq:matrix-svd-differential3}, we have that
    \begin{equation*}
        \Omegav_{\Uv, kl}[\uv_s\vv_t^\top]
        =
        - \frac{\delta_{sk}\delta_{tl}\sigma_l +
        \delta_{sl}\delta_{tk}\sigma_k}{\sigma_k^2 - \sigma_l^2}
        ,\quad
        \Omegav_{\Vv, kl}[\uv_s\vv_t^\top]
        =
        \frac{\delta_{sk}\delta_{tl}\sigma_k +
        \delta_{sl}\delta_{tk}\sigma_l}{\sigma_k^2 - \sigma_l^2}
        , \quad
        \text{if } s \neq t, s \leq p_2,
    \end{equation*}
    and
    \begin{equation*}
        \Omegav_{\Uv, kl}[\uv_s\vv_t^\top]
        =
        \frac{\delta_{sk}\delta_{tl}}{\sigma_l}
        , \quad
        \text{if } s \neq t, s > p_2.
    \end{equation*}

    Based on all these, we can obtain that
    \begin{equation*}
        \langle \uv_k\vv_l, d\nabla R(\Bv)[\uv_s\vv_t^\top] \rangle
        =
        \left\{
            \begin{array}{ll}
                f''(\sigma_t), & s = t = k = l, \\
                \frac{\sigma_s f'(\sigma_s) - \sigma_tf'(\sigma_t)}{\sigma_s^2
                - \sigma_t^2}, & s \neq t, s \leq p_2, (k, l) = (s, t), \\
                - \frac{\sigma_s f'(\sigma_t) - \sigma_tf'(\sigma_s)}{\sigma_s^2
                - \sigma_t^2}, & s \neq t, s \leq p_2, (k, l) = (t, s), \\
                \frac{f'(\sigma_t)}{\sigma_t}, & s \neq j, s > p_2,
                    (k,l)=(s,t), \\
                0, & \text{otherwise.}
            \end{array}
        \right.
    \end{equation*}

    Notice that we obtained the above expressions under the orthonormal bases
    $\{\uv_s\vv_t^\top\}_{s, t}$. In order to get the Hessian form under the
    canonical bases $\{\Ev_{st}\}_{s, t}$, let $\Qv \in \mathbb{R}^{p_1p_2 \times
    p_1p_2}$, with each column $\Qv_{\cdot, st} = \vecop(\uv_s\vv_t^\top)$.
    Denote the matrix form under the canonical bases by $\nabla^2 R(\Bv)$ and that
    under $\{\uv_s\vv_t^\top\}_{s,t}$ by $\widetilde{\nabla^2 R(\Bv)}$. We then have
    that
    \begin{equation*}
        \nabla^2 R(\Bv)
        =
        \Qv \widetilde{\nabla^2 R(\Bv)} \Qv^\top.
    \end{equation*}
    
    This completes our proof.
\end{proof}

\subsubsection{ALO for Smooth Unitarily Invariant Penalties}
\label{append:ssec:smooth-unitary}

In the following two sections, we discuss ALO formula for unitarily invariant
regularizer $R$ of the form:
\begin{equation*}\label{eq:separable-matrix-norm}
    R(\Bv) = \sum_{j=1}^{\min(p_1, p_2)} r(\sigma_j),
\end{equation*}
where $r$ is a convex and even scalar function. The nuclear norm, Frobenius and numerous
other matrix norms all fall in this category. In this section, we assume that $r$
is a twice differentiable function. In the next section, we consider
the case of the nuclear norm, where $r$ is nonsmooth.

Consider the matrix regression problem:
\begin{equation*}
    \estim{\Bv}
    =
    \arg\min_\Bv \sum_{j=1}^n \ell(\langle \Xv_j, \Bv \rangle; y_j) + \lambda
    R(\Bv).
\end{equation*}
Let $\estim{\Bv}=\estim{\Uv}\diag[\estim{\sigmav}]\estim{\Vv}^\top$.
By plugging the Hessian formula from Theorem \ref{thm:matrix-twice-diff}
in \eqref{eq:primal-alo-smooth} and \eqref{eq:smooth-H},
we have the following ALO formula:
\begin{equation}\label{eq:matrix-general-alo}
    \langle \Xv_i, \surrogi{\Bv} \rangle
    =
    \langle \Xv_i, \estim{\Bv} \rangle
    +
    \frac{H_{ii}}{1 - H_{ii}\ddot{\ell}(\langle \Xv_i, \estim{\Bv}
    \rangle; y_i)}\dot{\ell}(\langle \Xv_i, \estim{\Bv}\rangle; y_i),
\end{equation}
where
\begin{equation*}
    \Hv
    :=
    \cb{\tilde{X}}\Big[\cb{\tilde{X}}^\top \diag[\ddot{\ell}(\langle \Xv_j,
    \estim{\Bv}\rangle; y_j)]\cb{\tilde{X}} + \lambda \Qv \cb{G} \Qv^\top\Big]^{-1}
    \cb{\tilde{X}}^\top,
\end{equation*}
with the matrix $\cb{\tilde{X}} \in \mathbb{R}^{n \times p_1p_2}$, $\cb{G} \in
\mathbb{R}^{p_1p_2 \times p_1p_2}$. Each row $\cb{\tilde{X}}_{j,\cdot} =
\vecop(\Xv_j)$. $\cb{G}$ is defined by
\begin{equation}\label{eq:matrix-reg-formula-G}
    \cb{G}_{kl, st}
    =
    \left\{
        \begin{array}{ll}
            r''(\estim{\sigma}_t), & s = t = k = l, \\
            \frac{\estim{\sigma}_s r'(\estim{\sigma}_s) - \estim{\sigma}_tr'(\estim{\sigma}_t)}
            {\estim{\sigma}_s^2 - \estim{\sigma}_t^2},
            & i \neq t, s \leq p_2, (k, l) = (s, t), \\
            - \frac{\estim{\sigma}_s r'(\estim{\sigma}_t) - \estim{\sigma}_tr'(\estim{\sigma}_s)}
            {\estim{\sigma}_s^2 - \estim{\sigma}_t^2},
            & s \neq t, s \leq p_2, (k, l) = (t, s), \\
            \frac{r'(\estim{\sigma}_t)}{\estim{\sigma}_t},
            & s \neq t, s > p_2, (k,l)=(s,t), \\
            0, & \text{otherwise.}
        \end{array}
    \right.
\end{equation}

Note that $[\cb{\tilde{X}}\Qv]_{j, st} = \langle \Xv_j,
\estim{\uv}_s\estim{\vv}_t^\top \rangle = \estim{\uv}_s^\top \Xv_j
\estim{\vv}_t$, we have $[\cb{\tilde{X}}\Qv]_{j,\cdot} =
\vecop(\estim{\Uv}^\top\Xv_j\estim{\Vv})$. Let $\cb{X}=\cb{\tilde{X}}\Qv$. This
gives us the following nicer form of the $\Hv$ matrix:
\begin{equation*}
    \Hv
    :=
    \cb{X}\Big[\cb{X}^\top \diag[\ddot{\ell}(\langle \Xv_j,
    \estim{\Bv}\rangle; y_j)]\cb{X} + \lambda \cb{G} \Big]^{-1}
    \cb{X}^\top.
\end{equation*}

\subsubsection{Proof of Theorem \ref{thm:matrix-nonsmooth-alo}:
ALO for Nuclear Norm}
\label{append:ssec:nuclear-proof}

For the nuclear norm, we have:
\begin{equation*}
    \ell(u; y) = \frac{1}{2}(u - y)^2
    ,\quad
    R(\Bv) = \sum_{j=1}^{\min(p_1, p_2)} \sigma_j.
\end{equation*}

Let $P(\Bv) = \frac{1}{2}\sum_{j=1}^n (y_j - \langle \Xv_j, \Bv \rangle)^2 +
\lambda\|\Bv\|_*$ denote the primal objective. For the full data optimizer
$\estim{\Bv}$ with SVD $\estim{\Bv} = \estim{\Uv} \diag[\estim{\sigmav}]
\estim{\Vv}$, let $m=\rank(\estim{\Bv})$, the number of nonzero
$\estim{\sigma}_j$'s. Furthermore, suppose that we have
the following assumption on the full data solution $\estim{\Bv}$.

\begin{assumption}
    Let $\estim{\Bv}$ be the full-data minimizer, and let
    $\estim{\Bv} = \estim{\Uv} \diag[\estim{\sigmav}] \estim{\Vv}^\top$
    be its SVD.
    \begin{enumerate}
    \item
        $\estim{\Bv}$ is the unique optimizer of the nuclear norm minimization
        problem,

    \item
        For all $j$ such that $\estim{\sigma}_j = 0$, the subgradient
        $g_r[\estim{\sigma}_j]$ at $\estim{\sigma}_j$ satisfies
        $g_r[\estim{\sigma}_j] < 1$.
    \end{enumerate}
\end{assumption}
Note that the first assumption often holds in practice. The discussion of the second assumption is similar to the discussion of part (iii) of Assumption \ref{assum:nonsmooth-reg-assump1} and is hence skipped. Since the nuclear norm is nonsmooth, we consider a smoothed version of it. For a matrix and
its SVD $\Bv = \Uv\diag[\sigmav]\Vv^\top$, and a smoothing parameter $\epsilon > 0$,
 define the following smoothed version of nuclear norm as
\begin{equation*}
    R_\epsilon(\Bv) = \sum_{j=1}^{\min(p_1, p_2)} r_{\epsilon}(\sigma_j), \text{ where }
    r_{\epsilon}(x) = \sqrt{x^2 + \epsilon^2}.
\end{equation*}

Let $P_\epsilon(\Bv) = \frac{1}{2}\sum_{j=1}^n (y_j - \langle \Xv_j,
\Bv \rangle)^2 + \lambda R_\epsilon(\Bv)$ denote the smoothed primal objective, and
let $\estim{\Bv}_\epsilon$ be the minimizer of $P_\epsilon$.
Note that instead of using the general kernel smoothing strategy we mentioned
in the previous section, in this specific case we consider this
choice $R_\epsilon$ for technical convenience. There are no essential
differences between the two smoothing schemes. Finally, let $r(x) = \abs{x}$

Lemma \ref{lemma:unitary1} guarantees the smoothness and convexity of the function
$R_\epsilon$. Additionally, $r_\epsilon$ satisfies several desirable properties:
\begin{enumerate}
    \item
        $\dot{r}_\epsilon(x) = \frac{x}{\sqrt{x^2 + \epsilon^2}}$,
        $\ddot{r}_\epsilon(x) = \frac{\epsilon^2}{(x^2 +
        \epsilon^2)^{\frac{3}{2}}}$;
    \item
        $r(x) < r_\epsilon(x) < r(x) + \epsilon$. 
\end{enumerate}
In particular, we note that the second property implies that
$\sup_x |r(x) - r_\epsilon(x)| \leq \epsilon$ and that
$\sup_{\Bv} |R(\Bv) - R_\epsilon(\Bv)| \leq \epsilon \min(p_1, p_2)$. We now go through a similar strategy as the one presented in Section
\ref{append:ssec:primal-nonsmooth-reg} to obtain the limiting $\alo$ formula as $\epsilon
\rightarrow 0$.

\paragraph{Convergence of the optimizer ($\estim{\Bv}_\epsilon \rightarrow \estim{\Bv}$)}
By definition of $\estim{\Bv}$ as the minimizer of the primal objective, we have
\begin{equation*}
\lambda \|\estim{\Bv}\|_*
\leq \frac{1}{2}\sum_j (y_j - \langle \Xv_j, \estim{\Bv} \rangle)^2 + \lambda\|\estim{\Bv}\|_*
\leq \frac{1}{2}\|\yv\|_2^2.
\end{equation*}
Similarly, we have 
\begin{align*}
    \lambda\norm{\estim{\Bv}_\epsilon}_*
    &\leq \lambda R(\estim{\Bv}_\epsilon)
    \leq \lambda R_\epsilon(\estim{\Bv}_\epsilon) + \lambda\epsilon\min(p_1,
    p_2)\nonumber \\
&\leq \frac{1}{2}\sum_j (y_j - \langle \Xv_j, \estim{\Bv}_\epsilon \rangle)^2
    + \lambda R_\epsilon(\estim{\Bv}_\epsilon) + \lambda \epsilon \min(p_1, p_2) \nonumber \\
&\leq \frac{1}{2}\|\yv\|_2^2 + \lambda \epsilon \min(p_1, p_2).
\end{align*}
Thus, for all $\epsilon \leq 1$ both $\estim{\Bv}$ and $\estim{\Bv}_\epsilon$ are contained
in a compact set given by $\lambda\|\Bv\|_* \leq \frac{1}{2}\|\yv\|_2^2 +
\lambda\min(p_1, p_2)$. In particular, any subsequence of $\estim{\Bv}_\epsilon$ contains a convergent
sub-subsequence, let us abuse notations and still use $\estim{\Bv}_\epsilon$ for
this convergent sub-subsequence. The uniform bound between $R$ and $R_\epsilon$
implies that:
\begin{equation*}
    P(\lim_{\epsilon\rightarrow 0}\estim{\Bv}_\epsilon)
    =\lim_{\epsilon\rightarrow 0}P(\estim{\Bv}_\epsilon)
    =\lim_{\epsilon \rightarrow 0} P_\epsilon(\estim{\Bv}_\epsilon)
    \leq \lim_{\epsilon \rightarrow 0} P_\epsilon(\estim{\Bv})
    = P(\estim{\Bv}).
\end{equation*}
By the uniqueness of the optimizer $\estim{\Bv}$, we have
\begin{equation*}
    \lim_{\epsilon \rightarrow 0} \estim{\Bv}_\epsilon = \estim{\Bv}.
\end{equation*}
This is true for all such subsequences, which confirms the full sequence of $\estim{\Bv}_\epsilon$ converges to $\estim{\Bv}_\epsilon$ as $\epsilon \rightarrow 0$. 

\paragraph{Convergence of the gradient ($\nabla R_\epsilon(\estim{\Bv}_\epsilon) \rightarrow g_{\|\cdot\|_*}
(\estim{\Bv})$)}
Let $g_{\|\cdot\|_*}$ denote the subgradient of the nuclear norm $\|\cdot\|_*$ in
the first order optimality condition of $\estim{\Bv}$.
By the continuity of $\dot{\ell}$ and the first order condition, we have:
\begin{equation}\label{eq:nuclear:frobenius-bound}
    \big\| g_{\|\cdot\|_*}(\estim{\Bv}) - \nabla
    R_\epsilon(\estim{\Bv}_{\epsilon}) \big\|_F
    =
    \Bigg\| \sum_{j=1}^n \langle \Xv_j, \estim{\Bv} - \estim{\Bv}_\epsilon
    \rangle \Xv_j \Bigg\|_F \rightarrow 0.
\end{equation}

Let $\estim{\Bv}_\epsilon = \estim{\Uv}_\epsilon \diag[\estim{\sigmav}_\epsilon] \estim{\Vv}_\epsilon$
denote the SVD of $\estim{\Bv}_\epsilon$. By Lemma \ref{lemma:unitary1} we have:
\begin{align*}
    g_{\|\cdot\|_*} (\estim{\Bv})
    &=
    \estim{\Uv} \diag(\{g_r[\estim{\sigma}_j]\}_{j}) \estim{\Vv}^\top, \nonumber \\
    \nabla R_\epsilon(\estim{\Bv}^{\epsilon}) &= \estim{\Uv}_\epsilon
    \diag(\{\dot{r}_\epsilon(\estim{\sigma}_{\epsilon,j})\}_{j})
    \estim{\Vv}_\epsilon^\top.
\end{align*}
where $g_r[x] = 1$ if $x > 0$ and $0 \leq g_r[x] \leq 1$ if $x = 0$. We wish to translate the limit in matrix norm
\eqref{eq:nuclear:frobenius-bound} to a limit on their singular
values. In order to do this, we use the following lemma from Weyl
\cite{weyl1912das} or Mirsky \cite{mirsky1960symmetric}. We note that our
conclusion may follow from either, although we include both for completeness.

\begin{lemma}[\cite{weyl1912das},\cite{mirsky1960symmetric}]
    \label{lemma:sv-control}
    Let $A$ and $B$ be two rectangular matrices of the same shape.
    Let $\sigma_j$ denote the $j$\textsuperscript{th} largest eigenvalue, then we have that
    for all $j$:
    \begin{gather*}
        \abs{\sigma_j(A) - \sigma_j(B)} \leq \norm{A - B}_2, \nonumber \\
        \sqrt{\sum_j (\sigma_j(A) - \sigma_j(B))^2} \leq \norm{A - B}_F.
    \end{gather*}
\end{lemma}

By Lemma \ref{lemma:sv-control}, we have that $\estim{\sigma}_{\epsilon, j} \rightarrow
\estim{\sigma}_j$ and $\frac{\estim{\sigma}_{\epsilon,j}}{\sqrt{\estim{\sigma}_{\epsilon, j}^2 +
\epsilon^2}} \rightarrow g_r[\estim{\sigma}_j]$ as $\epsilon \rightarrow 0$. Additionally,
by the assumption $g_r[\estim{\sigma}_j] < 1$ if $\estim{\sigma}_j=0$, we have that:
\begin{equation}\label{eq:nuclear-singular-asymp-behavior}
    \frac{\estim{\sigma}_{\epsilon, j}}{\epsilon}
    \rightarrow
    \begin{cases}
        +\infty, & \text{if } \estim{\sigma}_j > 0, \\
        < +\infty, & \text{if } \estim{\sigma}_j = 0. \\
    \end{cases}
\end{equation}

This further implies the matrices $\cb{G}_{\epsilon}$
defined as in \eqref{eq:matrix-reg-formula-G} for $R_\epsilon$ satisifies:
\begin{equation}\label{eq:nuclear-reg-formula-G}
    \lim_{\epsilon \rightarrow 0}
    \cb{G}_{\epsilon, kl, ij}
    =
    \left\{
        \begin{array}{ll}
            0, & s = t = k = l \leq m, \\
            \infty, & s = t = k = l > m, \\
            \frac{1}{\estim{\sigma}_s + \estim{\sigma}_t}, & 1 \leq s \neq t \leq m,(k,l)=(s,t), \\
            \frac{1}{\estim{\sigma}_s}, & 1 \leq s \leq m < t \leq p_2, (k, l) = (s, t), \\
            \frac{1}{\estim{\sigma}_t}, & 1 \leq t \leq m < s \leq p_2, (k, l) = (s, t), \\
            -\frac{1}{\estim{\sigma}_s + \estim{\sigma}_t}, & 1 \leq s \neq t \leq m, (k,l)=(t,s), \\
            -\frac{g_r[\estim{\sigma}_t]}{\estim{\sigma}_s}, & 1 \leq s \leq m < t \leq p_2,
            (k, l) = (t, s), \\
            -\frac{g_r[\estim{\sigma}_s]}{\estim{\sigma}_t}, & 1 \leq t \leq m < s \leq p_2,
            (k, l) = (t, s), \\
            \frac{1}{\estim{\sigma}_t}, & 1 \leq t \leq m \leq p_2 < s \leq p_1, (k,l)=(s,t), \\
            \infty, & m < t \leq p_2 < s \leq p_1, (k,l)=(s,t), \\
            0, & \text{otherwise.}
        \end{array}
    \right.
\end{equation}

By inspecting the indices in  \eqref{eq:nuclear-reg-formula-G} we note that two index sets are missing:
\begin{enumerate}
 \item $m < s \neq t \leq p_2$, $(k, l) = (s, t)$.
 \item $m < s \neq t \leq p_2$, $(k, l) =
(t, s)$.  
\end{enumerate}
We need to process these blocks separately. We will show that the
inverse of the corresponding blocks in $\cb{G}_{\epsilon}$ converges to
0. As a result, according to Lemma
\ref{lemma:woodbury-block} we can ignore these two parts. Each $2\times 2$ sub-matrix within these two blocks in $\cb{G}_{\epsilon}$ has the
form
\begin{equation*}
    \frac{1}{\estim{\sigma}_{\epsilon, s}^2 - \estim{\sigma}_{\epsilon, t}^2}
    \left[
        \begin{array}{cc}
            \estim{\sigma}_{\epsilon, s}\dot{r}_\epsilon(\estim{\sigma}_{\epsilon, s}) -
            \estim{\sigma}_{\epsilon, t} \dot{r}_\epsilon(\estim{\sigma}_{\epsilon, t})
            &
            - \estim{\sigma}_{\epsilon, s}\dot{r}_\epsilon(\estim{\sigma}_{\epsilon, t}) +
            \estim{\sigma}_{\epsilon, t} \dot{r}_\epsilon(\estim{\sigma}_{\epsilon, s}) \\
            - \estim{\sigma}_{\epsilon, s}\dot{r}_\epsilon(\estim{\sigma}_{\epsilon, t}) +
            \estim{\sigma}_{\epsilon, t} \dot{r}_\epsilon(\estim{\sigma}_{\epsilon, s})
            &
            \estim{\sigma}_{\epsilon, s}\dot{r}_\epsilon(\estim{\sigma}_{\epsilon, s}) -
            \estim{\sigma}_{\epsilon, t} \dot{r}_\epsilon(\estim{\sigma}_{\epsilon, t})
        \end{array}
    \right].
\end{equation*}
It is straightforward to verify that the inverse of the above matrix takes the following
form
\begin{equation}\label{eq:nuclear-inverse-sub22}
\frac{1}{\dot{r}^2(\estim{\sigma}_{\epsilon, s}) - \dot{r}^2(\estim{\sigma}_{\epsilon, t})}
    \left[
        \begin{array}{cc}
            \estim{\sigma}_{\epsilon, s}\dot{r}_\epsilon(\estim{\sigma}_{\epsilon, s}) -
            \estim{\sigma}_{\epsilon, t} \dot{r}_\epsilon(\estim{\sigma}_{\epsilon, t})
            &
            \estim{\sigma}_{\epsilon, s}\dot{r}_\epsilon(\estim{\sigma}_{\epsilon, t}) -
            \estim{\sigma}_{\epsilon, t} \dot{r}_\epsilon(\estim{\sigma}_{\epsilon, s}) \\
            \estim{\sigma}_{\epsilon, s}\dot{r}_\epsilon(\estim{\sigma}_{\epsilon, t}) -
            \estim{\sigma}_{\epsilon, t} \dot{r}_\epsilon(\estim{\sigma}_{\epsilon, s})
            &
            \estim{\sigma}_{\epsilon, s}\dot{r}_\epsilon(\estim{\sigma}_{\epsilon, s}) -
            \estim{\sigma}_{\epsilon, t} \dot{r}_\epsilon(\estim{\sigma}_{\epsilon, t})
        \end{array}
    \right].
\end{equation}
For the two distinct component values in the matrix in
\eqref{eq:nuclear-inverse-sub22}, we have that
\begin{equation*}
    \frac{\estim{\sigma}_{\epsilon, s}\dot{r}_\epsilon(\estim{\sigma}_{\epsilon, s}) -
    \estim{\sigma}_{\epsilon, t} \dot{r}_\epsilon(\estim{\sigma}_{\epsilon, t})}
    {\dot{r}^2(\estim{\sigma}_{\epsilon, s}) - \dot{r}^2(\estim{\sigma}_{\epsilon, t})}
    =
    \frac{
        \frac{\estim{\sigma}_{\epsilon, s}^2}{\sqrt{\estim{\sigma}_{\epsilon, s} + \epsilon^2}}
        -
        \frac{\estim{\sigma}_{\epsilon, t}^2}{\sqrt{\estim{\sigma}_{\epsilon, t} + \epsilon^2}}
    }
    {
        \frac{\estim{\sigma}_{\epsilon, s}^2}{\estim{\sigma}_{\epsilon, s} + \epsilon^2}
        -
        \frac{\estim{\sigma}_{\epsilon, t}^2}{\estim{\sigma}_{\epsilon, t} + \epsilon^2}
    }
    =
    \epsilon\frac{\frac{u_{\epsilon, s}}{\sqrt{1 - u_{\epsilon, s}}}
    - \frac{u_{\epsilon, t}}{\sqrt{1 - u_{\epsilon, t}}}}
    {u_{\epsilon, s} - u_{\epsilon, t}}
    =
    \epsilon \frac{1 - \frac{1}{2}\tilde{u}_{\epsilon}}{(1 -
    \tilde{u}_{\epsilon})^{\frac{3}{2}}}
    \rightarrow 0,
\end{equation*}
where we did a change of variable $u=\frac{\estim{\sigma}^2}{\estim{\sigma}^2 + \epsilon^2}$
and $\tilde{u}_{\epsilon}$ is a value between $u_{\epsilon, s}$ and
$u_{\epsilon, t}$ where we apply Taylor expansion to function
$\frac{x}{\sqrt{1-x}}$. The last convergence to 0 is
obtained by noticing that $\lim_{\epsilon \rightarrow 0} u_{\epsilon, s},
\lim_{\epsilon \rightarrow 0} u_{\epsilon, t} \in [0, 1)$ due to
\eqref{eq:nuclear-singular-asymp-behavior}. Similarly, we have the following analysis for the off-diagonal term
\begin{equation*}
    \frac{\estim{\sigma}_{\epsilon, s}\dot{r}_\epsilon(\estim{\sigma}_{\epsilon, t}) -
    \estim{\sigma}_{\epsilon, t} \dot{r}_\epsilon(\estim{\sigma}_{\epsilon, s})}
    {\dot{r}^2(\estim{\sigma}_{\epsilon, s}) - \dot{r}^2(\estim{\sigma}_{\epsilon, t})}
    =
    \frac{
        \frac{\estim{\sigma}_{\epsilon, s}\estim{\sigma}_{\epsilon, t}}
        {\sqrt{\estim{\sigma}_{\epsilon, t} + \epsilon^2}}
        -
        \frac{\estim{\sigma}_{\epsilon, s}\estim{\sigma}_{\epsilon, t}}
        {\sqrt{\estim{\sigma}_{\epsilon, s} + \epsilon^2}}
    }
    {
        \frac{\estim{\sigma}_{\epsilon, s}^2}{\estim{\sigma}_{\epsilon, s} + \epsilon^2}
        -
        \frac{\estim{\sigma}_{\epsilon, t}^2}{\estim{\sigma}_{\epsilon, t} + \epsilon^2}
    }
    =
    \frac{\estim{\sigma}_{\epsilon, s}\estim{\sigma}_{\epsilon, t}}{\epsilon}
    \frac{\sqrt{1 - u_{\epsilon, t}} - \sqrt{1 - u_{\epsilon, s}}}
    {u_{\epsilon, s} - u_{\epsilon, t}}
    =
    \frac{\estim{\sigma}_{\epsilon, s}\estim{\sigma}_{\epsilon, t}}{\epsilon^2}
    \frac{\epsilon}{2\sqrt{1 - \bar{u}_\epsilon}}
    \rightarrow 0,
\end{equation*}
where $\bar{u}_{\epsilon}$ is a value between $u_{\epsilon, s}$ and
$u_{\epsilon, t}$ where we use Taylor expansion to $\sqrt{1-x}$. The last convergence to 0 is
obtained based on the same reason as the previous one. Let $E:=\{kl : k \leq m \text{ or } l \leq m\}$, by Lemma
\ref{lemma:woodbury-block}, we have
\begin{equation*}
    \Hv_\epsilon
    \rightarrow
    \cb{X}_{\cdot,E}\Big[\cb{X}_{\cdot,E}^\top
    \cb{X}_{\cdot,E} + \lambda\cb{G} \Big]^{-1}
    \cb{X}_{\cdot,E}^\top := \Hv,
\end{equation*}
where $\cb{G}$ is defined in \eqref{eq:nuclear-hessian-G}. Finally, we obtain our approximation of leave-$i$-out prediction by substituting the
above formula of $\Hv$ into the general formula \eqref{eq:matrix-general-alo}.

\begin{remark}
    Similar to what we did in Figure \ref{fig:matrix-general}, it is helpful to
    visualize the structure of $\cb{G}$ in correspondence to the blocks of the
    original matrix. Specifically we have Figure \ref{fig:nuclear-general}.
    \begin{figure}[!th]
        \begin{center}
            \begin{tikzpicture}[scale=1.0]
                \path [fill=cyan] (0, 0) rectangle (1.2, 1);
                \path [fill=darkgray] (1.2, 0) rectangle (2, 1.8);
                \path [fill=olive] (0, 1) rectangle (1.2, 1.8);
                \path [fill=olive] (1.2, 1.8) rectangle (2, 3);
                \path [fill=green] (0, 1.8) rectangle (1.2, 3);
                \path [fill=orange] (0, 3) -- (0.1, 3) -- (1.2, 1.9) -- (1.2,
                1.8) -- (1.1, 1.8) -- (0, 2.9) -- (0, 3);
                \draw (0, 0) rectangle (2, 3);
                \draw [dashed, thick] (0, 1) -- (1.2, 1);
                \draw [dashed, thick] (0, 1.8) -- (2, 1.8);
                \draw [dashed, thick] (1.2, 0) -- (1.2, 3);
                \draw [dashed, thick] (0, 3) -- (1.2, 1.8);
                
                \draw [fill] (0.5, 2.5) circle [radius=0.04];
                \node [left] at (0.5, 2.5) {\tiny $(s_1,s_1)$};
                \draw [fill] (0.3, 2) circle [radius=0.04];
                \node [left] at (0.3, 2) {\tiny $(s_2, t_2)$};
                \draw [fill] (1, 2.7) circle [radius=0.04];
                \node [right] at (1, 2.7) {\tiny $(t_2, s_2)$};
                \draw [fill] (0.6, 1.3) circle [radius=0.04];
                \node [below] at (0.6, 1.4) {\tiny $(s_3, t_3)$};
                \draw [fill] (1.7, 2.4) circle [radius=0.04];
                \node [below] at (1.6, 2.4) {\tiny $(t_3, s_3)$};
                \draw [fill] (0.6, 0.6) circle [radius=0.04];
                \node [below] at (0.6, 0.6) {\tiny $(s_4, t_4)$};
                \draw [fill] (1.5, 0.5) circle [radius=0.04];
                \node [above] at (1.6, 0.5) {\tiny $(s_5, t_5)$};
                
                \draw [ultra thick, ->] (2.5, 1.5) -- (3.8, 1.5);
                
                \path [fill=lightgray] (5, -0.5) rectangle (9, 3.5);
                \path [fill=orange] (5, 3.5) rectangle (5.6, 2.9);
                \path [fill=green] (5.6, 2.9) rectangle (6.5, 2);
                \path [fill=olive] (6.5, 2) rectangle (7.5, 1.0);
                \path [fill=cyan] (7.5, 1.0) rectangle (8.4, 0.1);
                \path [fill=darkgray] (8.4, 0.1) rectangle (9, -0.5);
                \path [fill=white!70!lightgray] (8.4, 0.1) rectangle (9, 3.5);
                \path [fill=white!70!lightgray] (8.4, 0.1) rectangle (5, -0.5);
                \draw [dashed, thin] (5, 3.5) -- (8.4, 0.1);

                \draw [dotted, thick] (5, 2.9) -- (6.5, 2.9);
                \draw [dotted, thick] (5.6, 3.5) -- (5.6, 2);
                \draw [dotted, thick] (5.6, 2) -- (7.5, 2);
                \draw [dotted, thick] (6.5, 2.9) -- (6.5, 1.0);
                \draw [dotted, thick] (6.5, 1.0) -- (8.4, 1.0);
                \draw [dotted, thick] (7.5, 2) -- (7.5, 0.1);
                \draw [dotted, thick] (7.5, 0.1) -- (9, 0.1);
                \draw [dotted, thick] (8.4, 1.0) -- (8.4, -0.5);
                \draw [dotted, thick] (5, 0.1) -- (8.4, 0.1);
                \draw [dotted, thick] (8.4, 0.1) -- (8.4, 3.5);
                \draw (5, -0.5) rectangle (9, 3.5);
                
                \draw [fill] (5.3, 3.2) circle [radius=0.04];
                \draw [dotted, thick] (5, 3.2) -- (5.3, 3.2);
                \draw [dotted, thick] (5.3, 3.5) -- (5.3, 3.2);
                \node [left] at (5, 3.2) {\tiny $(s_1, s_1)$};
                \node [right, rotate=90] at (5.3, 3.5) {\tiny $(s_1, s_1)$};
                
                \draw [fill] (5.9, 2.6) circle [radius=0.04];
                \draw [fill] (5.9, 2.2) circle [radius=0.04];
                \draw [fill] (6.3, 2.2) circle [radius=0.04];
                \draw [fill] (6.3, 2.6) circle [radius=0.04];
                \draw [dotted, thick] (5, 2.6) -- (6.3, 2.6);
                \draw [dotted, thick] (5, 2.2) -- (6.3, 2.2);
                \draw [dotted, thick] (5.9, 3.5) -- (5.9, 2.2);
                \draw [dotted, thick] (6.3, 3.5) -- (6.3, 2.2);
                \node [left] at (5, 2.6) {\tiny $(s_2, t_2)$};
                \node [right, rotate=90] at (5.9, 3.5) {\tiny $(s_2, t_2)$};
                \node [left] at (5, 2.2) {\tiny $(t_2, s_2)$};
                \node [right, rotate=90] at (6.3, 3.5) {\tiny $(t_2, s_2)$};

                \draw [fill] (6.8, 1.7) circle [radius=0.04];
                \draw [fill] (6.8, 1.3) circle [radius=0.04];
                \draw [fill] (7.2, 1.7) circle [radius=0.04];
                \draw [fill] (7.2, 1.3) circle [radius=0.04];
                \draw [dotted, thick] (5, 1.7) -- (7.2, 1.7);
                \draw [dotted, thick] (5, 1.3) -- (7.2, 1.3);
                \draw [dotted, thick] (6.8, 3.5) -- (6.8, 1.3);
                \draw [dotted, thick] (7.2, 3.5) -- (7.2, 1.3);
                \node [left] at (5, 1.7) {\tiny $(s_3, t_3)$};
                \node [right, rotate=90] at (6.8, 3.5) {\tiny $(s_3, t_3)$};
                \node [left] at (5, 1.3) {\tiny $(t_3, s_3)$};
                \node [right, rotate=90] at (7.2, 3.5) {\tiny $(t_3, s_3)$};
 
                \draw [fill] (8.0, 0.5) circle [radius=0.04];
                \draw [dotted, thick] (5, 0.5) -- (8.0, 0.5);
                \draw [dotted, thick] (8.0, 3.5) -- (8.0, 0.5);
                \node [left] at (5, 0.5) {\tiny $(s_4, t_4)$};
                \node [right, rotate=90] at (8.0, 3.5) {\tiny $(s_4, t_4)$};
 
                \node [below] at (5.3, 3.2) {\scriptsize $a$};
                \node [below left] at (5.95, 2.70) {\scriptsize $b$};
                \node [right] at (6.25, 2.22) {\scriptsize $b$};
                \node [below left] at (5.95, 2.25) {\scriptsize $c$};
                \node [above right] at (6.23, 2.55) {\scriptsize $c$};
                \node [below left] at (6.85, 1.80) {\scriptsize $d$};
                \node [right] at (7.15, 1.32) {\scriptsize $d$};
                \node [below left] at (6.85, 1.35) {\scriptsize $e$};
                \node [above right] at (7.15, 1.65) {\scriptsize $e$};
                \node [below] at (8.0, 0.5) {\scriptsize $f$};

                \draw [thin, ->] (9.6, 1.8) -- (8.5, 1.8);
                \draw [thin, ->] (9.6, -0.2) -- (9.1, -0.2);
                \node [right] at (9.6, 1.8) {$\cb{G}$};
                \node [right] at (9.6, -0.2) {removed};

                \node [below] at (1, -0.5)
                {$\estim{\Uv}^\top\estim{\Bv}\estim{\Vv}$};
                \node [below] at (7, -0.5) {$\Qv^\top\nabla^2 R(\Bv)\Qv$};
            \end{tikzpicture}
        \end{center}
        \caption{An illustration of the correspondence between the structure of
            the original matrix and the structure of the $\cb{G}$ matrix.
            As we have mentioned in Theorem \ref{thm:matrix-twice-diff}, $a=0$,
            $b = \frac{1}{\estim{\sigma}_{s_2} + \estim{\sigma}_{t_2}}$,
            $c = - \frac{1}{\estim{\sigma}_{s_2} + \estim{\sigma}_{t_2}}$,
            $d = \frac{1}{\estim{\sigma}_{t_3}}$,
            $e=-\frac{g_r[\estim{\sigma}_{s_3}]}{\estim{\sigma}_{t_3}}$,
            $f = \frac{1}{\estim{\sigma}_{t_4}}$.}
            \label{fig:nuclear-general}
    \end{figure}
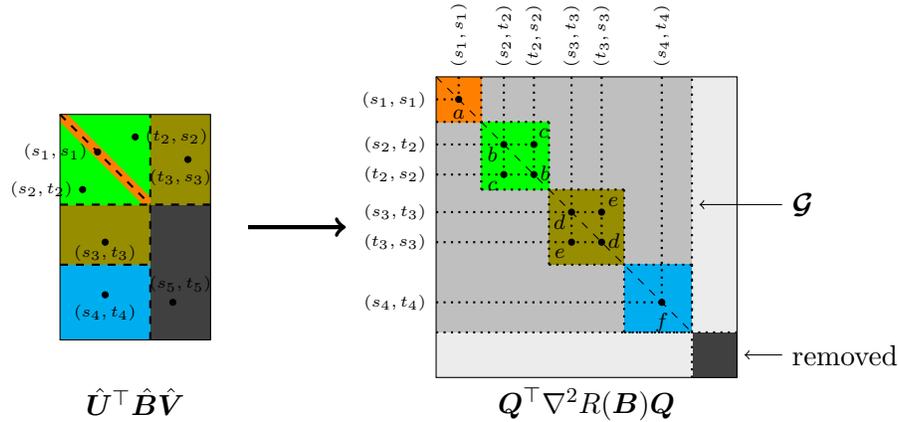
\end{remark}


\section*{Acknowledgements}
Arian Maleki, Wenda Zhou, and Shuaiwen Wang would like to acknowledge NSF grant DMS-1810888. 
We would also like to acknowledge computing resources from Columbia University's Shared Research
Computing Facility project, which is supported by NIH Research Facility
Improvement Grant 1G20RR030893-01, and associated funds from the New York State
Empire State Development, Division of Science Technology and Innovation
(NYSTAR) Contract C090171, both awarded April 15, 2010. 
\bibliography{reference}
\bibliographystyle{plain}

\clearpage

\appendix

\section{Proof of Equation \ref{eq:methods:dual}}
\label{sec:dual-derivation}
In this Section, we prove the primal-dual correspondence in
\eqref{eq:methods:primal} and \eqref{eq:methods:dual}.
Recall the form of the primal problem:
\begin{equation} \label{eq:primal-dual-proof:primal}
    \min_{\betav} \sum_{j = 1}^n \ell(\xv_j^\top \betav; y_j) + R(\betav).
\end{equation}

With a change of variable, we may transform \eqref{eq:primal-dual-proof:primal}
into the following form:
\begin{equation*}
    \min_{\betav, \muv} \sum_{j = 1}^n \ell(-\mu_j; y_j) + R(\betav),
    \quad
    \text{subject to: }
    \muv = -\Xv\betav.
\end{equation*}

We may further absorb the constraint into the objective function by adding a
Lagrangian multiplier $\thetav \in \mathbb{R}^n$:
\begin{equation} \label{eq:primal-dual-proof:primal1}
    \max_{\thetav}\min_{\betav, \muv}
    \sum_{j = 1}^n \ell(-\mu_j; y_j) + R(\betav) - \thetav^\top(\Xv\betav + \muv).
\end{equation}

Note that in \eqref{eq:primal-dual-proof:primal1}, $\betav$ and $\muv$ decoupled
from each other and we can optimize over them respectively. Specifically, we
have that
\begin{align}
    &\min_{\betav} R(\betav) - \thetav^\top\Xv\betav
    =
    - \max_{\betav} \big\{ \langle \betav, \Xv^\top\thetav \rangle - R(\betav)
    \big\}
    =
    - R^*(\Xv^\top\thetav), \label{eq:primal-dual-proof:beta} \\
    &\min_{\mu_j} \ell(-\mu_j; y_j) - \theta_j \mu_j
    =
    - \max \{ \mu_j\theta_j - \ell(-\mu_j; y_j) \}
    =
    - \ell^*(-\theta_j; y_j). \label{eq:primal-dual-proof:theta}
\end{align}

We plug \eqref{eq:primal-dual-proof:beta} and
\eqref{eq:primal-dual-proof:theta} in \eqref{eq:primal-dual-proof:primal1} and
obtain that
\begin{equation*}
    \max_{\thetav}
    \sum_{j = 1}^n - \ell^*(-\theta_j; y_j) - R^*(\Xv^\top \thetav).
\end{equation*}

\section{Proof of Lemma \ref{lem:proxproperties}} \label{proof:proxproperties}

\paragraph{Part 3.}
$u$ minimizes $\frac{1}{2\tau}(z - u)^2 + h(u)$ if and only if
\begin{equation*}
    \frac{z}{\tau} \in \frac{u}{\tau} + \partial h(u).
\end{equation*}

Obviously when $u = v_j$, $\partial h(v_j) = [\dot{h}_-(v_j), \dot{h}_+(v_j)]$.
This implies the set of possible values of $z$ is $[v_j + \tau \dot{h}_-(v_j),
v_j + \tau \dot{h}_+(v_j)]$. The convexity of $h$ guarantees that for different
$v_j$, these intervals are non-overlapping with each other.

\paragraph{Part 4.}
Note that since
\begin{equation*}
    \proxv_h(\uv) = \argmin_{\zv \in \mathbb{R}^p}\frac{1}{2} \|\uv-\zv\|_2^2 + h(\zv),
\end{equation*}
we have that
\begin{equation*}
    \proxv_h(\uv)- \uv + \nabla h(\proxv_h(\uv))= \bm{0}.
\end{equation*}

Let $\Jv$ be the Jacobina of $\proxv_h$. By taking derivatives of both sides of
the above equation we have
\begin{equation*}
    \Jv(\uv) - \Iv + \nabla^2 h(\proxv_h(\uv)) \Jv(\uv) =0,
    \quad \Rightarrow \quad
    \Jv(\uv) = [\Iv + \nabla^2 h(\proxv_R(\uv))]^{-1}.
\end{equation*}

Note that since $h$ is convex, $\nabla^2 h$ is a positive semidefinite matrix.
This means that all the eigenvalues of $\Iv + \nabla^2 h(\proxv_R(\uv))$ are
greater than or equal to one. This completes our proof.

\section{Derivation of the Dual for Generalized LASSO}
\label{append:sec:generalized-lasso-dual}

In this section we derive the dual form of the generalized LASSO stated in
the main paper. We recall that for a given matrix $\Dv \in \RR^{m \times p}$,
the generalized LASSO is given by:
\begin{equation*}
    \min_{\betav} \frac{1}{2} \sum_{j = 1}^n (y_j - \xv_j^\top \betav)^2
    + \lambda \norm{\Dv \betav}_1.
\end{equation*}

Introduce dummy variables $\zv \in \RR^n$, $\wv \in \RR^m$, and consider the following
equivalent constrained optimization problem:
\begin{equation*}
    \min_{\betav, \zv, \wv} \frac{1}{2}\norm{\zv}_2^2 + \lambda \norm{\wv}_1,
    \quad
    \text{subject to: } \yv - \Xv\betav = \zv \text{ and } \Dv \betav = \wv.
\end{equation*}

We may now consider the Lagrangian form of the optimization problem, introducing dual variables
$\dualv \in \RR^n$ and $\uv \in \RR^m$, the dual problem is
\begin{align*}\label{eq:genlasso-dual:lagrangian}
    &
    \max_{\dualv, \uv} \min_{\betav, \zv, \wv} \frac{1}{2} \norm{\zv}_2^2 + \lambda \norm{\wv}_1
    + \dualv^\top(\yv - \Xv\betav - \zv) + \uv^\top(\Dv \betav - \wv) \nonumber \\
    =&
    - \min_{\dualv, \uv} \bigg[ \max_{\zv} \bigg\{\thetav^\top \zv
    - \frac{1}{2}\|\zv\|_2^2 \bigg\} + \max_{\wv} \big\{ \uv^\top \wv - \lambda
    \norm{\wv}_1 \big\} + \max_{\betav} \big\{\thetav^\top\Xv\betav -
    \uv^\top\Dv\betav \big\} - \thetav^\top\yv \bigg].
\end{align*}

Consider the three subproblems within square brackets respectively, we have
\begin{equation*}
    \max_{\zv} \bigg\{\dualv^\top \zv - \frac{1}{2}\norm{\zv}_2^2 \bigg\}
    = \frac{1}{2} \norm{\dualv}_2^2,
    \qquad
    \max_{\wv} \big\{\uv^\top \wv - \lambda \norm{\wv}_1 \big\} = \begin{cases}
        0 & \text{ if } \norm{\uv}_\infty \leq \lambda, \\
        \infty & \text{ otherwise. }
    \end{cases}
\end{equation*}
where $ \thetav^\top\Xv\betav - \uv^\top\Dv\betav $ is unbounded unless
$\Xv^\top \dualv = \Dv^\top \uv$.
Finally, we substitute the above results into our Lagrangian dual problem to obtain:
\begin{equation*}
    \min_{\dualv, \uv} \frac{1}{2} \norm{\dualv}_2^2 - \dualv^\top \yv,
    \quad
    \text{subject to: }  \Dv^\top \uv = \Xv^\top \dualv \text{ and }
    \norm{\uv}_\infty \leq \lambda.
\end{equation*}
which is equivalent to the stated dual problem.

\section{Jacobian of the Projection on Positive Semidefinite Cone}\label{ssec:JacobianPSDcalc}
First note that for an arbitrary matrix $\Bv$  the projection involves two steps:
(i) symmetrization, i.e. projecting $\Bv$ to $\calS_p$ and obtain
$\projv_{\calS^p}(\Bv) = \frac{1}{2}(\Bv + \Bv^\top)$; and
(ii) projection of $\projv_{\calS^p}(\Bv)$ on $\calS_{+}^p$: if
$\projv_{\calS^p}(\Bv) = \Qv \diag[\{\lambda_j\}_j] \Qv$, then the projection on
$\calS_+^p$ is $\projv_{\calS_+}(\Bv) = \Qv \diag[\{(\lambda_j)_+\}_j] \Qv^\top.$
Hence, by using the chain rule, the Jacobian $\Jv$ of the entire projection
process can be written as $\Jv = \Jv_1 \Jv_2$, where $\Jv_2$ is the Jacobian of the
$\projv_{\calS^p}(\Bv)$, and $\Jv_1$ is the Jacobian of $\projv_{S_+^p}(\cdot)$ at
$\projv_{\calS^p}(\Bv)$.
The calculation of $\Jv_2$ is simple. In the rest of this section, we only
focus on characterizing $\Jv_1$. Let $\Av = \frac{1}{2}(\Bv + \Bv^\top)$. Define
$F(\Av) = \Qv \diag[\{f(\lambda_j)\}_j] \Qv^\top$. The directional
derivative of $F(A)$ in the direction of $\Deltav$ is given by
\begin{equation*}
    d F(A)[\Deltav] =
    d\Qv[\Deltav] \diag[\{f(\lambda_j)\}_j] \Qv^\top +
    \Qv \diag[\{f(\lambda_j)\}_j] d\Qv[\Deltav]^\top +
    \Qv \diag[\{f'(\lambda_j)\}_j] \diag[\{d\lambda_j[\Deltav]\}_j] \Qv^{\top}.
\end{equation*}

This leads to
\begin{align}\label{eq:firstonepsd1}
     & \Qv^\top d F(A)[\Deltav] \Qv \nonumber \\
    =& \Qv^\top d\Qv[\Deltav] \diag[\{f(\lambda_j)\}_j] +
    \diag[\{f(\lambda_j)\}_j] d\Qv[\Deltav]^\top\Qv +
    \diag[\{f'(\lambda_j)\}_j] \diag[\{d\lambda_j[\Deltav]\}_j] \nonumber \\
    =& \Qv^\top d\Qv[\Deltav] \diag[\{f(\lambda_j)\}_j] -
    \diag[\{f(\lambda_j)\}_j] \Qv^\top d\Qv[\Deltav]
    +  \diag[\{f' (\lambda_j)\}_j] \diag[\{d\lambda_j[\Deltav]\}_j].
\end{align}
where the last equality is due to the fact that $\Qv^\top \Qv = \Iv$, and thus
$\Qv^\top d\Qv[\Deltav] = -d\Qv[\Deltav]^\top\Qv$. In order to find the
elements of the Jacobian, we consider the following bases for the space of
symmetric matrices $\calS^p$:
\begin{align*}
    \Kv_{ii} =& \qv_i \qv_i^\top, \quad  i= 1, \ldots, p, \nonumber \\
    \Kv_{ij} =& \frac{1}{\sqrt{2}} \qv_i \qv_j^\top + \frac{1}{\sqrt{2}}
    \qv_j \qv_i^\top, \quad 1 \leq i < j \leq p.
\end{align*}

Let ${\Ev}_{ij}$ denote the canonical basis for $\calS^p$:
for $i < j$, $\Ev_{ij}$ denotes the matrix which equals $1/\sqrt{2}$ at
$(i,j)^{\rm th}$ and $(j,i)^{\rm th}$ location and 0 elsewhere; for $i = j$,
$\Ev_{ii}$ has only a $1$ at $(i,i)^{\rm th}$ and 0 elsewhere.
Define $\Omegav[\Deltav] =  \Qv^\top d\Qv[\Deltav]$. By setting $f(\lambda) =
\lambda$ in \eqref{eq:firstonepsd1} and taking inner product with $\Ev_{ij}$ of
both sides, it is not hard to see that
\begin{align}\label{eq:psd2}
    \langle \Omegav[\Deltav] , \Ev_{ij} \rangle
    =&
    \frac{\langle \Qv^\top \Deltav \Qv, \Ev_{ij} \rangle}
    {\lambda_j - \lambda_i}, \quad i \neq j \nonumber \\
    \langle \Omegav[\Deltav] , \Ev_{ii}\rangle
    =& 0, \nonumber \\
    d\lambda_i [\Deltav]
    =& \langle \Qv^\top \Deltav \Qv, \Ev_{ii} \rangle.
\end{align}

Set $\Deltav = \Kv_{st}$ in \eqref{eq:firstonepsd1}, we have that
\begin{equation*}
      \langle dF(A)[\Kv_{s t}],  \Kv_{ij} \rangle
    = \langle \Qv^\top dF(A)[\Kv_{s t}] \Qv,  \Qv^\top \Kv_{ij} \Qv \rangle
    = \langle \Qv^\top dF(A)[\Kv_{s t}] \Qv,  \Ev_{ij} \rangle
\end{equation*}

Using \eqref{eq:psd2}, it is straightforward to see that, when $s < t$, the
only way to make $\langle dF(\Av)[\Kv_{st}], \Kv_{ij} \rangle$ not zero is when
$s = i$ and $t=j$. In that case $\langle dF(\Av)[\Kv_{st}], \Kv_{ij} \rangle =
\frac{f(\lambda_t) - f(\lambda_s)} {\lambda_t - \lambda_s}$. Similarly when $s
= t$, we need $i=j=s=t$ to have nonzero inner product and in this case $\langle
dF(\Av)[\Kv_{ss}], \Kv_{ij} \rangle = f'(\lambda_s)$.

Finally to obtain the result for projection on $\calS_+^p$, we pick $f(\lambda)
= (\lambda)_+$ and everything then follows.

\end{document}